\documentclass[11pt]{amsart}

\usepackage[left=1in, right=1in, top=1in, bottom=1in]{geometry}

\usepackage[utf8]{inputenc}
\usepackage[T1]{fontenc}
\usepackage{amsmath, amsfonts, amssymb}
\usepackage{graphicx, booktabs, float, adjustbox}
\usepackage{microtype, xcolor, url}
\usepackage[normalem]{ulem}
\usepackage{bm}
\usepackage{thmtools}
\usepackage{thm-restate}
\usepackage{caption}

\usepackage{hyperref}
\usepackage{cleveref}

\newtheorem{theorem}{Theorem}[subsection]
\newtheorem{proposition}[theorem]{Proposition}

\newtheorem{assumption}[theorem]{Assumption}
\newtheorem{remark}{Remark}



\title{Plug-and-Play Image Restoration with Flow Matching:\\
       A Continuous Viewpoint}

\author{Fan Jia}
\author{Yuhao Huang}
\author{Shih-Hsin Wang}
\author{Cristina Garcia-Cardona}
\author{Andrea L. Bertozzi}
\author{Bao Wang}

\address[Fan Jia, Yuhao Huang, Shih-Hsin Wang, Bao Wang]{Department of Mathematics, University of Utah, Salt Lake City, UT 84112, USA}
\email[Fan Jia]{fan.jia@utah.edu}
\email[Yuhao Huang]{u1430219@utah.edu}
\email[Shih-Hsin Wang]{u1371684@utah.edu}
\email[Bao Wang]{bwang@math.utah.edu}

\address[Cristina Garcia-Cardona]{Los Alamos National Laboratory, Los Alamos,  USA}
\email[Cristina Garcia-Cardona]{cgarciac@lanl.govv}

\address[Andrea L. Bertozzi]{Department of Mathematics, University of California, Los Angeles, USA}
\email[Andrea L. Bertozzi]{bertozzi@math.ucla.edu}



\hypersetup{
  pdftitle    = {Plug-and-Play Image Restoration with Flow Matching: A Continuous Viewpoint},
  pdfauthor   = {Fan Jia, Yuhao Huang, Shih-Hsin Wang, Cristina Garcia-Cardona, Andrea L. Bertozzi, Bao Wang},
  pdfkeywords = {image restoration, flow matching, plug-and-play, stochastic differential equations}
}

\begin{document}

\begin{abstract}
Flow matching-based generative models have been integrated into the plug-and-play (PnP) image restoration framework, yielding the PnP-Flow model that has achieved remarkable empirical success. However, its theoretical understanding remains limited. In this paper, we derive a continuous-limit stochastic differential equation (SDE) surrogate for PnP-Flow. This continuous viewpoint provides two key insights: (1) it enables rigorous error quantification, guiding improved step scheduling and Lipschitz regularization of the neural vector field; (2) it motivates an extrapolation-based acceleration strategy for off-the-shelf PnP-Flow models. Extensive experiments on image denoising, deblurring, super-resolution, and inpainting demonstrate that our SDE-informed enhancements significantly outperform baseline PnP-Flow and competing state-of-the-art methods across standard metrics.
\end{abstract}

\maketitle

\section{Introduction}
\label{sec:intro}
Image restoration seeks to reconstruct high-quality images suffering from 
degradations, such as noisy \cite{buades2005review,elad2023image,jia2021ddunet,pearl2022nan}, blurry \cite{zhang2020deblurring,zhang2022deep,quan2021gaussian}, low-resolution \cite{huang2020unfolding,fan2020neural}, enhancement \cite{xu2022snr,ma2022toward,jia2024variational}, and foreground occlusions \cite{bar2022visual,yeh2017semantic,kawar2022denoising}. The relationship between an unknown high-quality image ${\bm x}\in\mathbb{R}^n$ and a degraded observation $\boldsymbol{\omega}\in\mathbb{R}^m$ can be modeled as 
\begin{equation*}
    \boldsymbol{\omega} = \mathcal{A}{\bm x} + \boldsymbol{\nu},
\end{equation*}
where $\mathcal{A}: \mathbb{R}^n\to \mathbb{R}^m$ is a degradation operator that can be linear or nonlinear, and $\boldsymbol{\nu}$ represents the underlying additive noise. Traditional approaches often formulate this as an optimization problem, minimizing a well-designed objective function---that balances a data fidelity term $f(\cdot)$ and a regularization term $g(\cdot)$ reflecting image priors---as follows:
\begin{equation*}
   {\bm x}^* = \arg\min\limits_{{\bm x}} f({\bm x}) + g({\bm x}),
\end{equation*}
where $f({\bm x})$ is often defined to be $\frac{1}{2}\|\mathcal{A}{\bm x} - \boldsymbol{\omega}\|^2$ when $\boldsymbol{\nu}$ is a zero-mean Gaussian noise.

Crafting an effective regularizer for complex priors is challenging, and the plug-and-play (PnP) framework \cite{venkatakrishnan2013plug} addresses this by replacing the explicit regularizer \(g({\bm z})\) with an implicit prior---often a pre-trained neural network (NN) denoiser. PnP combines the flexibility of optimization methods with the expressivity of NN, achieving state-of-the-art (SOTA) performance across various tasks. Leveraging their impressive performance in image generation, generative models have been 
adopted within the PnP framework (cf.~\cite{wu2024principled,song2023loss,zhu2023denoising}). A noticeable generative model is based on flow matching (FM)~\cite{lipman2022flow}, which learns a vector field with associated probability density path interpolating between a prior and the data distributions, offering computational efficiency through deterministic ordinary differential equation (ODE) flows \cite{chen2023probability}. Integrating FM into PnP, termed PnP-Flow~\cite{martin2024pnp}, has shown promising empirical results.

While PnP methods with learned implicit priors (e.g., NNs or generative models) often outperform traditional approaches empirically, their theoretical analysis is complex. Implicit regularizers may not satisfy assumptions required by classical optimization tools, limiting insights into convergence or error bounds. To address this issue, by assuming NNs possess specific properties, such as being Lipschitz differentiable or convex, provably convergent PnP methods have been developed with rigorous convergence analysis \cite{tan2024provably,ryu2019plug,cohen2021regularization}. However, convergence and error bounds for PnP with generative models still remain largely unaddressed.

In this paper, we propose a continuous perspective on PnP-Flow by modeling it with a stochastic differential equation (SDE). The SDE formulation is not merely a mathematical reformulation---it provides fundamental insights that bridge the theoretical gap in PnP frameworks and enables principled design of guidance algorithms as well as training loss functions across generative models. This approach introduces novel analytical tools to characterize key parameters governing convergence rates and error bounds within the PnP framework.

\subsection{Our Contributions}
We aim to enhance our understanding of PnP-Flow and improve its performance and efficiency in a principled manner. 
Our key contributions are:

\begin{itemize}
    \item \textbf{Continuous Limit Derivation}: We show that the continuous-in-time (iteration) limit of PnP-Flow is an SDE; see Section~\ref{sec:SDE-limit}. 
    
    \item \textbf{Error Quantification}: Using the SDE surrogate model, we bound the image restoration error for PnP-Flow, highlighting the roles of the Lipschitz constant of the NN-parameterized vector field and step scheduling in error reduction; see Section~\ref{sec:Theory}.
        
    \item \textbf{Theoretically-Principled Improvement:} We propose 1) a new step scheduling for PnP-Flow, 2) an efficient Lipschitz regularization based on Hutchinson's method \cite{hutchinson1989stochastic}, 
    and 3) an extrapolated 
    iteration to enhance efficiency without retraining; see Section~\ref{sec:Algorithm}.
    
    \item \textbf{Empirical Validation}: We validate the advantages of our improved PnP-Flow using various benchmark tasks; see Section~\ref{sec:Experiments}. 
\end{itemize}

\subsection{Additional Related Works}

Using generative models as PnP denoisers within an iterative framework is an emerging research area. As one of the earliest works in this direction, \cite{bora2017compressed} demonstrates that generative models could act as powerful priors for compressed sensing. Subsequent developments extend this idea to generative adversarial network (GAN)-based approaches \cite{shah2018solving}, score-based generative methods \cite{kadkhodaie2021stochastic}, and more recently, diffusion models (DMs) tailored for inverse problem-solving \cite{chung2022come,chung2023diffusion}. Recent advances further propose PnP frameworks that alternate between likelihood and DM sampling for general inverse problems \cite{wu2024principled,xu2024provably}.

\subsection{Organization}

We structure this paper as follows: We review FM and PnP-Flow in Section~\ref{sec:Preliminaries}. We derive the SDE-limit of PnP-Flow in Section~\ref{sec:SDE-limit}. We analyze the PnP-Flow error and present improved PnP-Flow in Sections~\ref{sec:Theory} and \ref{sec:Algorithm}, respectively. We validate the efficacy and efficiency of the improved PnP-Flow, together with ablation studies, in Section~\ref{sec:Experiments}. Technical proofs and additional experimental details are provided in the appendix.

\section{Preliminaries}\label{sec:Preliminaries}

\subsection{Flow Matching with Straight-Line Paths}

FM (cf.~\cite{lipman2022flow,albergo2023building,liu2023flow}) learns a vector field that induces a probability density path interpolating between a prior distribution $q$ (e.g., the standard Gaussian $\mathcal{N}({\bf 0},{\bm I})$) and the data distribution $p$. Mathematically, for a given vector field ${\bm u}_t:[0,1]\times \mathbb{R}^d\to \mathbb{R}^d$, it defines a flow $\phi_t: [0,1]\times \mathbb{R}^n \to \mathbb{R}^n$ via the following ODE:
\begin{equation}\label{FM_ODE}
\frac{d}{dt}\phi_t({\bm x}) = {\bm u}_t(\phi_t({\bm x})), \ \ 
        \phi_0({\bm x}) = {\bm x}_0.
\end{equation}
The flow $\phi_t$ interpolates between the noise sample $\phi_0({\bm x})={\bm x}_0\sim q$ and the data sample $\phi_1({\bm x})={\bm x}_1\sim p$. Moreover, $\phi_t$ induces a probability path, defined as $p_t({\bm x})=p_0(\phi_t^{-1}({\bm x})){\rm det}[\frac{\partial\phi_t^{-1}({\bm x})}{\partial{\bm x}}],\forall {\bm x}\in p_0=q$, interpolates between the prior ($p_0=q$) and data ($p_1=p$) distributions. In practice, the vector field ${\bm u}_t$ is unavailable, which is only defined---in its conditional form ${\bm u}_t(\cdot|{\bm x}_1)$---for each training data ${\bm x}_1\sim p$. FM regresses a neural network $\tilde{{\bm u}}_t(\cdot,\theta)$---with $\theta$ being learnable parameters---against ${\bm u}_t(\cdot|{\bm x}_1)$ for each training data, resulting in an unbiased estimate for the unconditional vector field ${\bm u}_t$ \cite{lipman2022flow}; this process is named conditional flow matching (cf.~\cite{lipman2022flow}). The learned vector field $\tilde{{\bm u}}_t(\cdot,\theta)$ drives any noise sample to a generated realistic data. 

In \cite{liu2023flow,martin2024pnp}, the authors construct a straight-line path for FM as follows: for any 
${\bm x}_1\sim p$, we can define a conditional flow $\phi_t({\bm x}|{\bm x}_1)=(1-t){\bm x}+t{\bm x}_1$. 
The corresponding conditional vector field is ${\bm u}_t(\phi_t({\bm x})|{\bm x}_1)={\bm x}_1-{\bm x}$. While other path formulations for flow matching exist, straight-line paths give a concise form of $\phi_t({\bm x}|{\bm x}_1)$ and simplify training and implementation. For simplicity, we denote ${\bm u}_t({\bm x}_t)$ as in \cite{martin2024pnp}, where it represents ${\bm u}_t({\bm x}_t|{\bm x}_1)$. The definition and implementation of the optimal transport (OT) flow and ${\bm u}_t$ in our paper are consistent with those in \cite{martin2024pnp}.

\subsection{
PnP-Flow for Image Restoration}
Given a minimization problem
\begin{equation}\label{prob}
    \min_{{\bm x}} f({\bm x}) + g({\bm x}),
\end{equation}
the proximal gradient method solves \cref{prob} by the following iteration:
\begin{equation}\label{eq:prox}
    \begin{cases}
        {\bm z}_{k+1} = {\bm x}_k - \gamma_k \nabla f({\bm x}_k), \\
        {\bm x}_{k+1} = \operatorname{prox}_{\gamma_k g}({\bm z}_{k+1}),
    \end{cases}
\end{equation}
where the proximal operator is defined as:
\begin{equation*}
    \operatorname{prox}_{\gamma_k g}({\bm z}) = \arg\min_{{\bm x}} g({\bm x}) + \frac{1}{2\gamma_k} \|{\bm x} - {\bm z}\|_2^2.
\end{equation*}
As PnP methods generally replace the proximal operator with an off-the-shelf denoiser, PnP-Flow \cite{martin2024pnp} replaces the proximal operator in \cref{eq:prox} by an FM process, resulting in 
\begin{equation}\label{iter:PnPFM}
    \begin{cases}
        {\bm z}_{k} = {\bm x}_k - \gamma_k \nabla f({\bm x}_k), \\
        {\bm y}_k = (1 - l_k) \boldsymbol{\xi} + l_k {\bm z}_{k}, \quad \boldsymbol{\xi} \sim \mathcal{N}(0, \boldsymbol{I}), \\
        {\bm x}_{k+1} = D_{l_k}({\bm y}_k),
    \end{cases}
\end{equation}
where $l_k$ increases 
with $l_0=0$ and $\lim_{k \to \infty} l_k = 1$, and $D_{l_k} := \operatorname{Id} + (1 - l_k) \boldsymbol{u}_{l_k}$ with $\boldsymbol{u}_{l_k}$ being the vector field of the FM model.
The following remark provides rationale for 
\cref{iter:PnPFM}.
\begin{remark}
    Instead of solving \cref{FM_ODE}, \cite{martin2024pnp} suggests that the operator $D_t({\bm x}_t) = [\operatorname{Id}+(1-t){\bm u}_t]\circ {\bm x}_t$ with $\operatorname{Id}$ being the identity, can be interpreted as the best approximation of ${\bm x}_1$ given the knowledge of ${\bm x}_t$. Moreover, the second step in \cref{iter:PnPFM} projects ${\bm z}_k$ to the straight-line path.
\end{remark}

The convergence of the iteration process in \cref{iter:PnPFM} is guaranteed as follows:
\begin{proposition} \cite[Proposition 4]{martin2024pnp}\label{prop4}
Assume that $f: \mathbb{R}^n \to \mathbb{R}$ is continuously differentiable and the learned vector field $\tilde{{\bm u}}_t(\cdot,\theta)$ is continuous. Let $\{l_k\}_{k \in \mathbb{N}}$ satisfy $\sum\limits_{k=0}^{\infty}(1 - l_k) <  +\infty$ and $\gamma_k := 1 - l_k$. If
$\{{\bm x}_k\}_{k \in \mathbb{N}}$ obtained by \cref{iter:PnPFM} with $\tilde{{\bm u}}_t(\cdot,\theta)$ is bounded, then it converges.
\end{proposition}

\section{The Continuous Limit of PnP-Flow}\label{sec:SDE-limit}
As the first contribution, we derive an SDE model for \cref{iter:PnPFM}. 
Plugging $D_{l_k} := \operatorname{Id} + (1 - l_k){\bm u}_{l_k}$ into the last equation of \cref{iter:PnPFM}, we have
\begin{equation}\label{eq:xk-update}
    \begin{aligned}
        {\bm x}_{k+1} & =[\operatorname{Id} + (1 - l_k){\bm u}_{l_k}]\circ ({\bm y}_k) \\
        & = (1-l_k)\boldsymbol{\xi} + l_k{\bm z}_k +(1-l_k){\bm u}_{l_k}({\bm y}_k) \\
        & = (1-l_k)\boldsymbol{\xi} + l_k({\bm x}_k - \gamma_k \nabla f({\bm x}_k)) + (1-l_k){\bm u}_{l_k}({\bm y}_k). 
    \end{aligned}  
\end{equation}
Applying Taylor expansion to ${\bm u}_{l_k}$ with respect to ${\bm y}_k$ at point ${\bm x}_k$ gives
\begin{equation}\label{eq:taylor_u}
    \begin{aligned}
        {\bm u}_{l_k}({\bm y}_k) & = {\bm u}_{l_k}({\bm x}_k) + ({\bm y}_k -{\bm x}_k)\nabla {\bm u}_{l_k}({\bm x}_k) + o({\bm y}_k -{\bm x}_k).
   \end{aligned}
\end{equation}
Moreover, by the first two equations in \cref{iter:PnPFM}, we have
\begin{equation*}
    \begin{aligned}
        {\bm y}_k -{\bm x}_k & =  (1 - l_k) \boldsymbol{\xi} + l_k {\bm z}_{k} -{\bm x}_k\\
        & = (1-l_k)(\boldsymbol{\xi}-{\bm x}_k)-l_k\gamma_k\nabla f({\bm x}_k),
    \end{aligned}
\end{equation*}
let $\beta_k=\frac{l_k\gamma_k}{1-l_k}$, we have
\begin{equation}\label{eq:y-x}
     {\bm y}_k -{\bm x}_k = (1-l_k)(\boldsymbol{\xi}-{\bm x}_k-\beta_k\nabla f({\bm x}_k)).
\end{equation}
Combining \cref{eq:xk-update,eq:taylor_u,eq:y-x}, we have 
\begin{equation}\label{eq:discrete-iter}
\begin{aligned}
    {\bm x}_{k+1} - {\bm x}_k &= (1-l_k)(\boldsymbol{\xi}-{\bm x}_k - \beta_k\nabla f({\bm x}_k) +{\bm u}_{l_k}({\bm x}_k)) + (1-l_k)O({\bm y}_k-{\bm x}_k)  \\
    &= (1-l_k)(\boldsymbol{\xi}-{\bm x}_k - \beta_k\nabla f({\bm x}_k) +{\bm u}_{l_k}({\bm x}_k)) + o(1-l_k),
\end{aligned}
\end{equation}
where the second equality comes from the fact that $(1 - l_k)({\bm y}_k - {\bm x}_k) = O((1-l_k)^2) = o(1 - l_k)$---multiplying both sides of \cref{eq:y-x} by $1-l_k$ and choose $\beta_k = O(1 - l_k)$---as $1 - l_k \to 0$.

To derive a continuous-time limit for \cref{eq:discrete-iter}, we introduce the \textit{ansatz} $ {\bm x}_k \approx X(t)$, where $X(t)$ is a smooth curve defined for $ t \geq 0 $. Let $\Delta t:=1-l_k$ representing the stepsize for the $k$th iteration, then $X(t + \Delta t) \approx {\bm x}_{k+1}$. Therefore, as $1-l_k\to 0$, \cref{eq:discrete-iter} becomes
\begin{equation}\label{delta_t}
X(t+\Delta t) - X(t) = \Delta t (-X(t) - \beta(t) \nabla f(X(t)) + {\bm u}_t(X(t)))+ \sigma(t) (W(t+\Delta t)-W(t)),
\end{equation}
where $W(t)$ is a standard Wiener process, $\beta(t)$ is a time-dependent coefficient corresponding to $\beta_k$, and ${\bm u}_t(X(t))$ is a time-dependent vector field corresponding to ${\bm u}_{l_k}(X(t))$, $\sigma(t)$ is the continuous analogy of $\sqrt{1 - l_k}$ arising from the Euler-Maruyama discretization of the SDE \cite{kloeden1992stochastic}. For simplicity, we denote $X(t)$ by $X_t$ in the rest of this paper.

Notice that iterations \cref{iter:PnPFM} converge when $\sum_{k=0}^\infty (1 - l_k) <  +\infty$. This condition ensures that the cumulative step size $t = \sum_{i=0}^k (1 - l_i)$ converges to a finite limit $T$. We summarize the above derivation in the following proposition.

\begin{proposition}\label{prop:SDE-PnP-FM}
Let $X(t_0) = {\bm x}_K$ for some iteration number $K$, \cref{iter:PnPFM} can be viewed as a discrete version of the following SDE for $0<t_0\leq t\leq T$: 
\begin{equation}\label{SDE}
d X_t = {\bm b}_t(X_t)dt + \sigma(t) dW_t,
\end{equation}
where ${\bm b}_t(X_t)=-X_t - \beta(t) \nabla f(X_t) + {\bm u}_t(X_t)$, ${\bm u}_t(X_t)$ is a time-dependent vector field associated to ${\bm u}_{l_k}$, and $\beta(t), \sigma(t)\in [0,1)$  are monotonically non-increasing in $t$, and 
go to zeros.
\end{proposition}

\begin{remark}
As the iteration number $k$ increases, $1-l_k \to 0$, but it remains insufficiently small in early iterations to ensure an accurate continuous approximation of the discrete process. Consequently, the iterative update in \cref{iter:PnPFM} can be approximated by an SDE with an initial condition at some $t\geq t_0 > 0$, given by $X(t_0) = {\bm x}_K$. The SDE in \cref{SDE} accounts for the non-negligible step size in early iterations by starting the SDE at a positive time $t_0$.
\end{remark}

\section{Error and Convergence Analysis of PnP-Flow}\label{sec:Theory}
\subsection{Error Estimation}\label{sec:4.1}
In image restoration applications, the ground truth vector field satisfying the ODE \cref{FM_ODE} is only defined for each training data in a conditional fashion. The unknown marginal vector field ${\bm u}_t(\cdot)$ is approximated by a NN-parameterized vector field $\tilde{{\bm u}}_t(\cdot,\theta)$---obtained by regressing against the conditional vector field using CFM \cite{lipman2022flow}. Consequently, the error between the true restored ${\bm x}^*$ and the approximated solution $\tilde{{\bm x}}^*$ obtained via the iteration \cref{iter:PnPFM} comes from the gap between the vector fields ${\bm u}_t$ and $\tilde{{\bm u}}_t$.

In this subsection, we aim to estimate the error $\epsilon = \|{\bm x}^*-\tilde{{\bm x}}^*\|$ and the convergence rate of PnP-Flow. Using the continuous counterpart, we seek to bound the expected error $\mathbb{E}[\|X_t-\tilde{X}_t\|]$. Consider the initial error $\epsilon_0 = \|{\bm x}_K-\tilde{{\bm x}}_K\|$, and the following pair of SDEs:
\begin{equation*}
\begin{aligned}
d X_t = {\bm b}_t(X_t)dt + \sigma(t) dW_t, \ \ \mbox{and}\ \ d \tilde{X}_t = \tilde{{\bm b}}_t(\tilde{X}_t) dt + \sigma(t) dW_t, 
\end{aligned}
\end{equation*}
where $\tilde{{\bm b}}_t(\tilde{X}_t) = -\tilde{X}_t - \beta(t) \nabla f(\tilde{X}_t) +\tilde{{\bm u}}_t(\tilde{X}_t)$. We notice that \cref{SDE} admits a global unique solution \cite{oksendal2013stochastic} given
$X(t_0) = {\bm x}_K$ under the following common mild assumptions \cite{oksendal2013stochastic,su2016differential}:
\begin{assumption}[\cite{oksendal2013stochastic}[Theorem 5.2.1]\label{ass1}
    Let $T> 0$ and ${\bm b}_t(X_t): [0,T]\times \mathbb{R}^n\to\mathbb{R}^n$ and $\sigma(t):\mathbb{R}\to\mathbb{R}$ be measurable functions satisfying $\|{\bm b}_t(X_t)\| + \sigma (t) \leq C_1(1+\|X_t\|); \ X_t \in\mathbb{R}^n, \ \forall t\in[0,T]$ for some constant $C_1$, and $\|{\bm b}_t(X_t)-{\bm b}_t(Y_t)\| \leq C_2\|X_t-Y_t\|$ for some constant $C_2$. Moreover, the initial condition satisfies $\mathbb{E}[\|X_{t_0}\|^2]< \infty$.
\end{assumption}
To analyze the error term $\mathbb{E}[\|X_t-\tilde{X}_t\|]$, we further assume that 
\begin{assumption}\label{ass2}
The learned vector field $\tilde{{\bm u}}_t(X_t)$ is bounded and Lipschitz continuous, i.e.,
\begin{equation*}
\|\tilde{{\bm u}}_s(X_s)-\tilde{{\bm u}}_t(X_t)\| \leq L_u\|X_s-X_t\|,
\end{equation*}
where $L_u$ is the Lipschitz constant. Since $\tilde{{\bm u}}_t$ is an approximation of ${\bm u}_t$, we further assume that ${\bm u}_t$ is also bounded and Lipschitz continuous with $L_u$.
\end{assumption}

\begin{assumption}\label{ass3}
$f(\cdot): \mathbb{R}^n \to \mathbb{R}$ is a bounded twice differentiable function with a bounded Jacobian matrix, and $\nabla f(\cdot)$ is Lipschitz continuous. That is,
\begin{equation*}
    \begin{aligned}
        \|\nabla f(X_t)-\nabla f(Y_t)\| \leq L_f\|X_t-Y_t\|,\ \ \mbox{and}\ \ 
         \|\nabla f(X_t)\| \le M_f,
    \end{aligned}
\end{equation*}
where $L_f$ is the Lipschitz constant for function $f$, and $M_f$ is a positive constant.
\end{assumption}

The following theorem gives an upper bound for $\mathbb{E}[\|X_t-\tilde{X}_t\|]$:
\begin{restatable}{theorem}{thmError}\label{thm:error}
Let $X_t, \tilde{X}_t$ be the variables generated by \cref{SDE} with ground truth ${\bm u}_t(X_t)$ and learned vector $\tilde{{\bm u}}_t(\tilde{X}_t)$ fields at time $t$. Let $Z_t = X_t-\tilde{X}_t$, and $\epsilon_0 = \|X_{t_0}-\tilde{X}_{t_0}\| = \|{\bm x}_{K} - \tilde{{\bm x}}_K\|$. Under \cref{ass1}, \ref{ass2} and \ref{ass3}, we have 
\begin{equation}\label{error}
    \begin{aligned}
        \mathbb{E}[\|Z_t\|] &\leq C_{\epsilon}e^{B_{\epsilon}},
    \end{aligned}
\end{equation}
where $C_{\epsilon} = \epsilon_0 +\int^t_{t_0}\mathbb{E}[\|{\bm u}_s(X_s)-\tilde{{\bm u}}_s(X_s)\|]ds$ and $B_{\epsilon} = \int^t_{t_0}(1+\beta(s)L_f+L_u)ds$.
\end{restatable}
\begin{proof}
    \begin{equation}
    \begin{aligned}
            \mathbb{E}[\|Z_t\|] &= \mathbb{E}[\|X_t-\tilde{X}_t\|] \\
            & = \mathbb{E}[\|X_{t_0}-\tilde{X}_{t_0}+\int^t_{t_0}({\bm b}_s(X_s)-\tilde{{\bm b}}_s(\tilde{X}_s))ds + \int^t_{t_0}\sigma(s)-\sigma(s)dW_s\|] \\
            & \le \epsilon_0+\mathbb{E}[\|\int^t_{t_0}({\bm b}_s(X_s)-\tilde{{\bm b}}_s(\tilde{X}_s))ds\|],
    \end{aligned}
    \end{equation}
    where 
    \begin{equation}
        \begin{aligned}
                & \quad \mathbb{E}[\|\int^t_{t_0}({\bm b}_s(X_s)-\tilde{{\bm b}}_s(\tilde{X}_s))ds\|] \\
                & \leq  \mathbb{E}[\|\int^t_{t_0}(-X_s - \beta(s) \nabla f(X_s) + {\bm u}_s(X_s))-(-\tilde{X_s} - \beta(s) \nabla f(\tilde{X_s}) +\tilde{{\bm u}}_s(\tilde{X_s}))ds\|]\\
                & =  \mathbb{E}[\|\int^t_{t_0}(-(X_s-\tilde{X_s}) - \beta(s) (\nabla f(X_s)-\nabla f(\tilde{X_s})) + ({\bm u}_s(X_s) -\tilde{{\bm u}}_s(\tilde{X}_s))ds\|]\\
                & \leq  \mathbb{E}[\int^t_{t_0}\|(-(X_s-\tilde{X_s}) - \beta(s) (\nabla f(X_s)-\nabla f(\tilde{X_s})) + ({\bm u}_s(X_s) -\tilde{{\bm u}}_s(\tilde{X}_s))\|ds] \\
                &  \leq  \mathbb{E}[\int^t_{t_0}\|X_s-\tilde{X_s}\| + \beta(s) \|\nabla f(X_s)-\nabla f(\tilde{X_s})\| + \|{\bm u}_s(X_s) -\tilde{{\bm u}}_s(\tilde{X}_s)\|ds] \\
                & \leq  \mathbb{E}[\int^t_{t_0}(1+\beta(s)L_f)\|X_s-\tilde{X_s}\|+\|{\bm u}_s(X_s)-\tilde{{\bm u}}_s(X_s)+\tilde{{\bm u}}_s(X_s)-\tilde{{\bm u}}_s(\tilde{X_s})\|ds]\\
                & \leq  \mathbb{E}[\int^t_{t_0}(1+\beta(s)L_f)\|X_s-\tilde{X_s}\|+\|{\bm u}_s(X_s)-\tilde{{\bm u}}_s(X_s)\|+\|\tilde{{\bm u}}_s(X_s)-\tilde{{\bm u}}_s(\tilde{X_s})\|ds] \\
                & \leq  \mathbb{E}[\int^t_{t_0}(1+\beta(s)L_f+L_u)\|X_s-\tilde{X_s}\|+\|{\bm u}_s(X_s)-\tilde{{\bm u}}_s(X_s)\|ds] \\
                & = \int^t_{t_0}\mathbb{E}[\|{\bm u}_s(X_s)-\tilde{{\bm u}}_s(X_s)\|]ds+\int^t_{t_0}(1+\beta(s)L_f+L_u)\mathbb{E}[\|Z_s\|]ds.
        \end{aligned}
    \end{equation}
    Given the initial condition and $\tilde{{\bm u}}$, $\epsilon_0 +\int^t_{t_0}\mathbb{E}[\|{\bm u}_s(X_s)-\tilde{{\bm u}}_s(X_s)\|]ds$ can be treated as a constant, then applying Grönwall's inequality yields
    \begin{equation*}
    \begin{aligned}
        \mathbb{E}[\|Z_t\|] &\leq C_{\epsilon}e^{B_{\epsilon}},  \\
        \text{where} \ C_{\epsilon} &= \epsilon_0 +\int^t_{t_0}\mathbb{E}[\|{\bm u}_s(X_s)-\tilde{{\bm u}}_s(X_s)\|]ds, \ \ 
        B_{\epsilon} & = \int^t_{t_0}(1+\beta(s)L_f+L_u)ds.
    \end{aligned}
    \end{equation*}
    This completes the proof.
\end{proof}
\cref{error} implies that given a fixed function $ f $, the error $\mathbb{E}[\|Z_t\|]$ is determined not only by the initial error $ \epsilon_0 $ and the approximation error $ \|{\bm u}_t(X_t) - \tilde{{\bm u}}_t(\tilde{X}_t)\| $, but also by $ L_u $. Clearly, given a fixed $t\ge t_0$, $\epsilon_0$ increases with the parameter $t_0$ ( i.e. $K$), while $ C_{\epsilon} $ and $ B_{\epsilon} $ decrease as $t_0$ increases. We will verify the impact of $K$ on the error bound empirically in Section~\ref{subsec:ablation}.

\subsection{Convergence Rates of PnP-Flow}
Proposition~\ref{prop4}---a convergence result established in \cite{martin2024pnp}---only provides convergence for PnP-Flow without the convergence rate. By \cref{prop4}, the sequence $\{{\bm x}_k\}$ generated by \cref{iter:PnPFM} with $\tilde{{\bm b}}$ converges \cite{martin2024pnp}. Let ${\bm x}^*$ be the limit of the sequence $\{{\bm x}_k\}$, $X_T \approx {\bm x}^*$. Since $\sum^{\infty}_{k=0}(1-l_k)\le T$, we have $1-l_k=0$ when $t\geq T$, indicating that $\Delta t = 0$, $\beta(t) = 0$, and $\sigma(t)=0$. Therefore, for $t\geq T$, $\tilde{{\bm b}}_T(X_T) dt + \sigma(t)dW_t = 0$.  Define the function $\mathcal{E}(t) := \|X_t-X_T\|^2$, where $X_T$ approximates ${\bm x}^*$ of \cref{iter:PnPFM} with $\tilde{{\bm b}}$, and $X_T$ is also an approximate solution to \cref{prob}. 
\begin{restatable}{theorem}{thmConvergence}\label{thm:convergence}
    Under \cref{ass1}, \ref{ass2}, and \ref{ass3}, the numerical solution $X_T$ of \cref{SDE} with the initial condition $X_{t_0}$ satisfies,
    \begin{equation}\label{convergence}
        \begin{aligned}
            & \mathbb{E}[\|X_t-X_T\|^2] \leq A + 2\mathbb{E}[\int^t_{t_0}B(s)\|X_s-X_T\|^2ds], \forall t\geq t_0,
        \end{aligned}
    \end{equation}
where $A = \mathbb{E}[\|X_{t_0}-X_T\|^2] + \int^t_{t_0}(n\sigma^2(s)+ \frac{M^2_f\beta(s)}{2\eta})ds$, $\eta>0$ and 
$$
B(s) = 
    \begin{cases}
           &   -1+L_u+\beta(s)L_f+\frac{\eta\beta(s)}{2}, \ \text{if $f(\cdot)$ is Lipschitz differentiable with $L_f$}. \\
           & -1+L_u+\frac{\eta\beta(s)}{2}, \ \hspace{1.5cm} \text{if $f(\cdot)$ is a convex function}. \\
           & -1+L_u-\beta(s)\mu_f+\frac{\eta\beta(s)}{2}, \ \text{if $f(\cdot)$ is a $\mu_f$-strongly convex function}.
    \end{cases}
$$
\end{restatable}
\begin{proof}
Given any function $\mathcal{J}(t) = \mathcal{E}(t,X_t)$, applying the Ito's lemma to the SDE with $\tilde{{\bm b}}$ in \cref{SDE} gives
\begin{equation*}
\begin{aligned}
        d(\mathcal{J}(t)) &= [\partial_t \mathcal{E}(t,X_t) + \nabla \mathcal{E}(t,X_t)\cdot\tilde{{\bm b}}_t(X_t)]dt + \sigma(t)\nabla \mathcal{E}(t,X_t)dW_t  + \frac{\sigma^2(t)}{2}\operatorname{Trace}\{\nabla^2\mathcal{E}(t,X_t)\}dt,
\end{aligned}
\end{equation*}
where the notation $\nabla$ and $\nabla^2$, respectively, denote the Jacobian matrix and the Hessian matrix with respect to the space variable $X_t$.

 Define the function $\mathcal{E}(t,X_t) := \|X_t-X_T\|^2$, the Ito's lemma gives
\begin{equation}
    \begin{aligned}
        d(\mathcal{E}(t,X_t)) &= 2\langle X_t-X_T, \tilde{{\bm b}}_t(X_t)\rangle dt + 2\sigma(t)\langle X_t-X_T, dW_t\rangle  + n\sigma^2(t)dt.
    \end{aligned}
\end{equation}
Substitute $\tilde{{\bm b}}_T(X_T) dt + \sigma(T)dW_T = 0$, we have
\begin{equation*}
    \begin{aligned}
        d(\mathcal{E}(t,X_t))  
        &= 2\langle X_t-X_T, \tilde{{\bm b}}_t(X_t)-\tilde{{\bm b}}_T(X_T)\rangle dt + 2(\sigma(t)-\sigma(T))\langle X_t-X_T,dW_t\rangle  + n\sigma^2(t)dt.
    \end{aligned}
\end{equation*}

Taking its integral form, we have
\begin{equation}
    \begin{aligned}
        \mathbb{E}[\mathcal{E}(t,X_t)] &= \mathbb{E}[\|X_{t_0}-X_T\|^2] + 2\int^t_{t_0} \langle X_s-X_T, \tilde{{\bm b}}_s(X_s) - \tilde{{\bm b}}_T(X_T)\rangle +n\sigma^2(s) ds \\
        & \quad + 2\int^t_{t_0} (\sigma(s)-\sigma(T))\langle X_s-X_T, dW_t\rangle ] \\
        & = \mathbb{E}[\|X_{t_0}-X_T\|^2] + \mathbb{E}[\int^t_{t_0} 2\langle X_s-X_T, \tilde{{\bm b}}_s(X_s) - \tilde{{\bm b}}_T(X_T)\rangle +n\sigma^2(s) ds] \\
        & = \mathbb{E}[\|X_{t_0}-X_T\|^2] + n\int^t_{t_0}\sigma^2(s)ds  + 2\mathbb{E}[\int^t_{t_0} \langle X_s-X_T, \tilde{{\bm b}}_s(X_s) - \tilde{{\bm b}}_T(X_T)\rangle  ds], 
    \end{aligned}
\end{equation}
and 

\begin{equation}
    \begin{aligned}
        & \langle X_s-X_T, \tilde{{\bm b}}_s(X_s) - \tilde{{\bm b}}_T(X_T)\rangle   \\
        =& \langle X_s-X_T, -X_s - \beta(s) \nabla f(X_s) + \tilde{{\bm u}}_s(X_s) +X_T + \beta(T) \nabla f(X_T) - \tilde{{\bm u}}_T(X_T)\rangle \\
        = & -\|X_s-X_T\|^2 - \beta(s)\langle X_s-X_T, \nabla f(X_s) - \nabla f(X_T)\rangle  - \beta(s)\langle X_s-X_T, \nabla f(X_T)\rangle \\
        & + \langle X_s-X_T, \tilde{{\bm u}}_s(X_s) - \tilde{{\bm u}}_T(X_T)\rangle  \\
        \le & -\|X_s-X_T\|^2 + \|X_s-X_T\|\|\tilde{{\bm u}}_s(X_s) - \tilde{{\bm u}}_T(X_T)\|\\
        & - \beta(s)\langle X_s-X_T, \nabla f(X_s) - \nabla f(X_T)\rangle - \beta(s)\langle X_s-X_T, \nabla f(X_T)\rangle \\ 
        \le & (-1+L_u)\|X_s-X_T\|^2 - \beta(s)\langle X_s-X_T, \nabla f(X_s) - \nabla f(X_T)\rangle \\
        & + \beta(s)\|X_s-X_T\|\|\nabla f(X_T)\|  \\
        \le & (-1+L_u)\|X_s-X_T\|^2 - \beta(s)\langle X_s-X_T, \nabla f(X_s) - \nabla f(X_T)\rangle \\
        & + \beta(s)(\frac{\eta}{2}\|X_s-X_T\|^2+\frac{M_f^2}{2\eta})  \\
    \end{aligned}
\end{equation}
where $\eta>0$ is an arbitrary positive number by Young's inequality. The upper bound of $\langle X_s-X_T, \tilde{{\bm b}}_s(X_s) - \tilde{{\bm b}}_T(X_T)\rangle $ depends on the property of $f(\cdot)$ as follows
\begin{enumerate}
    \item[(i)] When $f$ is nonconvex but Lipschitz differentiable with Lipschitz constant $L_f$ and have bounded gradient $\|\nabla f(X_t)\| \le M_f$, such that 
    \begin{equation*}
        -L_f\|X_s-X_T\|^2 \leq \langle X_s-X_T, \nabla f(X_s) - \nabla f(X_T)\rangle  \leq L_f\|X_s-X_T\|^2, 
    \end{equation*}
    then 
    \begin{equation}
    \begin{aligned}
        \langle X_s-X_T, \tilde{{\bm b}}_s(X_s) - \tilde{{\bm b}}_T(X_T)\rangle \leq (-1+L_u+\beta(s)L_f+\frac{\eta\beta(s)}{2})\|X_s-X_T\|^2 + \frac{M^2_f\beta(s)}{2\eta} 
    \end{aligned}
    \end{equation}
    
    \item[(ii)] When $f$ is a convex function such that 
    \begin{equation*}
        \langle X_s-X_T, \nabla f(X_s) - \nabla f(X_T)\rangle \geq 0,
    \end{equation*}
    we have 
    \begin{equation}
        \langle X_s-X_T, \tilde{{\bm b}}_s(X_s) - \tilde{{\bm b}}_T(X_T)\rangle \leq (-1+L_u + \frac{\eta\beta(s)}{2})\|X_s-X_T\|^2+ \frac{M^2_f\beta(s)}{2\eta} .
    \end{equation}

    \item[(iii)] When $f$ is a $\mu_f$-strongly convex function, such that
    \begin{equation*}
        \langle X_s-X_T, \nabla f(X_s) - \nabla f(X_T)\rangle \geq \mu_f\|X_s-X_T\|^2,
    \end{equation*}
    we get 
    \begin{equation}
        \langle X_s-X_T, \tilde{{\bm b}}_s(X_s) - \tilde{{\bm b}}_T(X_T)\rangle \leq (-1+L_u-\beta(s)\mu_f+\frac{\eta\beta(s)}{2})\|X_s-X_T\|^2 + \frac{M^2_f\beta(s)}{2\eta}
    \end{equation}
\end{enumerate}
In summary, let $B(s)$ be the coefficient of $\|X_s-X_T\|^2$, we have 
\begin{equation}\label{B(s)}
B(s) = 
    \begin{cases}
           &   -1+L_u+\beta(s)L_f+\frac{\eta\beta(s)}{2}, \ \text{if $f(\cdot)$ is Lipschitz differentiable with $L_f$}. \\
           & -1+L_u+\frac{\eta\beta(s)}{2}, \ \hspace{1.5cm} \text{if $f(\cdot)$ is a convex function}. \\
           & -1+L_u-\beta(s)\mu_f+\frac{\eta\beta(s)}{2}, \ \text{if $f(\cdot)$ is a $\mu_f$-strongly convex function}.
    \end{cases}
\end{equation}
It is evident that: i) when $f$ is not a strongly convex function, $0< L_u<1$ is a necessary condition for $B(s)<0$, $\forall s\in[t_0,t]$; ii) when $f$ is a strongly convex function, $B(s)<0$ as long as $0\le L_u < 1+\beta(s)\mu_f - \frac{\eta\beta(s)}{2}$. The strong convexity of $f$ allows $\tilde{{\bm u}}$ to be potentially non-contractive.

Let $A = \mathbb{E}[\|X_{t_0}-X_T\|^2] + \int^t_{t_0}(n\sigma^2(s)+ \frac{M^2_f\beta(s)}{2\eta})ds$, we have
\begin{equation}
    \begin{aligned}
        \mathbb{E}[\|X_t-X_T\|^2] &\leq A + 2\mathbb{E}[\int^t_{t_0}B(s)\|X_s-X_T\|^2ds].
    \end{aligned}
\end{equation}
The sign of $B(s)$ depends on $L_u$, $\beta(s)$, and the properties of $f(\cdot)$. However, for fixed $L_u$ and $f(\cdot)$, since $\beta(s)$ is monotonically non-increasing, $B(s)$ is also a monotonic function. $B(s)$ is monotonically non-increasing when $f$ is a Lipschitz differentiable function or a convex function. 

This completes the proof.
\end{proof}
According to \cref{convergence}, the convergence rate is directly influenced by $\sigma(t)$ and $\beta(t)$, both of which are controlled by the step schedule $1 - l_k$, as well as by the Lipschitz constant $L_u$ and the initial condition $X_{t_0}$.

\section{SDE Informed Improvement for PnP-Flow}\label{sec:Algorithm}
In this section, we present a few strategies to improve PnP-Flow based on the theoretical results established in Section~\ref{sec:Theory}.

\subsection{A New Schedule for $1-l_k$ }\label{sec:lk}
The convergence of both \cref{iter:PnPFM} and \cref{iter:acc}---our proposed accelerated PnP-Flow---relies on the assumption that 
\(
    \sum_{k=0}^\infty (1 - l_k) < +\infty,
\)
which requires $\{1 - l_k\}$ to be a convergent sequence. However, in the implementation of \cite{martin2024pnp}, $l_k$ is defined as $l_k = \frac{k}{N}$, where $N$ is the total number of iterations. Although such an implementation provides good numerical results, this choice leads to $1 - l_k = \frac{N - k}{N} \geq \frac{1}{N}$, $\forall k\le N$,
where $\sum_{N=1}^{\infty}\frac{1}{N}$ diverges---contradicting the convergence assumption.

To address this issue, we propose an alternative form for $l_k$ that ensures $\sum_{k=0}^\infty (1 - l_k)$ converges:
\begin{equation}\label{eq:lk}
    l_k = 1 - \lambda^k, \quad \lambda \in (0, 1).
\end{equation}
Here, the sequence $\{1 - l_k\} = \{\lambda^k\}$ is geometrically decreasing, and its sum evaluates to:
\(
    \sum_{k=0}^\infty \lambda^k = \frac{1}{1 - \lambda} < +\infty.
\)
The selection of an appropriate $\lambda$ depends on the total iteration count $N$, and $\lambda^k$ converges to zero as $k \to N$. Excessively small $\lambda$ results in a large step size $1-l_k$, and it decays too rapidly to zero, increasing the local truncation error, as well as exacerbating the accumulation of numerical errors. For overly large $\lambda$, the final term $\lambda^N$ fails to approach zero, violating our convergence conditions. Furthermore, our proposed schedule strategically allocates more iterations to smaller $1-l_k$ values, enhancing the generation of fine details.

\subsection{Lipschitz Regularization}
In \cref{thm:error}, for a fixed $ t_0 $, $ B_{\epsilon} $ increases with $ L_u $, indicating that penalizing the Lipschitz constant of $\tilde{{\bm u}}_t$ helps suppress the overall error. The idea of penalizing the Lipschitz constant of NNs has been widely explored (e.g., in training Wasserstein GANs~\cite{arjovsky2017wasserstein, gulrajani2017improved, miyato2018spectral}). In our experiments, we adopt the approach from~\cite{gulrajani2017improved}. The output of the vector field regressor $\tilde{{\bm u}}_t(\tilde{X}_t)$ has the same dimension as its input, whereas the discriminator network in WGAN~\cite{gulrajani2017improved} produces a scalar output. Consequently, we penalize the Jacobian norm instead of the gradient norm for $\tilde{{\bm u}}_t(\tilde{X}_t)$, and computing the Jacobian of $\tilde{{\bm u}}_t(X_t)$ requires $n$ gradient evaluations, where $n$ is the data dimension. To improve computational and memory efficiency, we employ Hutchinson's method~\cite{hutchinson1989stochastic, lu2022maximum, huang2024efficient} to estimate the Frobenius norm of the Jacobian $\nabla \tilde{{\bm u}}_t(\tilde{X}_t)$ whose supremum provides an upper bound of the Lipschitz constant. This approach provides an unbiased estimator; see Appendix~\ref{sec:appendix_experiment} for details.

\subsection{Acceleration via Extrapolation}
Now we consider accelerating off-the-shelf PnP-Flow models without retraining. In particular, we consider an accelerated version of \cref{iter:PnPFM} via extrapolation, which is given as follows:
\begin{equation}\label{iter:acc}
    \begin{cases}
        {\bm w}_{k} = {\bm x}_{k} + h_{k}({\bm x}_{k}-{\bm x}_{k-1}), \\
        {\bm z}_{k} = {\bm w}_k - \gamma_k \nabla f({\bm w}_k), \\
        {\bm y}_k = (1 - l_k) \boldsymbol{\xi} + l_k {\bm z}_{k}, \\
        {\bm x}_{k+1} = D_{l_k}({\bm y}_k), \\  
    \end{cases}
\end{equation}
where $\gamma_k, l_k, h_{k}\in (0,1)$ are manually set parameters. We stress that the vector field in $D_{l_k}$ is inherited from \cref{iter:PnPFM} without restraining or fine-tuning. In the rest of this subsection, we will study the convergence rate of \cref{iter:acc} and compare it against that of \cref{iter:PnPFM}.

\begin{restatable}{proposition}{propCon}\label{prop:con}
Assume that $f: \mathbb{R}^n \to \mathbb{R}$ is a differentiable function with bounded gradient, and the learned vector field $\tilde{{\bm u}}_t: [0,1] \times \mathbb{R}^n \to \mathbb{R}^n$ is bounded and Lipschitz. Let the time sequence $\{l_k\}_{k \in \mathbb{N}}$ satisfy $l_k\in [0,1]$, $\sum\limits_{k=0}^\infty (1 - l_k) <  +\infty$, and let $\gamma_k \leq 1 - l_k$. If the sequence $\{{\bm x}_k\}_{k \in \mathbb{N}}$ obtained by \cref{iter:acc} is bounded, then  $h_k\in [0,\zeta]$ for some $\zeta\in(0,1)$ is a necessary condition for the convergence of the sequence $\{{\bm x}_k\}_{k \in \mathbb{N}}$.
\end{restatable}
\begin{proof}
    By the iteration \cref{iter:acc} and the definition of $D_{l_k}$, we have
    \begin{equation}
        \begin{aligned}
            & \quad \ {\bm x}_{k+1} -{\bm x}_k \\
            &= D_{l_k}({\bm y}_k) - {\bm x}_k \\
            & = {\bm y}_k + (1-l_k){\bm u}_{l_k}({\bm y}_k)-{\bm x}_k \\
            & = (1 - l_k) \boldsymbol{\xi} + l_k ({\bm w}_k - \gamma_k \nabla f({\bm w}_k))+ (1-l_k){\bm u}_{l_k}({\bm y}_k)-{\bm x}_k \\ 
            & = (1 - l_k)(\boldsymbol{\xi}-{\bm x}_k)+l_kh_k({\bm x}_{k}-{\bm x}_{k-1})-l_k\gamma_k\nabla f({\bm w}_k) + (1-l_k){\bm u}_{l_k}({\bm y}_k).
        \end{aligned}
    \end{equation}
    Then
    \begin{equation}
        \begin{aligned}
                &\quad \ \|{\bm x}_{k+1}-{\bm x}_{k}\| \\
                &= \|(1 - l_k)(\boldsymbol{\xi}-{\bm x}_k)+l_kh_k({\bm x}_{k}-{\bm x}_{k-1})-l_k\gamma_k\nabla f({\bm w}_k) + (1-l_k){\bm u}_{l_k}({\bm y}_k) \| \\
                &\le (1-l_k)(\|\boldsymbol{\xi}-{\bm x}_k+{\bm u}_{l_k}({\bm y}_k)\|+l_k\|\nabla f({\bm w}_k)\|)+ l_kh_k\|{\bm x}_{k}-{\bm x}_{k-1}\| \\
                &\le (1-l_k)(\|\boldsymbol{\xi}-{\bm x}_k+{\bm u}_{l_k}({\bm y}_k)\|+\|\nabla f({\bm w}_k)\|)+ h_k\|{\bm x}_{k}-{\bm x}_{k-1}\|
        \end{aligned}
    \end{equation}
    Since $\boldsymbol{\xi}, {\bm x}_k, {\bm u}_{l_k}({\bm y}_k)$ are all bounded by the assumption, there exists a constant $M>0$, such that 
    \begin{equation}
        \|{\bm x}_{k+1}-{\bm x}_{k}\| \le (1-l_k)M + h_k\|{\bm x}_{k}-{\bm x}_{k-1}\|.
    \end{equation}
    Therefore,
    \begin{equation}
        \|{\bm x}_{k+1}-{\bm x}_{k}\| + \sum\limits_{k=1}^{\infty}(1-h_k)\|{\bm x}_{k}-{\bm x}_{k-1}\| \le M\sum\limits_{k=1}^{\infty}(1-l_k) < +\infty.
    \end{equation}
    The above inequality is trivial when $h_k  \ge 1$. When $h_k\in[0,\zeta]$ for some $\zeta\in (0,1)$ and $\sum\limits_{k=1}^{\infty}(1-l_k)<+\infty$, we have
    \begin{equation}
        \sum\limits_{k=1}^{\infty} \|{\bm x}_{k+1}-{\bm x}_{k}\| \le \frac{M}{1-\zeta} \sum\limits_{k=1}^{\infty}(1-l_k) + \zeta\|{\bm x}_{k+1}-{\bm x}_{k}\| < +\infty,
    \end{equation}
    which indicates that $\{{\bm x}_k\}$ is a Cauchy sequence and converges.
    This completes the proof.
\end{proof}

\begin{remark}
The extrapolation coefficient $h_k$ in proximal gradient methods has been extensively studied in both convex and nonconvex settings \cite{pock2016inertial,wu2024extrapolated,xu2013block,iutzeler2018proximal}. Since \cref{iter:acc} can be interpreted as an inexact proximal gradient descent iteration in a nonconvex setting, the admissible range of \( h_k \) is more restricted than in convex scenarios. In the experimental section, we start with a small value of $h_k$ and gradually increase it until the iteration fails to converge.
\end{remark}

\begin{restatable}{proposition}{propASDE}\label{prop:ASDE}
The SDE counterpart of \cref{iter:acc} is given by the following SDE:
\begin{equation}\label{ASDE}
        dX_t = \frac{{\bm b}_t(X_t)}{1-\alpha(t)}dt + \frac{\sigma(t)}{1-\alpha(t)}dW_t,
\end{equation}
where ${\bm b}_t(X_t) = -X_t - \beta(t) \nabla f(X_t) + {\bm u}_t(X_t)$, $\alpha(t)\ge 0$, $\sigma(t)\in [0,1)$. \cref{ASDE} is in fact a rescaled version of \cref{SDE} with the scaling factor $\frac{1}{1-\alpha(t)}$. \cref{iter:acc} is more precisely approximated by \cref{ASDE} as the dynamic step size $1-l_k$ goes to zero.
\end{restatable}
\begin{proof} 
    Clearly, we have
    \begin{equation}\label{eq:difx}
        \begin{aligned}
            &\quad \ {\bm x}_{k+1} -{\bm x}_k  \\
            & = (1 - l_k)(\boldsymbol{\xi}-{\bm x}_k)+l_kh_k({\bm x}_{k}-{\bm x}_{k-1})-l_k\gamma_k\nabla f({\bm w}_k) + (1-l_k){\bm u}_{l_k}({\bm y}_k),
        \end{aligned}
    \end{equation}
    and 
    \begin{equation}\label{eq:nablaf}
        \begin{aligned}
            \nabla f({\bm w}_k) = \nabla f({\bm x}_k) +\nabla^2f({\bm x}_k)h_k({\bm x}_{k}-{\bm x}_{k-1}) + o(h_k({\bm x}_{k}-{\bm x}_{k-1})),
        \end{aligned}
    \end{equation}
    
    \begin{equation}\label{eq:u}
        \begin{aligned}
            {\bm u}_{l_k}({\bm y}_k) &= {\bm u}_{l_k}({\bm x}_k) + \nabla {\bm u}_{l_k}({\bm x}_k)({\bm y}_k -{\bm x}_k) + o({\bm y}_k -{\bm x}_k),  \\
            \text{where}  \ {\bm y}_k -{\bm x}_k & =  (1 - l_k) \boldsymbol{\xi} + l_k{\bm z}_k - {\bm x}_k \\
            & = (1 - l_k)(\boldsymbol{\xi}-{\bm x}_k) + l_kh_k({\bm x}_k-{\bm x}_{k-1}))-l_k\gamma_k\nabla f({\bm w}_k).
        \end{aligned}
    \end{equation}
    When selecting sufficiently small $\gamma_k, h_k$ such that $(1-l_k)({\bm y}_k -{\bm x}_k) = o(1-l_k) $ and $\gamma_kh_k = o(1-l_k)$, then combine \cref{eq:difx,eq:nablaf,eq:u} together, we have
    \begin{equation}
        \begin{aligned}
             & \quad {\bm x}_{k+1} -{\bm x}_k  \\
            & = (1 - l_k)(\boldsymbol{\xi}-{\bm x}_k)+l_kh_k({\bm x}_{k}-{\bm x}_{k-1})-l_k\gamma_k\nabla f({\bm w}_k) + (1-l_k){\bm u}_{l_k}({\bm y}_k) \\
            & = (1 - l_k)(\boldsymbol{\xi}-{\bm x}_k)+l_kh_k({\bm x}_{k}-{\bm x}_{k-1}) - l_k\gamma_k(\nabla f({\bm x}_k) +\nabla^2f({\bm x}_k)h_k({\bm x}_{k}-{\bm x}_{k-1})) \\
            &\quad - l_k\gamma_ko(h_k({\bm x}_{k}-{\bm x}_{k-1}))) + (1-l_k)({\bm u}_{l_k}({\bm x}_k) + \nabla {\bm u}_{l_k}({\bm x}_k)({\bm y}_k -{\bm x}_k) + o({\bm y}_k -{\bm x}_k)) \\
            & = (1 - l_k)(\boldsymbol{\xi}-{\bm x}_k)+l_kh_k({\bm x}_{k}-{\bm x}_{k-1}) - l_k\gamma_k\nabla f({\bm x}_k)+(1-l_k){\bm u}_{l_k}({\bm x}_k) +o(1-l_k).
        \end{aligned}
    \end{equation}
    
    Set $\beta_k  = \frac{l_k \gamma_k}{1 - l_k}$, we have 
    \begin{equation}
        {\bm x}_{k+1} -{\bm x}_k = (1 - l_k)(\boldsymbol{\xi}-{\bm x}_k - \beta_k\nabla f({\bm x}_k)+{\bm u}_{l_k}({\bm x}_k))+l_kh_k({\bm x}_{k}-{\bm x}_{k-1}) +o(1-l_k)
    \end{equation}
    
    Again, introducing \textit{Ansatz} $ {\bm x}_k \approx X_t $ where $X_t $ is a smooth curve defined for $ t \geq 0 $. As the step size $ 1 - l_k \to 0 $, we have $ X_t \approx {\bm x}_k $, $ X(t + 1-l_k) \approx {\bm x}_{k+1}$, $ X(t - (1-l_{k-1})) \approx {\bm x}_{k-1}$. And Taylor expansion gives
    \begin{equation}
        \begin{aligned}
            {\bm x}_{k+1}-{\bm x}_{k}  &=  (1-l_k)\dot{X}(t) + o(1-l_k) \\
            {\bm x}_{k}-{\bm x}_{k-1}  &= (1-l_{k-1})\dot{X}(t) + o(1-l_{k-1})
        \end{aligned}
    \end{equation}
    Therefore, 
    \begin{equation}
    \begin{aligned}
            &\quad ((1-l_k)-l_kh_k(1-l_{k-1}))\dot{X}(t) \\
            & = (1 - l_k)(\boldsymbol{\xi}-X(t) - \beta_k\nabla f(X(t))+{\bm u}_{l_k}(X(t)))+o(1-l_k) +o(1-l_{k-1})\\
    \end{aligned}
    \end{equation}
    Set $\alpha_k = \frac{l_kh_k(1-l_{k-1})}{1-l_k}$, when $1-l_k\to 0$, we get 
    \begin{equation}
        \begin{aligned}
            (1-\alpha_k)\dot{X}(t) =  \boldsymbol{\xi}-X(t) - \beta_k\nabla f(X(t))+{\bm u}_{l_k}(X(t)).
        \end{aligned}
    \end{equation}
    Replacing $\alpha_k, \beta_k, X(t)$ with $\alpha(t), \beta(t), X_t$, and introducing $\sigma(t)$ yields
    \begin{equation}
        dX_t = \frac{{\bm b}_t(X_t)}{1-\alpha(t)}dt + \frac{\sigma(t)}{1-\alpha(t)}dW_t,
    \end{equation}
    where ${\bm b}_t(X_t) = -X_t - \beta(t) \nabla f(X_t) + {\bm v}(t,X_t)$. The above SDE is in fact a recaled version of \cref{SDE} with the scaling factor $\frac{1}{1-\alpha(t)}$. $\frac{1}{1-\alpha(t)}$ increase with $h_k$, however, to ensure the convergence of the sequence $\{{\bm x}_k\}$, $h_k$ can be sufficiently small but it is upper bounded.
    
    This completes the proof.
\end{proof}
\begin{remark}
    Since we set $\alpha_k = \frac{l_kh_k(1-l_{k-1})}{1-l_k}$, the range of $\alpha_k$ directly depends on the value of $h_k$ and the schedule for $1-l_k$. In the setting in \cref{eq:lk}, we see that $\alpha_k = \frac{1-\lambda^k}{\lambda}h_k$. When selecting proper $0\ll \lambda <1$ and $0\le h_k\ll 1$, we can ensure that $0\le \alpha_k < 1$.
\end{remark}

\begin{restatable}{theorem}{thmAccCon}\label{thm:acc_con}
    Under \cref{ass1} and \ref{ass2}, the solution $X_T$ of \cref{ASDE} with the initial condition $X_{t_0}$ satisfies,
    \begin{equation}\label{acc_convergence}
        \begin{aligned}
            & \mathbb{E}[\|X_t-X_T\|^2] &\leq A_{\alpha} + 2\mathbb{E}[\int^t_{t_0}\frac{B(s)}{1-\alpha(s)}\|X_s-X_T\|^2ds],\ \ \forall t\geq t_0,
        \end{aligned}
    \end{equation}
where $A_{\alpha} = \mathbb{E}[\|X_{t_0}-X_T\|^2] + \int^t_{t_0}\frac{n\sigma^2(s)+ \frac{M^2_f\beta(s)}{2\eta}}{(1-\alpha(s))^2}ds$ and 
$$
B(s) = 
    \begin{cases}
           &   -1+L_u+\beta(s)L_f+\frac{\eta\beta(s)}{2}, \ \text{if $f(\cdot)$ is Lipschitz differentiable with $L_f$}. \\
           & -1+L_u+\frac{\eta\beta(s)}{2}, \ \hspace{1.5cm} \text{if $f(\cdot)$ is a convex function}. \\
           & -1+L_u-\beta(s)\mu_f+\frac{\eta\beta(s)}{2}, \ \text{if $f(\cdot)$ is a $\mu_f$-strongly convex function}.
    \end{cases}
$$
\end{restatable}
\begin{proof}
 Define the function $\mathcal{E}(t,X_t) := \|X_t-X_T\|^2$, taking the integral form, we have
\begin{equation}
    \begin{aligned}
        \mathbb{E}[\mathcal{E}(t,X_t)] 
        & = \mathbb{E}[\|X_{t_0}-X_T\|^2] + \int^t_{t_0}\frac{n\sigma^2(s)}{(1-\alpha(s))^2}ds \\
        & \quad + 2\mathbb{E}[\int^t_{t_0} \frac{1}{1-\alpha(s)}\langle X_s-X_T, \tilde{{\bm b}}_s(X_s) - \tilde{{\bm b}}_T(X_T)\rangle  ds], 
    \end{aligned}
\end{equation}
and 
\begin{equation}
    \begin{aligned}
        & \langle X_s-X_T, \tilde{{\bm b}}_s(X_s) - \tilde{{\bm b}}_T(X_T)\rangle   \\
        \le &(-1+L_u)\|X_s-X_T\|^2 - \beta(s)\langle X_s-X_T, \nabla f(X_s) - \nabla f(X_T)\rangle \\
        & + \beta(s)(\frac{\eta}{2}\|X_s-X_T\|^2+\frac{M_f^2}{2\eta}) .
    \end{aligned}
\end{equation}
Let $A_{\alpha} = \mathbb{E}[\|X_{t_0}-X_T\|^2] + \int^t_{t_0}\frac{n\sigma^2(s)+ \frac{M^2_f\beta(s)}{2\eta}}{(1-\alpha(s))^2}ds$ we have
\begin{equation}
    \begin{aligned}
        \mathbb{E}[\|X_t-X_T\|^2] &\leq A_{\alpha} + 2\mathbb{E}[\int^t_{t_0}\frac{B(s)}{1-\alpha(s)}\|X_s-X_T\|^2ds].
    \end{aligned}
\end{equation}

This completes the proof.
\end{proof}
The upper bound in \cref{convergence} becomes nontrivial when $\int_{t_0}^t B(s)\|X_s-X_T\|^2\,ds < 0$. Under the same setting, when \( \alpha(t) \) is a constant \( \alpha \in [0,1) \) , it is easy to verify that $\frac{1}{1 - \alpha} \int_{t_0}^t B(s)\|X_s-X_T\|^2\,ds \le \int_{t_0}^t B(s)\|X_s-X_T\|^2\,ds$, since \( \frac{1}{1 - \alpha} \ge 1 \). This shows that the extrapolation step in \cref{iter:acc} permits a wider admissible range of values for \( B(s) \) such that the upper bound remains nontrivial, thereby relaxing the conditions required on \( f \). In practice, we can select proper $\beta(t)$ and $\alpha(t)$ such that \( \frac{B(s)}{1 - \alpha(s)} \) is the dominant term determining the magnitude of the upper bound, as well as the convergence rate, leading to smaller $\|X_t-X_T\|^2$.

\section{Numerical Experiments}\label{sec:Experiments}
In this section, we validate the performance of our improved PnP-Flow (IPnP-Flow) method---as laid out in Section~\ref{sec:Algorithm}---using several benchmark image restoration tasks. The numerical experiments also aim to solidify our established theoretical analysis of PnP-Flow. We compare IPnP-Flow with baseline methods for various image restoration tasks in Section~\ref{subsec:exp-results}, and we perform ablation studies to investigate the effects of each additional component of IPnP-Flow over PnP-Flow in Section~\ref{subsec:ablation}. Addtional experimental details and results are presented in \cref{sec:appendix_experiment}.

\textbf{Experiment Setup:} Our implementation is based on the codes provided in the paper \cite{martin2024pnp} using PyTorch. All experiments are conducted on NVIDIA RTX 4090 GPUs. Training details and hyperparameters for different tasks are available in the \cref{sec:appendix_experiment}. For each task, we conduct five independent runs with different random seeds and report the mean performance ± standard deviation. 

\textbf{Image Restoration Tasks:} The experiments focus on denoising, deblurring, super-resolution, and inpainting tasks, conducted on two widely adopted datasets: CelebA \cite{yang2015facial} and AFHQ-Cat \cite{choi2020stargan}. To ensure a fair comparison with PnP-Flow (the baseline our work primarily improves upon), we adopt identical experimental settings from \cite{martin2024pnp}, including dataset splits (training/validation/test), model architectures, task settings, and evaluation protocols.

\textbf{Baseline Methods:} We compare our approach against the original PnP-Flow \cite{martin2024pnp} and other SOTA methods, including OT-ODE \cite{pokle2023training}, D-Flow \cite{ben2024d}, Flow priors \cite{zhang2024flow}, PnP-Diff \cite{zhu2023denoising}, and PnP-FBS \cite{hurault2021gradient}, demonstrating improvements in restoration quality and efficiency.

\subsection{Image Restoration Results}\label{subsec:exp-results}

\begin{figure}[!ht]
    \centering

    \begin{minipage}[b]{0.13\textwidth}
        \centering
        \small Clean
    \end{minipage}
    \begin{minipage}[b]{0.13\textwidth}
        \centering
        \small Degraded
    \end{minipage}
    \begin{minipage}[b]{0.13\textwidth}
        \centering
        \small PnP-GS
    \end{minipage}
    \begin{minipage}[b]{0.13\textwidth}
        \centering
        \small OT-ODE
    \end{minipage}
    \begin{minipage}[b]{0.13\textwidth}
        \centering
        \small Flow-Priors
    \end{minipage}
    \begin{minipage}[b]{0.13\textwidth}
        \centering
        \small PnP-Flow
    \end{minipage}
    \begin{minipage}[b]{0.13\textwidth}
        \centering
        \small Ours
    \end{minipage}

\noindent\begin{minipage}[t]{0.13\textwidth}
  \vspace{0pt}  
  \centering
  \includegraphics[width=\linewidth]{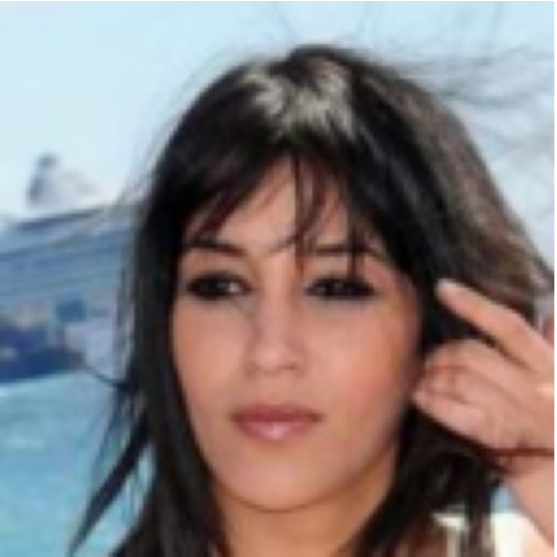}\\
  \makebox[\linewidth][c]{\tiny Denoising}
\end{minipage}%
\hspace{0.01cm}
\begin{minipage}[t]{0.13\textwidth}
  \vspace{0pt}
  \centering
  \includegraphics[width=\linewidth]{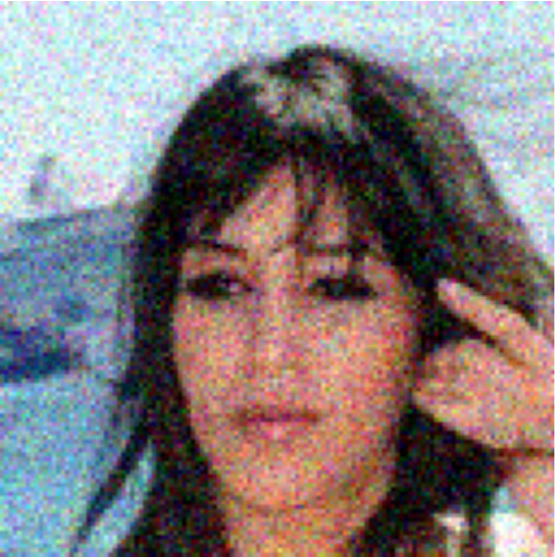}\\
  \noindent\makebox[\linewidth][c]{\tiny PSNR: 19.97}
\end{minipage}%
\hspace{0.01cm}
\begin{minipage}[t]{0.13\textwidth}
  \vspace{0pt}
  \centering
  \includegraphics[width=\linewidth]{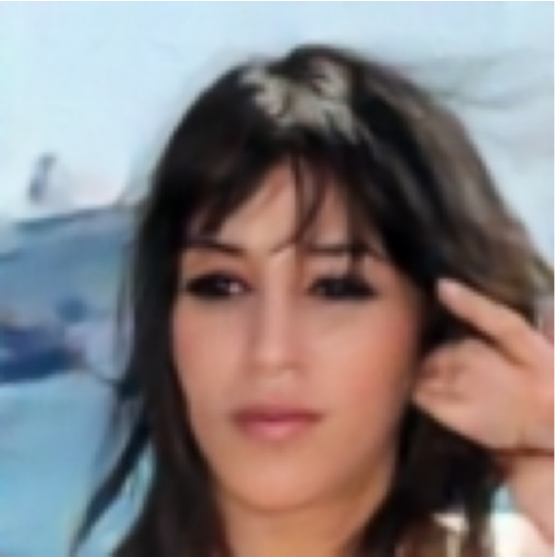}\\
  \noindent\makebox[\linewidth][c]{\tiny PSNR: 31.86}
\end{minipage}%
\hspace{0.01cm}
\begin{minipage}[t]{0.13\textwidth}
  \vspace{0pt}
  \centering
  \includegraphics[width=\linewidth]{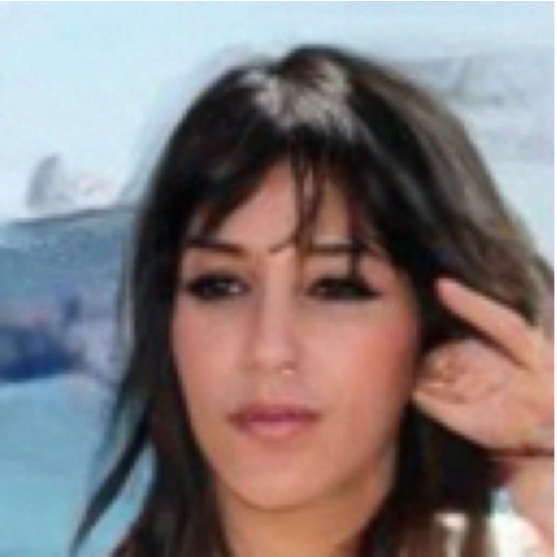}\\
  \noindent\makebox[\linewidth][c]{\tiny PSNR: 29.48}
\end{minipage}%
\hspace{0.01cm}
\begin{minipage}[t]{0.13\textwidth}
  \vspace{0pt}
  \centering
  \includegraphics[width=\linewidth]{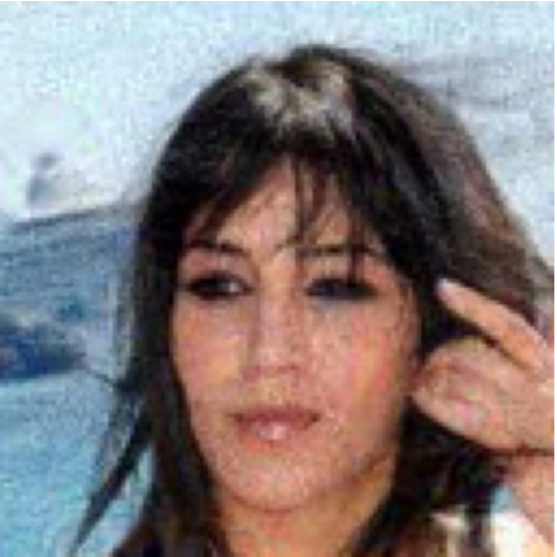}\\
  \noindent\makebox[\linewidth][c]{\tiny PSNR: 28.30}
\end{minipage}%
\hspace{0.01cm}
\begin{minipage}[t]{0.13\textwidth}
  \vspace{0pt}
  \centering
  \includegraphics[width=\linewidth]{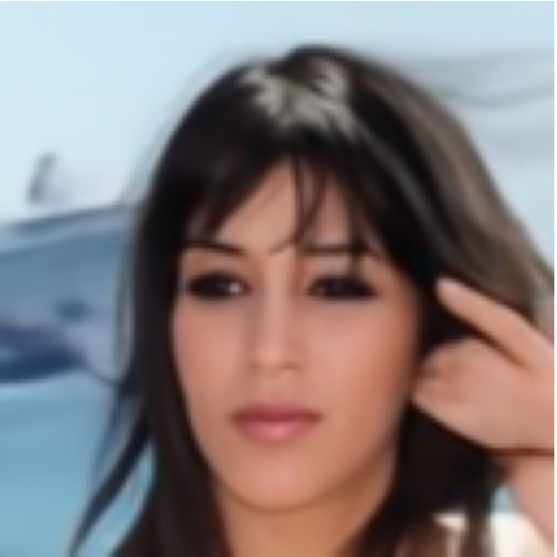}\\
  \noindent\makebox[\linewidth][c]{\tiny PSNR: 30.95}
\end{minipage}%
\hspace{0.01cm}
\begin{minipage}[t]{0.13\textwidth}
  \vspace{0pt}
  \centering
  \includegraphics[width=\linewidth]{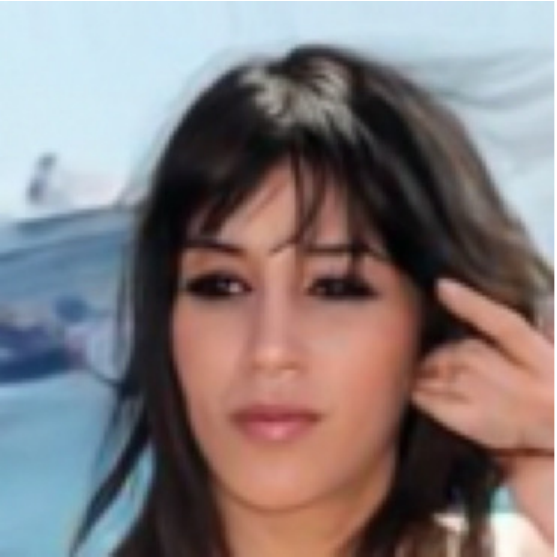}\\
  \noindent\makebox[\linewidth][c]{\tiny PSNR: \textbf{32.25}}
\end{minipage}


    \begin{minipage}[b]{0.13\textwidth}
    \captionsetup{skip=-0.11cm}
        \includegraphics[width=\linewidth]{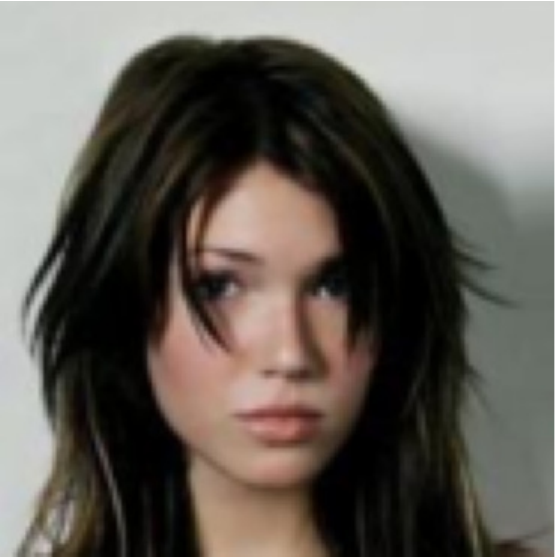}
       \makebox[\linewidth][c]{\tiny Deblurring }
    \end{minipage}
    \begin{minipage}[b]{0.13\textwidth}
    \captionsetup{skip=-0.11cm}
        \includegraphics[width=\linewidth]{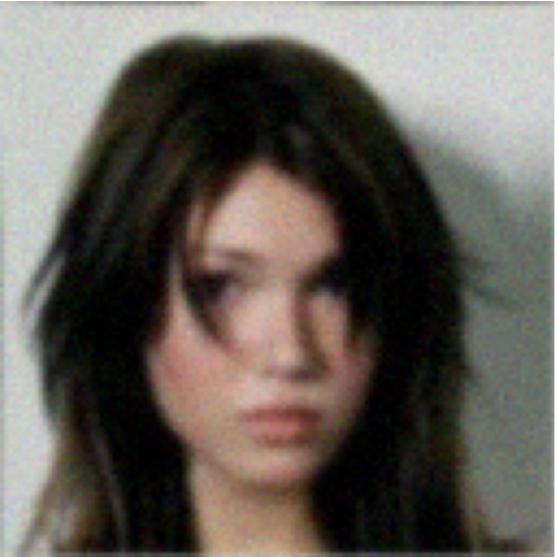}
       \makebox[\linewidth][c]{\tiny PSNR: 26.75}
    \end{minipage}
    \begin{minipage}[b]{0.13\textwidth}
    \captionsetup{skip=-0.11cm}
        \includegraphics[width=\linewidth]{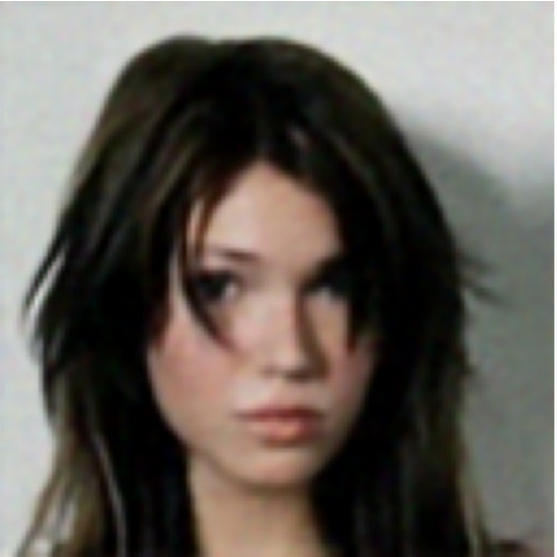}
       \makebox[\linewidth][c]{\tiny PSNR:  34.33}
    \end{minipage}
    \begin{minipage}[b]{0.13\textwidth}
    \captionsetup{skip=-0.11cm}
        \includegraphics[width=\linewidth]{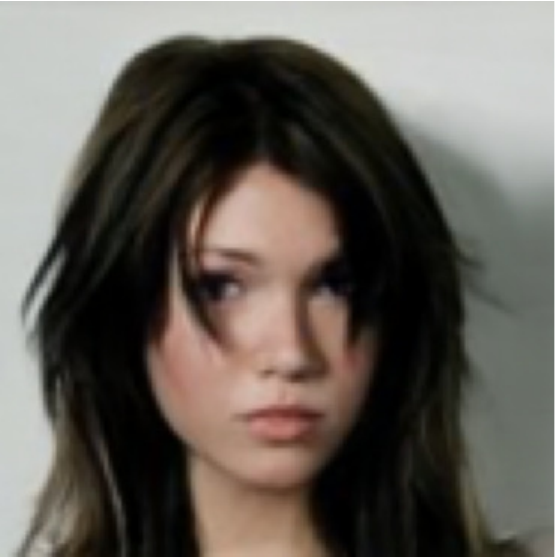}
       \makebox[\linewidth][c]{\tiny PSNR:  33.33}
    \end{minipage}
    \begin{minipage}[b]{0.13\textwidth}
    \captionsetup{skip=-0.11cm}
        \includegraphics[width=\linewidth]{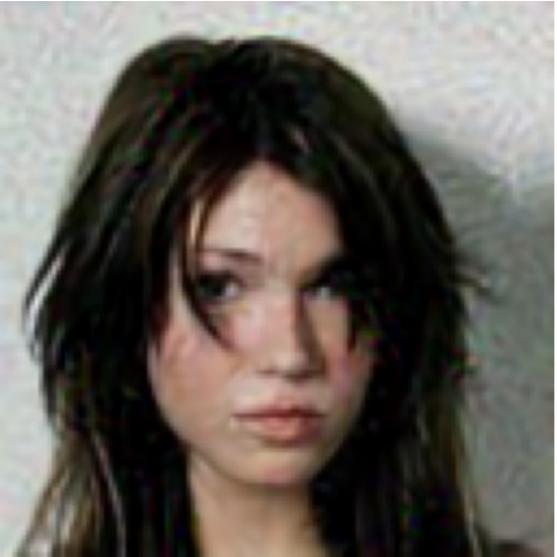}
       \makebox[\linewidth][c]{\tiny PSNR:  31.59}
    \end{minipage}
    \begin{minipage}[b]{0.13\textwidth}
    \captionsetup{skip=-0.11cm}
        \includegraphics[width=\linewidth]{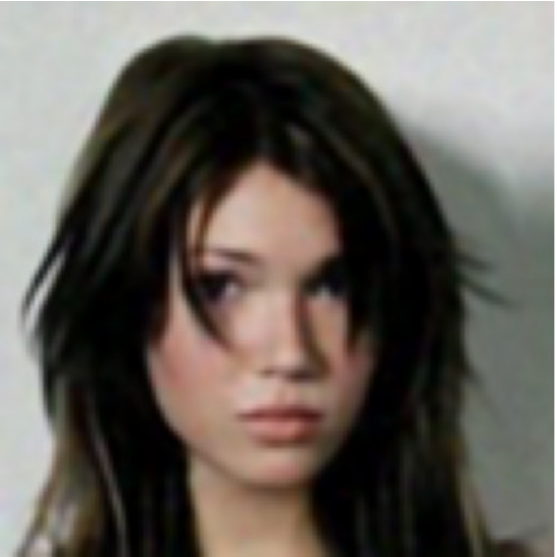}
       \makebox[\linewidth][c]{\tiny PSNR:  34.93}
    \end{minipage}
    \begin{minipage}[b]{0.13\textwidth}
    \captionsetup{skip=-0.11cm}
        \includegraphics[width=\linewidth]{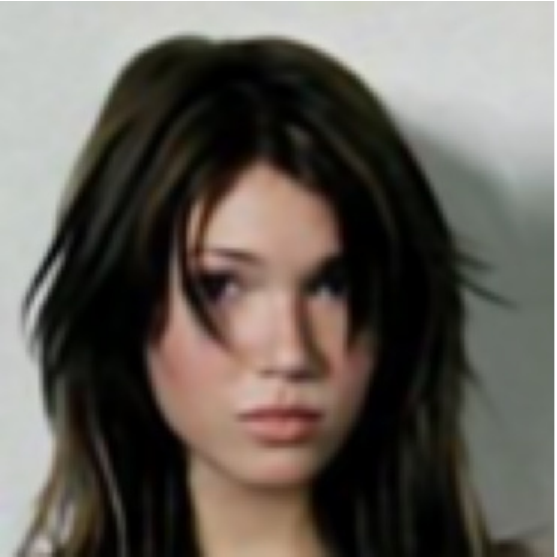}
       \makebox[\linewidth][c]{\tiny PSNR: \textbf{35.98}}
    \end{minipage}


    \begin{minipage}[b]{0.13\textwidth}
    \captionsetup{skip=-0.11cm}
        \includegraphics[width=\linewidth]{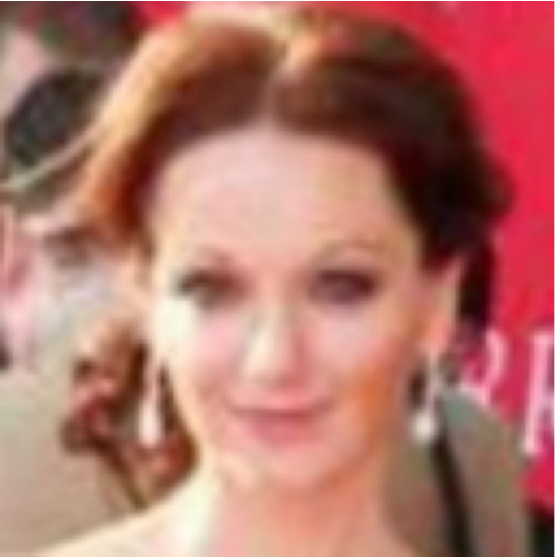}
       \makebox[\linewidth][c]{\tiny Super-resolution }
    \end{minipage}
    \begin{minipage}[b]{0.13\textwidth}
    \captionsetup{skip=-0.11cm}
        \includegraphics[width=\linewidth]{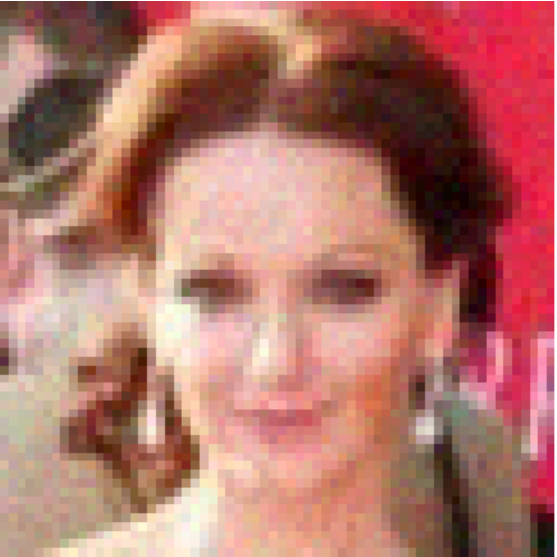}
       \makebox[\linewidth][c]{\tiny PSNR:  8.98}
    \end{minipage}
    \begin{minipage}[b]{0.13\textwidth}
    \captionsetup{skip=-0.11cm}
        \includegraphics[width=\linewidth]{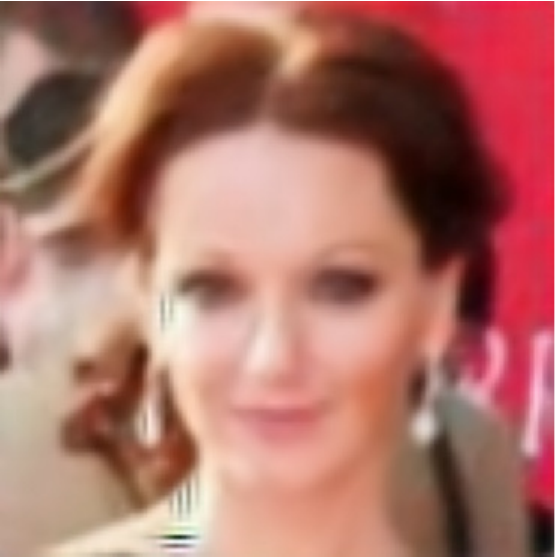}
       \makebox[\linewidth][c]{\tiny PSNR:  30.57}
    \end{minipage}
    \begin{minipage}[b]{0.13\textwidth}
    \captionsetup{skip=-0.11cm}
        \includegraphics[width=\linewidth]{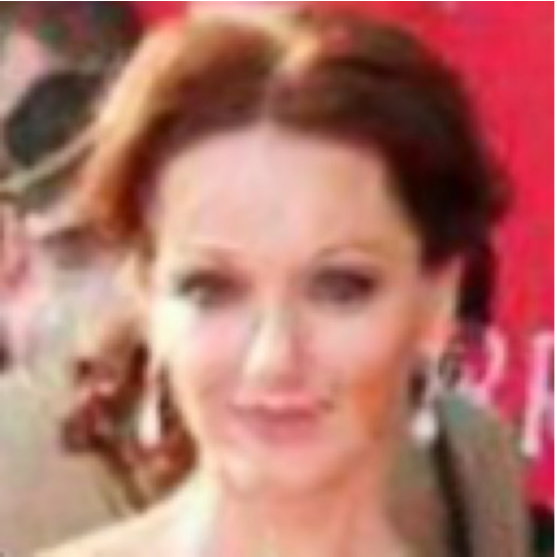}
       \makebox[\linewidth][c]{\tiny PSNR:  34.31}
    \end{minipage}
    \begin{minipage}[b]{0.13\textwidth}
    \captionsetup{skip=-0.11cm}
        \includegraphics[width=\linewidth]{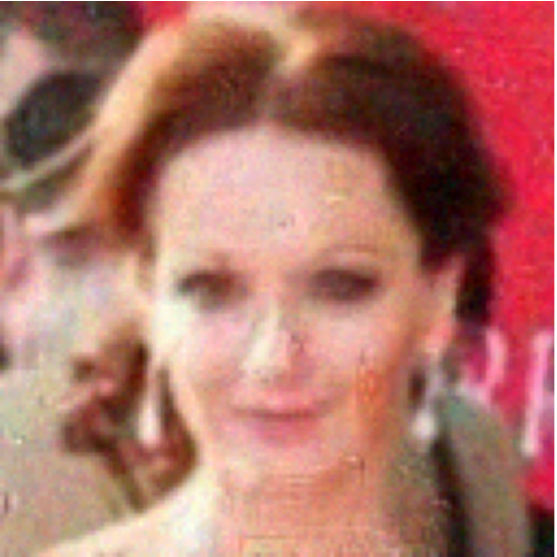}
       \makebox[\linewidth][c]{\tiny PSNR:  29.01}
    \end{minipage}
    \begin{minipage}[b]{0.13\textwidth}
    \captionsetup{skip=-0.11cm}
        \includegraphics[width=\linewidth]{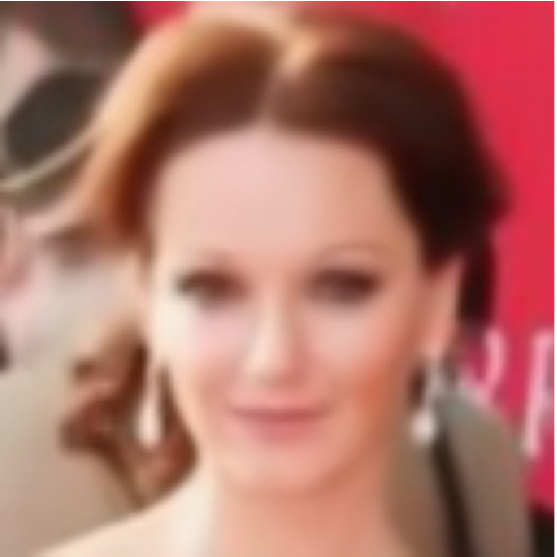}
       \makebox[\linewidth][c]{\tiny PSNR:  32.25}
    \end{minipage}
    \begin{minipage}[b]{0.13\textwidth}
    \captionsetup{skip=-0.11cm}
        \includegraphics[width=\linewidth]{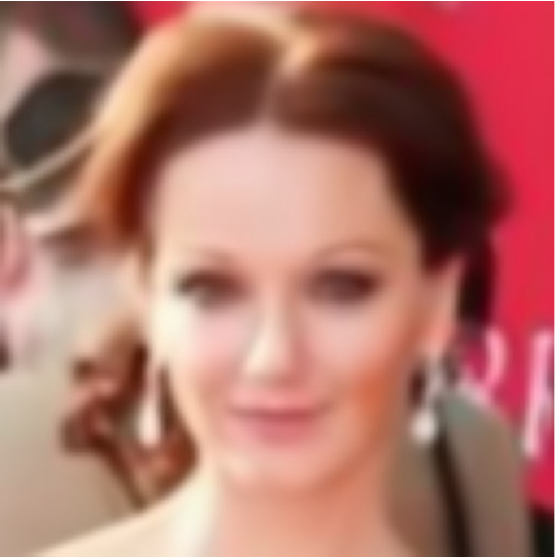}
       \makebox[\linewidth][c]{\tiny PSNR:  \textbf{34.56}}
    \end{minipage}


    \begin{minipage}[b]{0.13\textwidth}
    \captionsetup{skip=-0.11cm}
        \includegraphics[width=\linewidth]{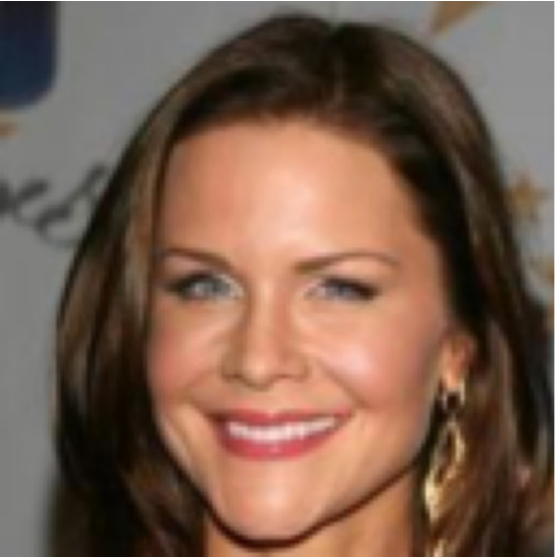}
       \makebox[\linewidth][c]{\tiny Rand. Inpainting }
    \end{minipage}
    \begin{minipage}[b]{0.13\textwidth}
    \captionsetup{skip=-0.11cm}
        \includegraphics[width=\linewidth]{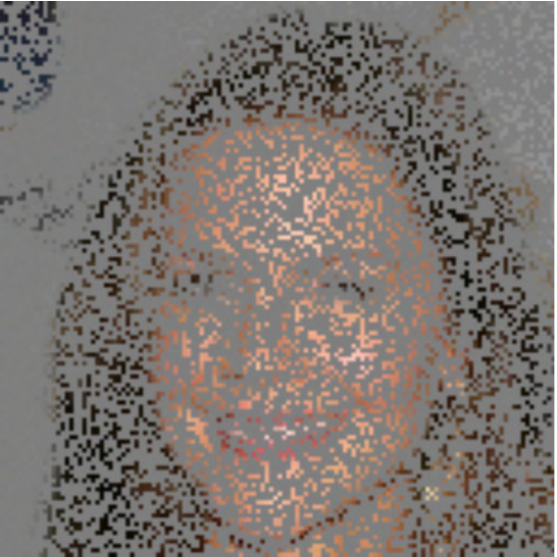}
       \makebox[\linewidth][c]{\tiny PSNR: 13.09}
    \end{minipage}
    \begin{minipage}[b]{0.13\textwidth}
    \captionsetup{skip=-0.11cm}
        \includegraphics[width=\linewidth]{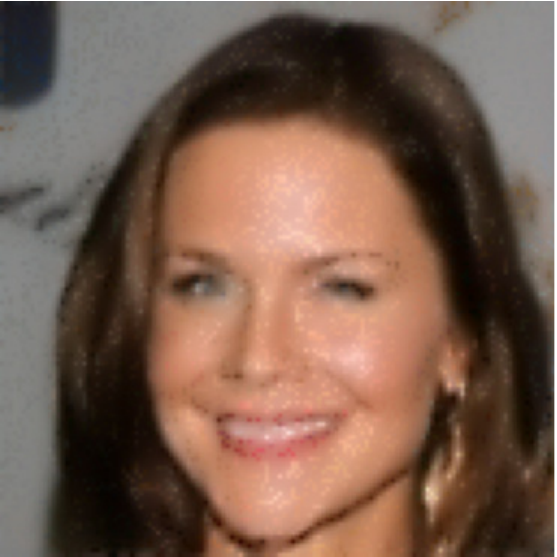}
       \makebox[\linewidth][c]{\tiny PSNR: 29.06}
    \end{minipage}
    \begin{minipage}[b]{0.13\textwidth}
    \captionsetup{skip=-0.11cm}
        \includegraphics[width=\linewidth]{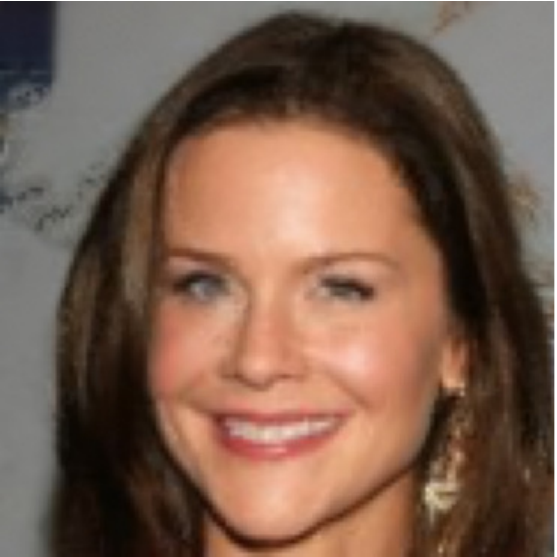}
       \makebox[\linewidth][c]{\tiny PSNR: 28.70}
    \end{minipage}
    \begin{minipage}[b]{0.13\textwidth}
    \captionsetup{skip=-0.11cm}
        \includegraphics[width=\linewidth]{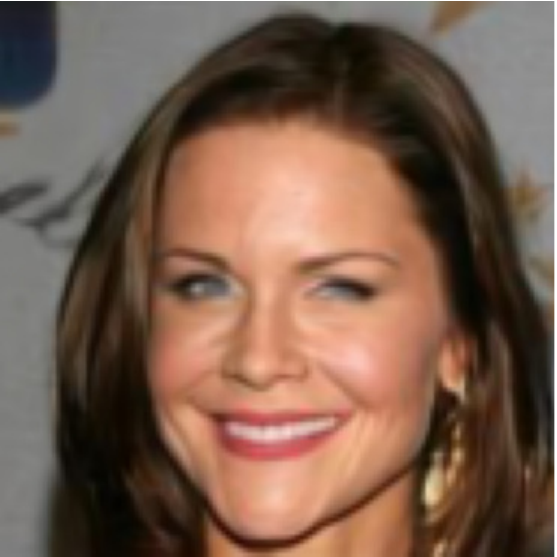}
       \makebox[\linewidth][c]{\tiny PSNR: 32.78}
    \end{minipage}
    \begin{minipage}[b]{0.13\textwidth}
    \captionsetup{skip=-0.11cm}
        \includegraphics[width=\linewidth]{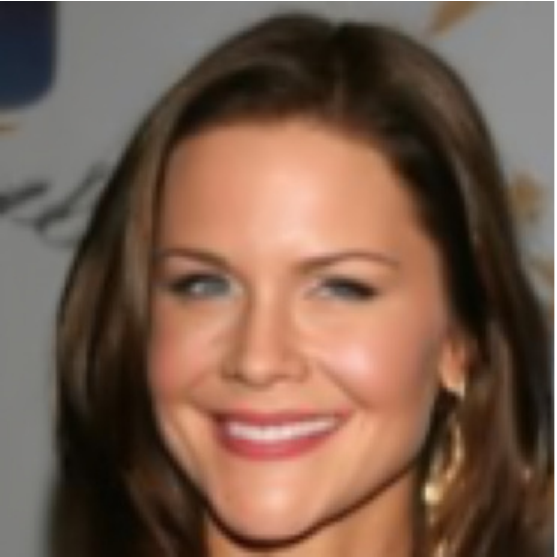}
       \makebox[\linewidth][c]{\tiny PSNR: 33.51}
    \end{minipage}
    \begin{minipage}[b]{0.13\textwidth}
        \captionsetup{skip=-0.11cm}
        \includegraphics[width=\linewidth]{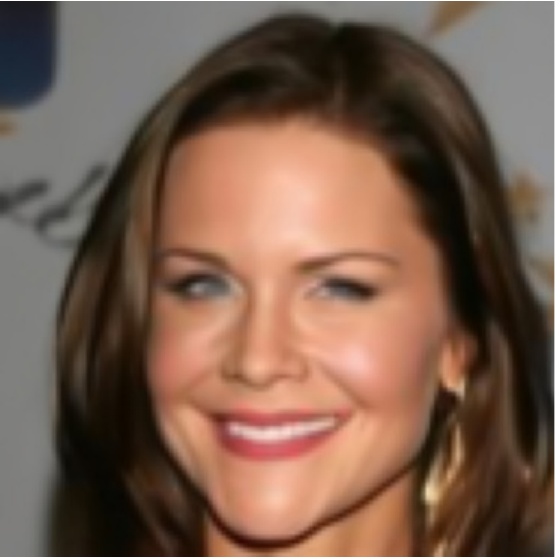}
       \makebox[\linewidth][c]{\tiny PSNR: \textbf{34.72}}
    \end{minipage}


    \begin{minipage}[b]{0.13\textwidth}
    \captionsetup{skip=-0.11cm}
        \includegraphics[width=\linewidth]{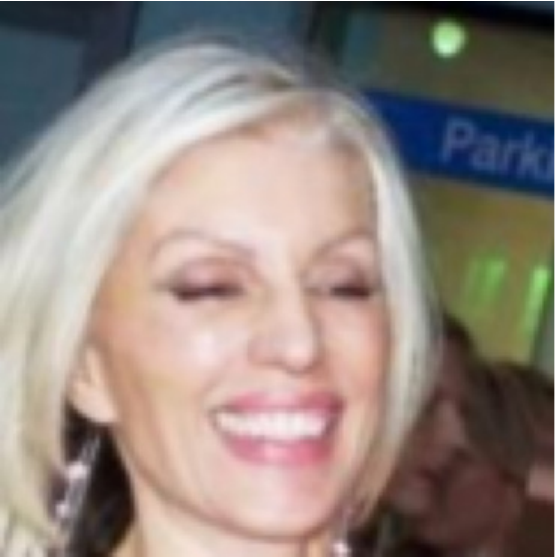}
       \makebox[\linewidth][c]{\tiny Box Inpainting }
    \end{minipage}
    \begin{minipage}[b]{0.13\textwidth}
    \captionsetup{skip=-0.11cm}
        \includegraphics[width=\linewidth]{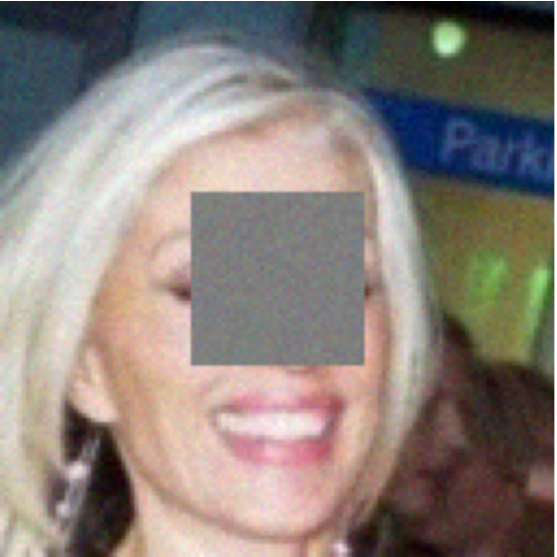}
       \makebox[\linewidth][c]{\tiny PSNR: 20.58}
    \end{minipage}
    \begin{minipage}[b]{0.13\textwidth}
    \captionsetup{skip=-0.11cm}
        \phantom{}
       \makebox[\linewidth][c]{\tiny PSNR: N/A}
    \end{minipage}
    \begin{minipage}[b]{0.13\textwidth}
    \captionsetup{skip=-0.11cm}
        \includegraphics[width=\linewidth]{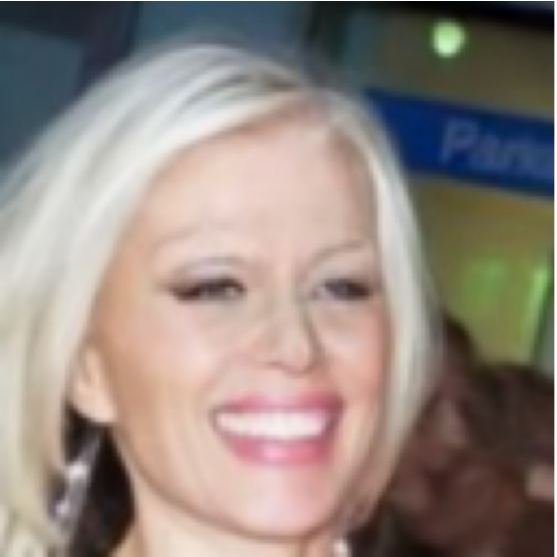}
       \makebox[\linewidth][c]{\tiny PSNR: 31.95}
    \end{minipage}
    \begin{minipage}[b]{0.13\textwidth}
    \captionsetup{skip=-0.11cm}
        \includegraphics[width=\linewidth]{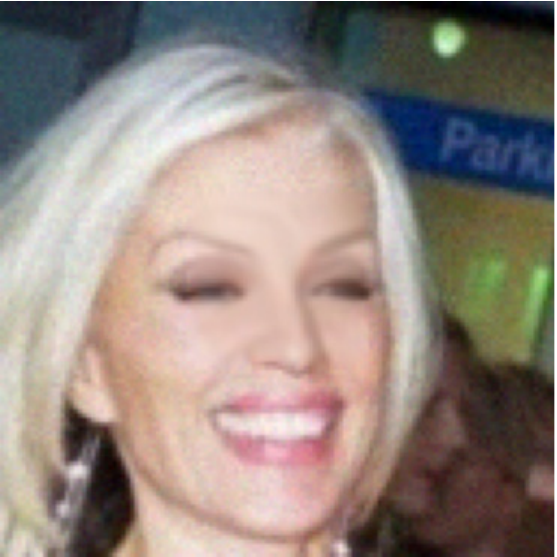}
       \makebox[\linewidth][c]{\tiny PSNR: 32.61}
    \end{minipage}
    \begin{minipage}[b]{0.13\textwidth}
    \captionsetup{skip=-0.11cm}
        \includegraphics[width=\linewidth]{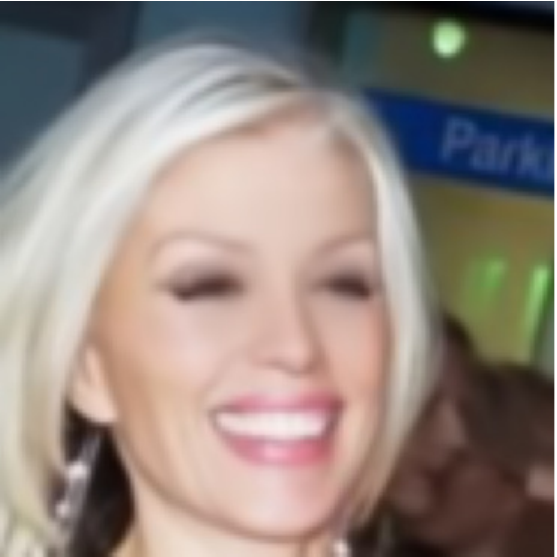}
       \makebox[\linewidth][c]{\tiny PSNR: 32.39}
    \end{minipage}
    \begin{minipage}[b]{0.13\textwidth}
        \captionsetup{skip=-0.11cm}
        \includegraphics[width=\linewidth]{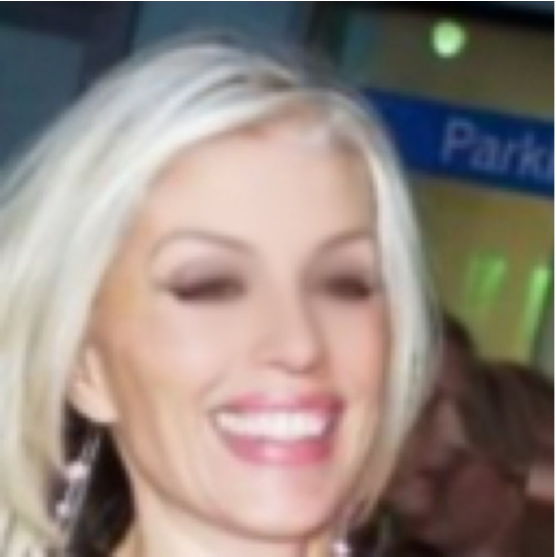}
       \makebox[\linewidth][c]{\tiny PSNR: \textbf{33.81}}
    \end{minipage}
    
     \vspace{-0.2cm}
    \caption{Comparison of image restoration methods on CelebA.}
    \label{fig:celeba}
    \vspace{-0.2cm}
\end{figure}

We present quantitative results in \cref{tab:celeba} (CelebA) and \cref{tab:afhq} (AFHQ-Cat), reporting PSNR and SSIM averaged over 100 test images per task. For competing methods, we directly adopt the best-reported results from \cite{martin2024pnp}, where hyperparameters (including task-specific iteration numbers ranging from 100 to 500) were meticulously optimized. In contrast, our IPnP-Flow uses a fixed 100 iterations across all tasks. Despite this conservative setup, IPnP-Flow consistently soutperforms PnP-Flow and other baselines, demonstrating the efficacy of our SDE-informed enhancements. While higher iterations could further improve results, our method achieves superior performance with greater efficiency.

\begin{table}[ht]
\caption{Performance (PSNR/SSIM) on CelebA Dataset}
\label{tab:celeba}
\centering
\small
\begin{adjustbox}{max width=\textwidth}
\begin{tabular}{lccccc}
\toprule
Method & Denoising & Deblurring & Super-resolution & Random Inpainting & Box Inpainting \\ 
\midrule
Degraded & 20.00 / 0.348 & 27.67 / 0.740 & 7.527 / 0.012 & 11.82 / 0.197  & 22.12 / 0.742 \\
OT-ODE & 30.50 / 0.867 & 32.63 / 0.915 & 31.05 / 0.902 & 28.36 / 0.865  & 28.84 / 0.914 \\
D-Flow & 26.42 / 0.651 & 31.07 / 0.877 & 30.75 / 0.866 & 33.07 / 0.938 & 29.70 / 0.893 \\
Flow-Priors & 29.26 / 0.766 & 31.40 / 0.856 & 28.35 / 0.717 & 32.33 / 0.945 & 29.40 / 0.858 \\
PnP-Diff & 31.00 / 0.883 & 32.49 / 0.911 & 31.20 / 0.893 & 31.43 / 0.917 & N/A  \\
PnP-GS & 32.45 / 0.908 & 33.65 / 0.924 & 30.69 / 0.889 & 28.45 / 0.848 & N/A \\ \hline
PnP-Flow & \uline{32.45} / \uline{0.911} & \uline{34.51} / \uline{0.940} & \uline{31.49} / \uline{0.907} & \uline{33.54} / \uline{0.953} & \uline{30.59} \ \uline{0.943} \\
Ours (PSNR) & \textbf{33.31} $\pm$ 0.07  & \textbf{35.13} $\pm$ 0.08  & \textbf{33.66} $\pm$ 0.16   & \textbf{34.65} $\pm$ 0.04  & \textbf{30.93} $\pm$ 0.05 \\
Ours (SSIM) & \textbf{0.927} $\pm$ 0.003 & \textbf{0.951} $\pm$ 0.002 & \textbf{0.946} $\pm$ 0.003 &  \textbf{0.965} $\pm$ 0.001 & \textbf{0.953} $\pm$ 0.001 \\
\bottomrule
\end{tabular}
\end{adjustbox}
\end{table}

\begin{table}[ht]
\caption{Performance  (PSNR/SSIM) on AFHQ-Cat Dataset}
\label{tab:afhq}
\centering
\small
\begin{adjustbox}{max width=\textwidth}
\begin{tabular}{lccccc}
\toprule
Method & Denoising & Deblurring & Super-resolution & Random Inpainting & Box Inpainting \\
\midrule
Degraded & 20.00 / 0.319 & 23.77 / 0.514 & 10.74 / 0.042 & 13.35 / 0.234 & 21.50 / 0.744 \\
OT-ODE & 29.90 / 0.831 & 26.43 / 0.709 & 25.17 / 0.711 & 28.84 / 0.838  & 23.88 / 0.874 \\
D-Flow & 26.22 / 0.620 & 27.49 / 0.740 & 24.10 / 0.595 & 31.37 / 0.888 & 26.69 / 0.833  \\
Flow-Priors & 29.32 / 0.768 & 25.78 / 0.692 & 23.34 / 0.573 & 31.76 / 0.909 & 25.85 / 0.822 \\
PnP-Diff & 30.27 / 0.835 & \uline{27.97} / \uline{0.764} & 23.22 / 0.601 & 31.08 / 0.882 & N/A \\
PnP-GS & \uline{32.26}  / \textbf{0.895} & 27.33 / 0.749 & 21.86 / 0.619 & 29.61 / 0.855  & N/A \\ \hline
PnP-Flow & 31.65 / 0.876 & 27.62 / 0.763 & \uline{26.75} / \uline{0.774} & \uline{32.98} / \uline{0.930} & \uline{26.87} / \uline{0.904} \\
Ours (PSNR) & \textbf{32.44} $\pm$ 0.13  & \textbf{28.15} $\pm$ 0.06 & \textbf{27.38} $\pm$ 0.04 & \textbf{33.21} $\pm$ 0.04 & \textbf{27.01} $\pm$ 0.05 \\
Ours (SSIM) &  \uline{0.894} $\pm$ 0.003 &  \textbf{0.780} $\pm$ 0.004 &  \textbf{0.796} $\pm$ 0.003 &  \textbf{0.937} $\pm$ 0.001 & \textbf{0.908} $\pm$ 0.002 \\
\bottomrule
\end{tabular}
\end{adjustbox}
\end{table}

Our method achieves superior PSNR and SSIM scores across all tasks and datasets compared to PnP-Flow, with gains of up to 2.17 dB in PSNR (CelebA Super-resolution) and 0.79 dB (AFHQ-Cat denoising). These improvements stem from the SDE model’s guidance on error reduction and acceleration. Qualitative results (see \cref{fig:celeba}) further illustrate that our method produces sharper, artifact-free restorations compared to baselines, enhancing visual fidelity.

Our improved PnP-Flow maintains nearly identical computational efficiency to the original PnP-Flow during inference, as the model architecture is unchanged. The additional extrapolation step has minimal impact on computation time, with the dominant cost still coming from denoising steps.

\vspace{-0.1cm}
\subsection{Ablation Study}\label{subsec:ablation}
\vspace{-0.1cm}
The performance of IPnP-Flow, particularly its error bound and convergence rate, is governed by three principal factors: (1) the extrapolation step size $h_k$, (2) the schedule of $1-l_k$, and (3) the Lipschitz constant $L_u$. We conduct a systematic empirical analysis on the CelebA validation dataset to investigate their individual effects. 

\begin{figure}[htpb]
    \centering 
    \begin{minipage}[b]{0.58\textwidth}
        \centering
        \includegraphics[height=4.5cm,width=6cm]{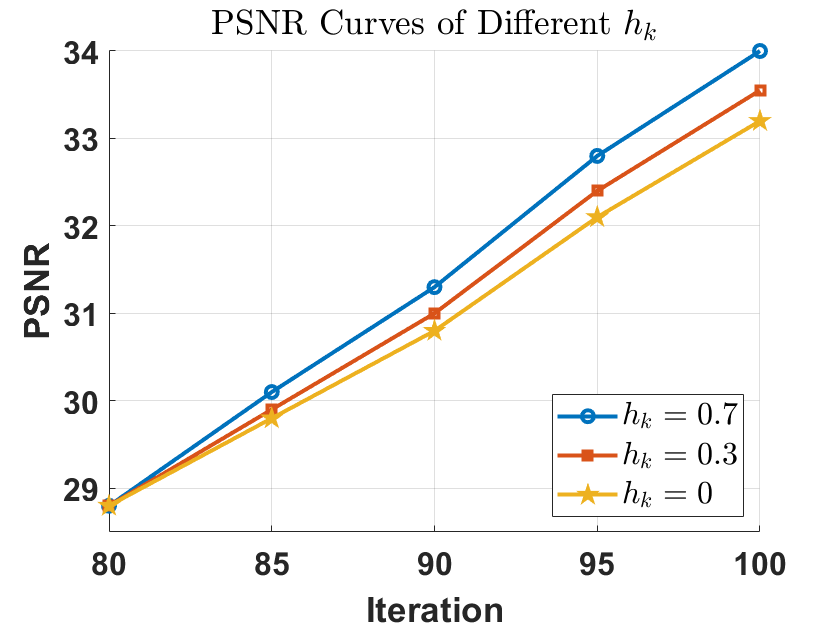}
        \vspace{-0.2cm}
       \caption{Studey on $h_k$}
        \label{fig:hk}
    \end{minipage}
    \begin{minipage}[b]{0.4\textwidth}
        \centering
        \begin{minipage}[b]{\linewidth}
            \centering
            \begin{tabular}{|c|c|c|c}
                \hline
                Schedule & PSNR & SSIM \\ \hline
                \cite{martin2024pnp} & 32.23 & 0.9361 \\ \hline
                $\lambda=0.90$ & 28.99 & 0.7237 \\ \hline
                $\lambda=0.94$ & 32.92 & 0.8837 \\ \hline
                $\lambda=0.96$ & \textbf{34.49} & \textbf{0.9479} \\ \hline
                $\lambda=0.99$ & 30.36 & 0.8984 \\ \hline
            \end{tabular}
            \vspace{-0.2cm}
            \captionof{table}{Study on $1 - l_k$ \label{tab:lk}}
            \vspace{0.2cm}
        \end{minipage}
        \begin{minipage}[b]{\linewidth}
            \centering
            \begin{tabular}{|c|c|c|c}
                \hline
                 & PSNR & SSIM \\ \hline
                With $L_u$ & 33.53 & 0.9401 \\ \hline
                W/o $L_u$ & 33.14 & 0.9358 \\ \hline
            \end{tabular}
            \vspace{-0.2cm}
            \captionof{table}{Study on Lipschitz Penalty \label{tab:lv}}
        \end{minipage}
    \end{minipage}
\vspace{-0.1cm}
\end{figure}

\textbf{Extrapolation Step Size $h_k$}: The upper bound in \cref{thm:acc_con} for $\mathbb{E}[\|X_t-X_T\|^2]$ depends critically on $h_k$. Since $X_T$ approximates the ground truth $X_G\in\mathbb{R}^n$, we use PSNR
as our evaluation metric, where higher PSNR corresponds to smaller $\mathbb{E}[\|X_t-X_T\|^2]$.

With fixed $\gamma_k=0.001$ and $N=100$, we first verify that $K\in\{40,60,80\}$ yields nearly identical PSNR (mean difference $< 0.1$dB), suggesting $K$ has minimal impact in this experiment. We therefore fix $K=80$,  gradient step size $\gamma_k=0.001$, total iteration number $N=100$, and then evaluate the impact of varying $h_k$ on the deblurring task. \cref{fig:hk} presents the PSNR curve after $K$-th iteration for different $h_k$ configurations (final PNSR = $\{33.23,33.55,33.95\}$ for $h_k$ = $\{0, 0.3, 0.7\}$) . The results show that an appropriately chosen $h_k$ ensures higher PSNR with fewer iteration numbers, implying smaller $\|X_t-X_T\|^2$—consistent with our theoretical results in \cref{prop:con} and \cref{thm:acc_con}. However, excessively large $h_k$ values degrade performance due to poor sequence convergence.

\textbf{Schedule of $1-l_k$}: With all other settings fixed, we set gradient step size $\gamma_k=0.001$, total iteration number $N=100$, $h_k=0$, and then evaluate the impact of $1-l_k$ on the deblurring task using the CelebA dataset. \cref{tab:lk} presents the final PSNR and SSIM values for different $\lambda$ configurations. The second row in \cref{tab:lk} uses the schedule proposed in PnP-Flow, the 3-6 rows use our proposed schedule \cref{eq:lk} with four different $\lambda$. As discussed in \cref{sec:lk}, a proper $\lambda$ can significantly improve the generated image quality in terms of PSNR and SSIM.

\textbf{Lipschitz Penalty $L_u$}: With all other settings fixed, we set step size $\gamma_k=0.001$, total iteration number $N=100$, $h_k=0$. We empirically set the coefficient of the Lipschitz penalty loss to be 0.1 and compare our method with and without Lipschitz regularization. Table~\ref{tab:lv} shows that penalizing the Lipschitz constant improves PSNR and SSIM, aligning with the error analysis in Section~\ref{sec:Theory}.

\vspace{-0.1cm}
\section{Conclusion}
\vspace{-0.1cm}
In this paper, we advance the theoretical and practical understandings of PnP-Flow for image restoration from a continuous-time SDE perspective. Our approach uniquely integrates SDE analysis with discrete PnP-Flow, providing both theoretical rigor and practical improvements for PnP-Flow. 
As a limitation of our work, we 
focuses on straight-line flows in this paper, generalizing our work to manifold-based flows remains future work. Although developed for PnP-Flow, our methodology can be extended to other generative models and deep learning frameworks. Key future directions include: 1) Extending the SDE analysis to manifold-based flow matching models to broaden applicability. 2) Developing advanced projection techniques for the interpolation step to improve denoising performance. 3) Investigating additional SDE-informed strategies, such as adaptive noise schedules or alternative extrapolation methods, to further enhance efficiency and accuracy. By pursuing these avenues, we aim to further develop the SDE-informed frameworks, making them more robust and versatile tools for image restoration and beyond.

\noindent{\bf Societal Impacts:} Our paper presents an SDE framework for analyzing and improving PnP-Flow. Imaging is a fundamental technique for many of the modern scientific disciplines, improving imaging algorithms will have foundational impact on these disciplines. We do not see additional negative societal impact compared to existing approaches due to our work.


\appendix

\section{Additional Experimental Details}\label{sec:appendix_experiment}
\subsection{Frobenius Norm Estimator}
For a $n\times n$ matrix ${\bm A}$, its trace can be unbiasedly estimated by~\cite{hutchinson1989stochastic}
\begin{equation}
\text{Trace}\big({\bm A}\big) = \mathbb{E}_{{\bm \epsilon}\sim p({\bm \epsilon})}\big[{\bm \epsilon}^\top{\bm A}{\bm \epsilon}\big]
\end{equation}
, where $p({\bm \epsilon})$ is a $n$-dimensional distribution s.t. $\mathbb{E}[{\bm \epsilon}]=\bm 0$ and $\text{Cov}[{\bm \epsilon}]={\bm I}_{n\times n}$ and $p({\bm \epsilon})$ is always chosen as standard Gaussian. Then given a vector field regressor $\tilde {\bm \epsilon}(t, X_t)$, the Frobenius norm of the Jacobian can be estimated
\begin{equation}\label{eq:Hutchison_F}
\begin{split}
\lVert\nabla_{\bm x} \tilde{\bm u}_t(X_t)\rVert_{F}^2 &= \text{Trace}\big(\nabla_{\bm x} \tilde{\bm u}_t(X_t)^\top\nabla_{\bm x} \tilde{\bm u}_t(X_t)\big)  \\
& = \mathbb{E}_{{\bm \epsilon}\sim p({\bm \epsilon})}\big[{\bm \epsilon}^\top \nabla_{\bm x} \tilde{\bm u}_t(X_t)^\top\nabla_{\bm x} \tilde{\bm u}_t(X_t) {\bm \epsilon}\big] \\
& = \mathbb{E}_{{\bm \epsilon}\sim p({\bm \epsilon})}\big[\lVert \nabla_{\bm x} \tilde{\bm u}_t(X_t){\bm \epsilon}  \rVert_2^2\big]
\end{split}
\end{equation}
and notice the fact that $\lVert \nabla_{\bm x} \tilde{\bm u}_t(X_t) \rVert_2\leq \lVert \nabla_{\bm x} \tilde{\bm u}_t(X_t) \rVert_F$ and $L_u = \sup_{X\in\mathbb R^n} \lVert \nabla_{\bm x} \tilde{{\bm u}}_t( X_t) \rVert_2$. So penalizing Eq.(\ref{eq:Hutchison_F}) serves to constrain the Lipschitz constant.

Given the influence of the Lipschitz constant $L_u$ of the estimated vector field $\tilde{{\bm u}}$ (or $\tilde{{\bm u}}$) on error bounds and convergence rates, we introduce an FM model regularized by a Lipschitz penalty term (see \cref{sec:4.1}). We use the same training settings in \cite{martin2024pnp}. We use the pre-trained model provided by \cite{martin2024pnp} as the baseline model.

\subsection{Hyper-parameter Setting}
We list the hyper-parameters used for our method on the CelebA dataset and AFHQ-Cat dataset in the following two tables.

\begin{table}[ht]
\caption{Hyper-parameters used for our method on the CelebA dataset.}
\label{tab:celeba_hyper}
\centering
\begin{adjustbox}{max width=\textwidth}
\begin{tabular}{lccccc}
\toprule
 & Denoising & Deblurring & Super-resolution & Random Inpainting & Box Inpainting \\
\midrule
$r_k$ (gradient step size) & 0.004 & 0.003  & 0.002  & 0.0002 & 0.0012 \\
$h_k$ (extrapolation step size) & 0.5  & 0.5 & 0.5 & 0.5 & 0.5 \\
$\lambda$ ($1-l_k$) & 0.965 & 0.965 & 0.965 & 0.965 & 0.965 \\
N (number of iterations) & 100  & 100 & 100 & 100  & 100 \\
K (SDE $t_0$) & 80 & 80  & 80 & 80 & 80 \\
\bottomrule
\end{tabular}
\end{adjustbox}
\end{table}

\begin{table}[ht]
\caption{Hyper-parameters used for our method on the AFHQ-Cat dataset.}
\label{tab:afhq_hyper}
\centering
\begin{adjustbox}{max width=\textwidth}
\begin{tabular}{lccccc}
\toprule
 & Denoising & Deblurring & Super-resolution & Random Inpainting & Box Inpainting \\
\midrule
$r_k$ (gradient step size) & 0.004 & 0.0045 & 0.004 &  0.0002 &  0.002 \\
$h_k$ (extrapolation step size) & 0.3  & 0.3 & 0.3 & 0.3 & 0.3 \\
$\lambda$ ($1-l_k$) & 0.965 & 0.965 & 0.965 & 0.965 & 0.965  \\
N (number of iterations) & 100  & 100 & 100 & 100 & 100  \\
K (SDE $t_0$) & 80 & 80  & 80 & 80 & 80  \\
\bottomrule
\end{tabular}
\end{adjustbox}
\end{table}

\subsection{Additional Visual Results}

In this section, we provide some additional visual results for comparison.

\begin{figure}[!ht]
    \centering

    \begin{minipage}[b]{0.13\textwidth}
        \centering
        {\footnotesize \text{Clean}}
    \end{minipage}
    \begin{minipage}[b]{0.13\textwidth}
        \centering
        \small Noisy
    \end{minipage}
    \begin{minipage}[b]{0.13\textwidth}
        \centering
        \small PnP-GS
    \end{minipage}
    \begin{minipage}[b]{0.13\textwidth}
        \centering
        \small OT-ODE
    \end{minipage}
    \begin{minipage}[b]{0.13\textwidth}
        \centering
        \small Flow-Priors
    \end{minipage}
    \begin{minipage}[b]{0.13\textwidth}
        \centering
        \small PnP-Flow
    \end{minipage}
    \begin{minipage}[b]{0.13\textwidth}
        \centering
        \small Ours
    \end{minipage}

    \begin{minipage}[b]{0.13\textwidth}
    \captionsetup{skip=-0.01cm}
        \includegraphics[width=\linewidth]{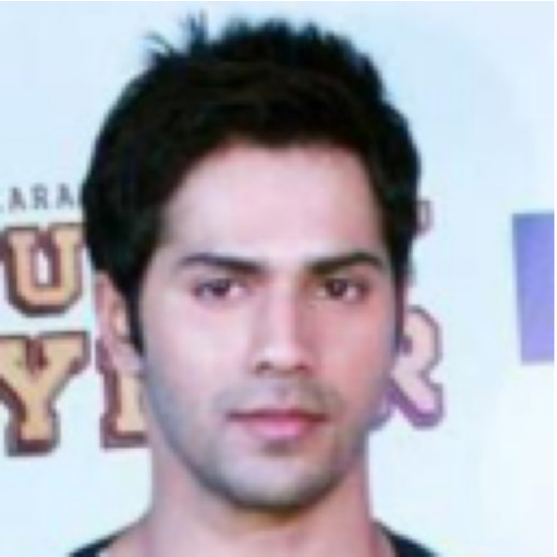}
        \makebox[\linewidth][c]{\footnotesize }
    \end{minipage}
    \begin{minipage}[b]{0.13\textwidth}
    \captionsetup{skip=-0.01cm}
        \includegraphics[width=\linewidth]{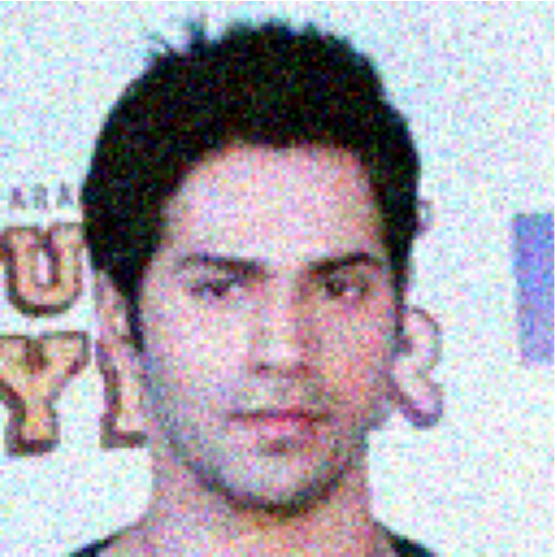}
        \makebox[\linewidth][c]{\footnotesize PSNR: 19.97}
    \end{minipage}
    \begin{minipage}[b]{0.13\textwidth}
    \captionsetup{skip=-0.01cm}
        \includegraphics[width=\linewidth]{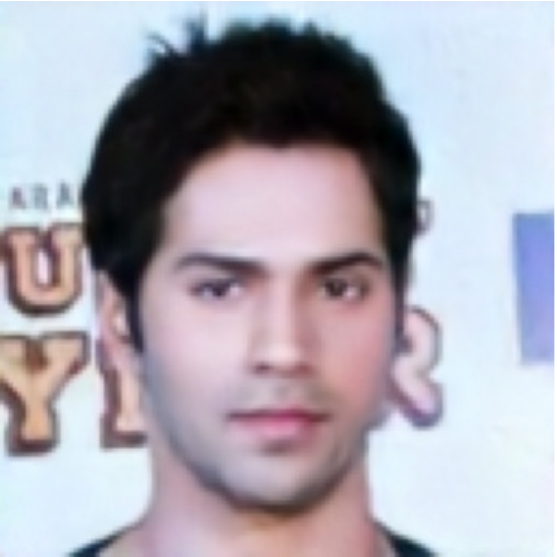}
       \makebox[\linewidth][c]{\footnotesize PSNR: 32.36}
    \end{minipage}
    \begin{minipage}[b]{0.13\textwidth}
    \captionsetup{skip=-0.01cm}
        \includegraphics[width=\linewidth]{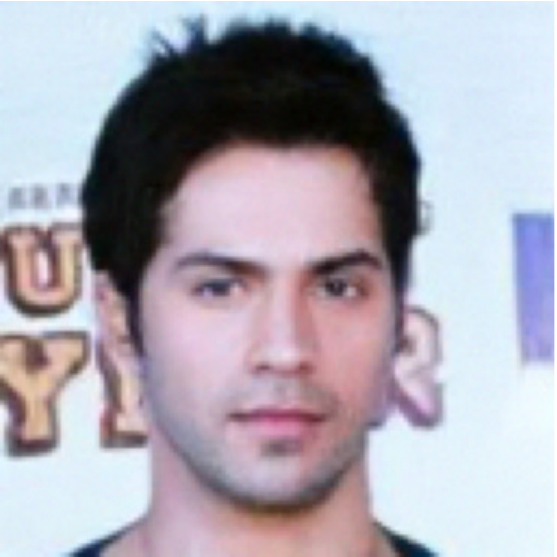}
       \makebox[\linewidth][c]{\footnotesize PSNR:  30.22}
    \end{minipage}
    \begin{minipage}[b]{0.13\textwidth}
    \captionsetup{skip=-0.01cm}
        \includegraphics[width=\linewidth]{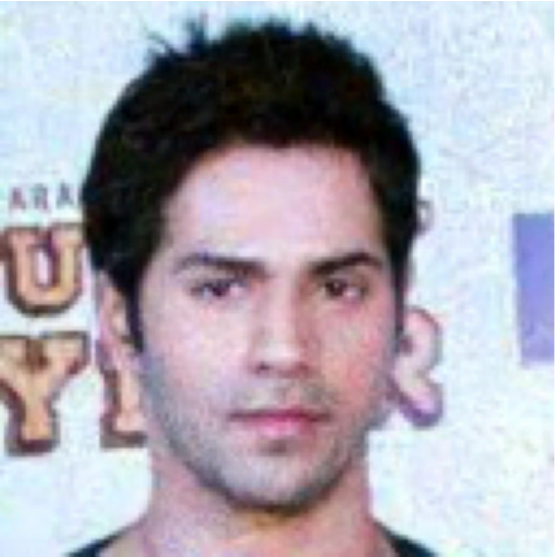}
       \makebox[\linewidth][c]{\footnotesize PSNR:  29.51}
    \end{minipage}
    \begin{minipage}[b]{0.13\textwidth}
    \captionsetup{skip=-0.01cm}
        \includegraphics[width=\linewidth]{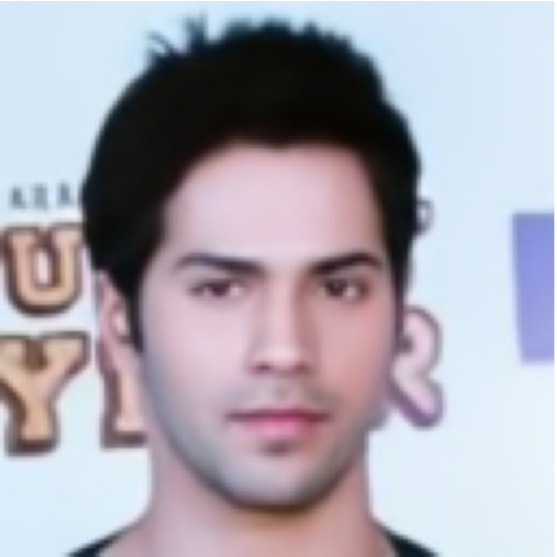}
       \makebox[\linewidth][c]{\footnotesize PSNR:  31.58}
    \end{minipage}
    \begin{minipage}[b]{0.13\textwidth}
    \captionsetup{skip=-0.01cm}
        \includegraphics[width=\linewidth]{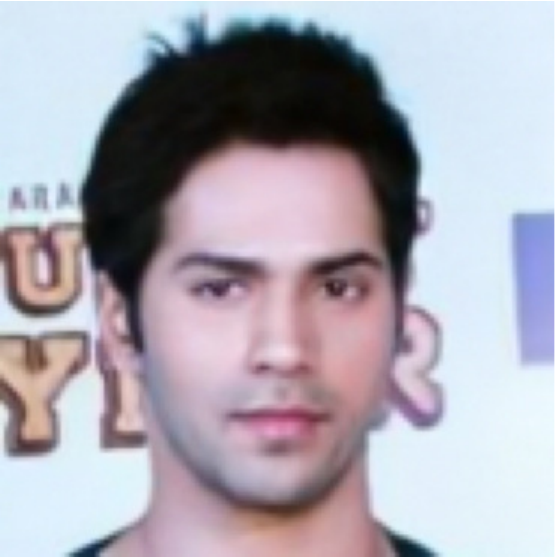}
       \makebox[\linewidth][c]{\footnotesize PSNR:  \textbf{32.94}}
    \end{minipage}

    \vspace{0.2cm}

    \begin{minipage}[b]{0.13\textwidth}
    \captionsetup{skip=-0.01cm}
        \includegraphics[width=\linewidth]{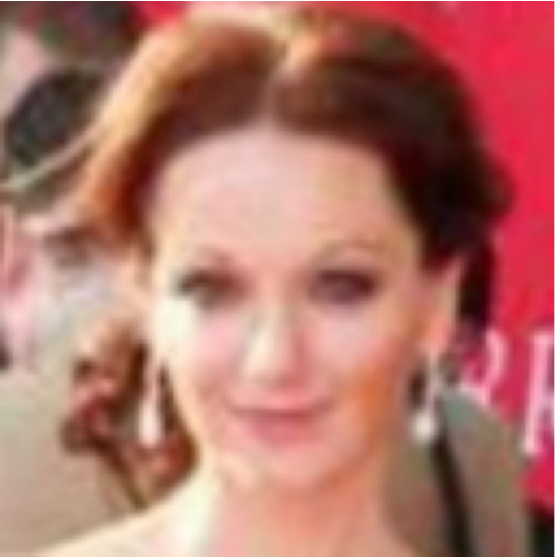}
       \makebox[\linewidth][c]{\footnotesize  }
    \end{minipage}
    \begin{minipage}[b]{0.13\textwidth}
        \captionsetup{skip=-0.01cm}
        \includegraphics[width=\linewidth]{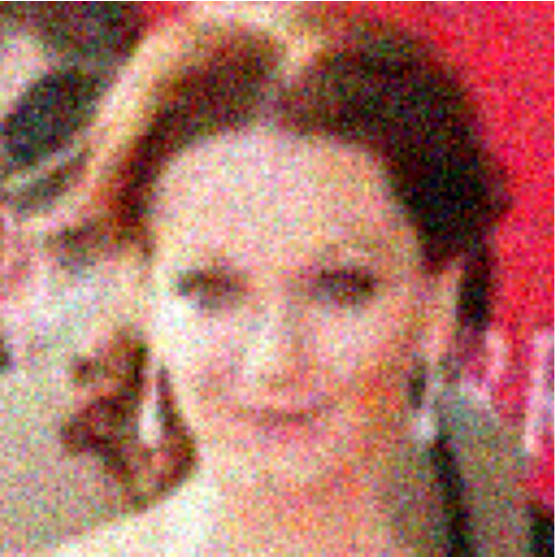}
       \makebox[\linewidth][c]{\footnotesize PSNR: 20.00}
    \end{minipage}
    \begin{minipage}[b]{0.13\textwidth}
    \captionsetup{skip=-0.01cm}
        \includegraphics[width=\linewidth]{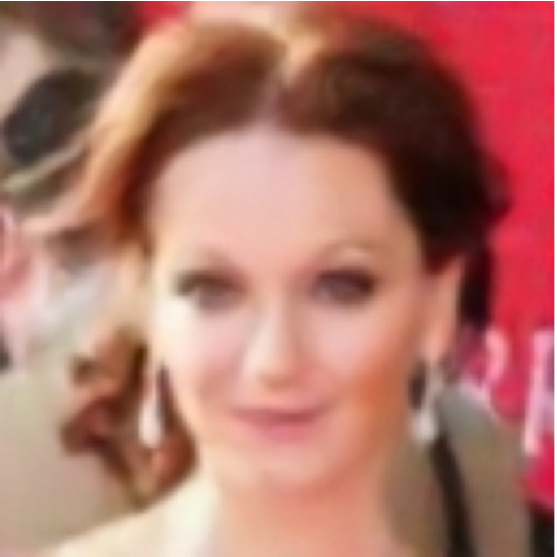}
       \makebox[\linewidth][c]{\footnotesize PSNR: 32.70}
    \end{minipage}
    \begin{minipage}[b]{0.13\textwidth}
    \captionsetup{skip=-0.01cm}
        \includegraphics[width=\linewidth]{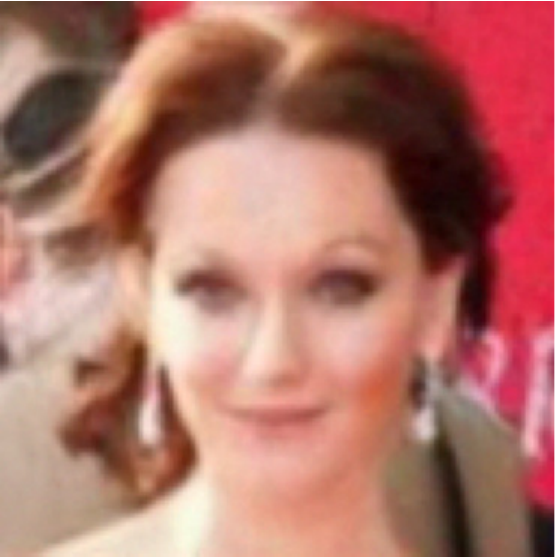}
       \makebox[\linewidth][c]{\footnotesize PSNR:  30.53}
    \end{minipage}
    \begin{minipage}[b]{0.13\textwidth}
    \captionsetup{skip=-0.01cm}
        \includegraphics[width=\linewidth]{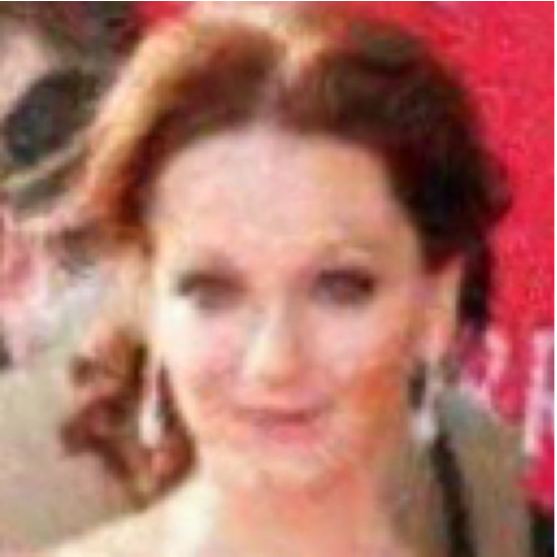}
       \makebox[\linewidth][c]{\footnotesize PSNR:  30.97}
    \end{minipage}
    \begin{minipage}[b]{0.13\textwidth}
    \captionsetup{skip=-0.01cm}
        \includegraphics[width=\linewidth]{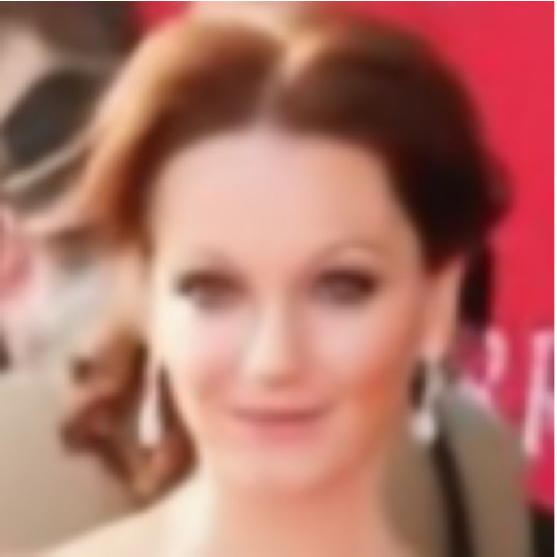}
       \makebox[\linewidth][c]{\footnotesize PSNR:  32.35}
    \end{minipage}
    \begin{minipage}[b]{0.13\textwidth}
    \captionsetup{skip=-0.01cm}
        \includegraphics[width=\linewidth]{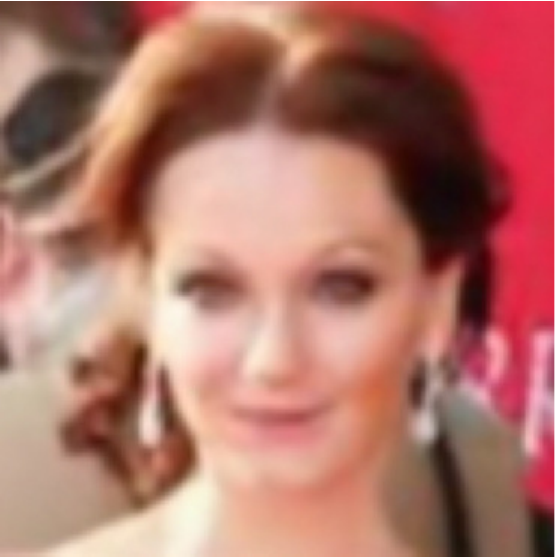}
       \makebox[\linewidth][c]{\footnotesize PSNR:  \textbf{33.40}}
    \end{minipage}

   \vspace{0.2cm}

    \begin{minipage}[b]{0.13\textwidth}
    \captionsetup{skip=-0.01cm}
        \includegraphics[width=\linewidth]{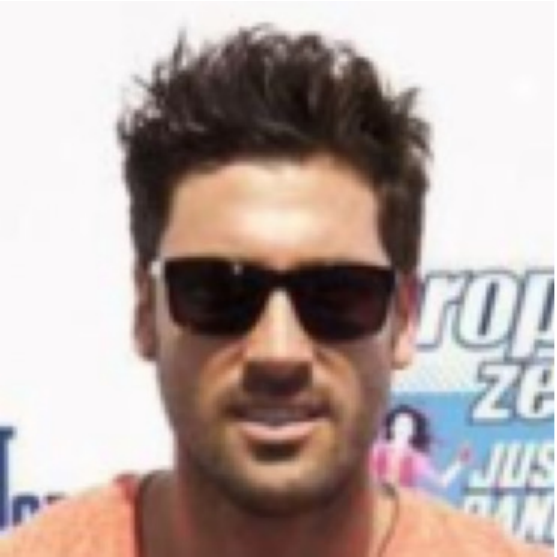}
       \makebox[\linewidth][c]{\footnotesize }
    \end{minipage}
    \begin{minipage}[b]{0.13\textwidth}
    \captionsetup{skip=-0.01cm}
        \includegraphics[width=\linewidth]{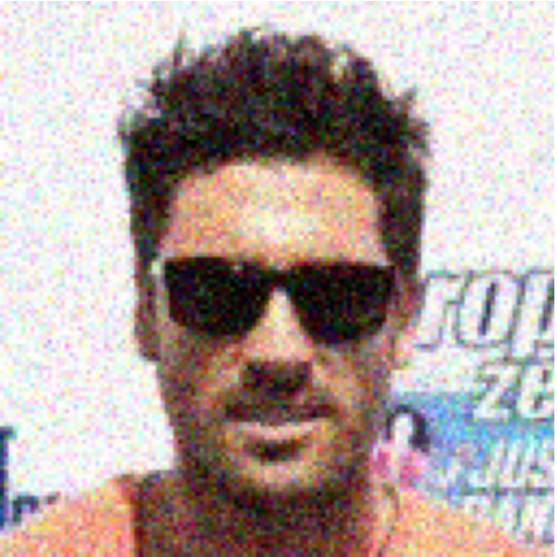}
       \makebox[\linewidth][c]{\footnotesize PSNR:  20.02}
    \end{minipage}
    \begin{minipage}[b]{0.13\textwidth}
    \captionsetup{skip=-0.01cm}
        \includegraphics[width=\linewidth]{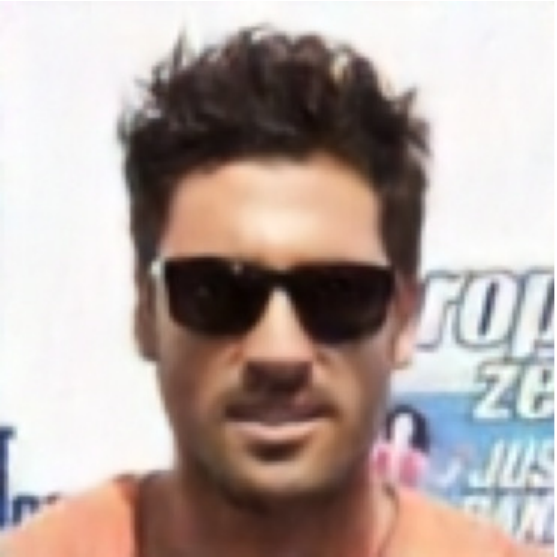}
       \makebox[\linewidth][c]{\footnotesize PSNR:  31.71}
    \end{minipage}
    \begin{minipage}[b]{0.13\textwidth}
    \captionsetup{skip=-0.01cm}
        \includegraphics[width=\linewidth]{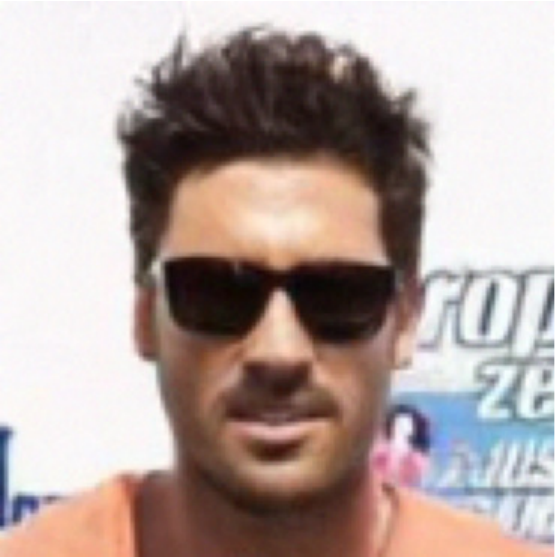}
       \makebox[\linewidth][c]{\footnotesize PSNR:  29.63}
    \end{minipage}
    \begin{minipage}[b]{0.13\textwidth}
    \captionsetup{skip=-0.01cm}
        \includegraphics[width=\linewidth]{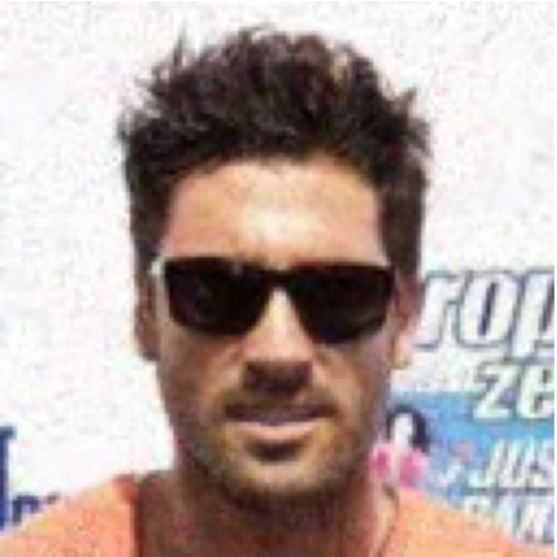}
       \makebox[\linewidth][c]{\footnotesize PSNR: 29.35}
    \end{minipage}
    \begin{minipage}[b]{0.13\textwidth}
    \captionsetup{skip=-0.01cm}
        \includegraphics[width=\linewidth]{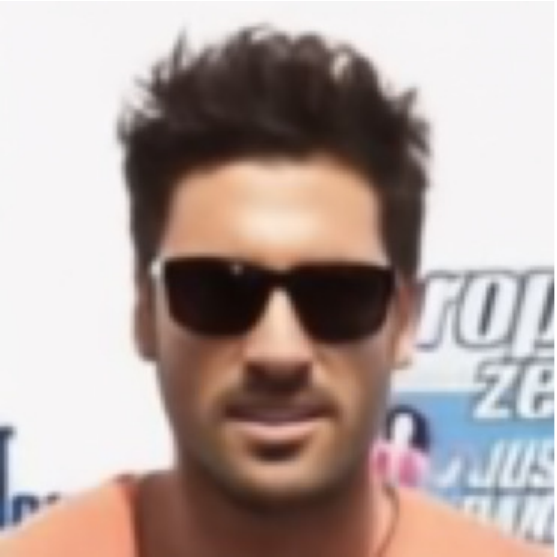}
       \makebox[\linewidth][c]{\footnotesize PSNR:  30.89}
    \end{minipage}
    \begin{minipage}[b]{0.13\textwidth}
    \captionsetup{skip=-0.01cm}
        \includegraphics[width=\linewidth]{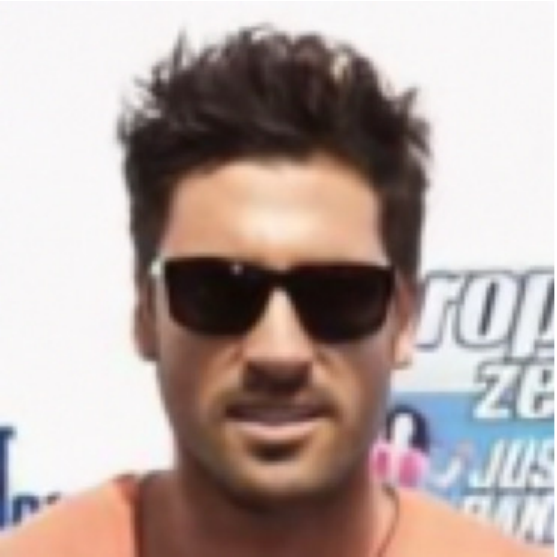}
       \makebox[\linewidth][c]{\footnotesize PSNR:  \textbf{32.27}}
    \end{minipage}

    \caption{Comparison of image denoising results on CelebA.}
    \label{fig:celeba_denoise}
\end{figure}

\vspace{5cm}

\begin{figure}[!ht]
    \centering

    \begin{minipage}[b]{0.13\textwidth}
        \centering
        {\footnotesize \text{Clean}}
    \end{minipage}
    \begin{minipage}[b]{0.13\textwidth}
        \centering
        \small Blurry
    \end{minipage}
    \begin{minipage}[b]{0.13\textwidth}
        \centering
        \small PnP-GS
    \end{minipage}
    \begin{minipage}[b]{0.13\textwidth}
        \centering
        \small OT-ODE
    \end{minipage}
    \begin{minipage}[b]{0.13\textwidth}
        \centering
        \small Flow-Priors
    \end{minipage}
    \begin{minipage}[b]{0.13\textwidth}
        \centering
        \small PnP-Flow
    \end{minipage}
    \begin{minipage}[b]{0.13\textwidth}
        \centering
        \small Ours
    \end{minipage}

    \begin{minipage}[b]{0.13\textwidth}
    \captionsetup{skip=-0.01cm}
        \includegraphics[width=\linewidth]{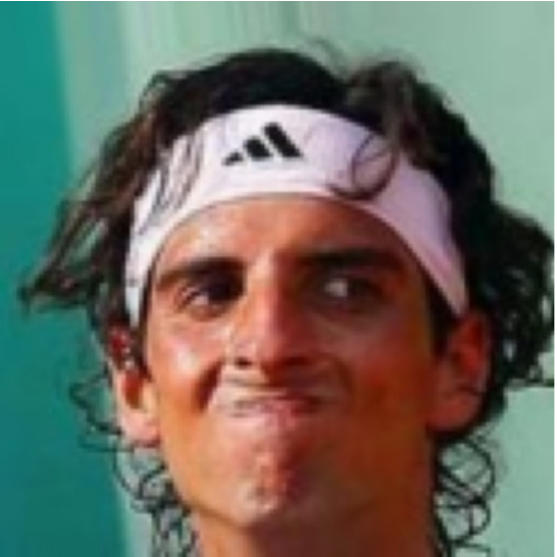}
        \makebox[\linewidth][c]{\footnotesize }
    \end{minipage}
    \begin{minipage}[b]{0.13\textwidth}
    \captionsetup{skip=-0.01cm}
        \includegraphics[width=\linewidth]{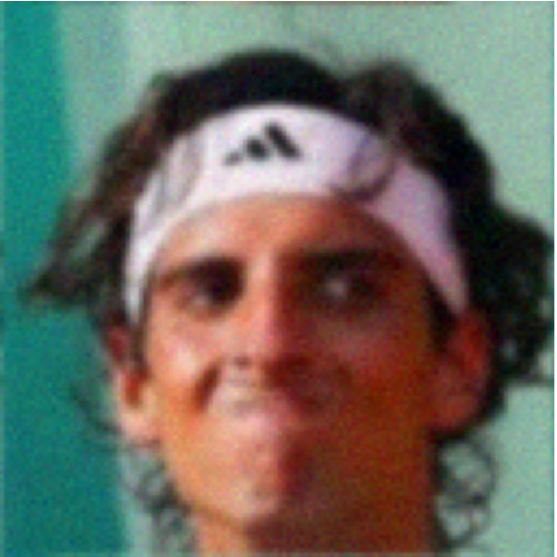}
        \makebox[\linewidth][c]{\footnotesize PSNR: 25.52}
    \end{minipage}
    \begin{minipage}[b]{0.13\textwidth}
    \captionsetup{skip=-0.01cm}
        \includegraphics[width=\linewidth]{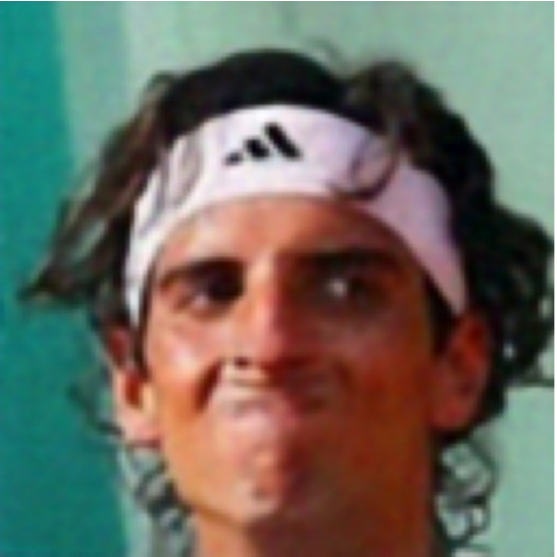}
       \makebox[\linewidth][c]{\footnotesize PSNR: 29.70}
    \end{minipage}
    \begin{minipage}[b]{0.13\textwidth}
    \captionsetup{skip=-0.01cm}
        \includegraphics[width=\linewidth]{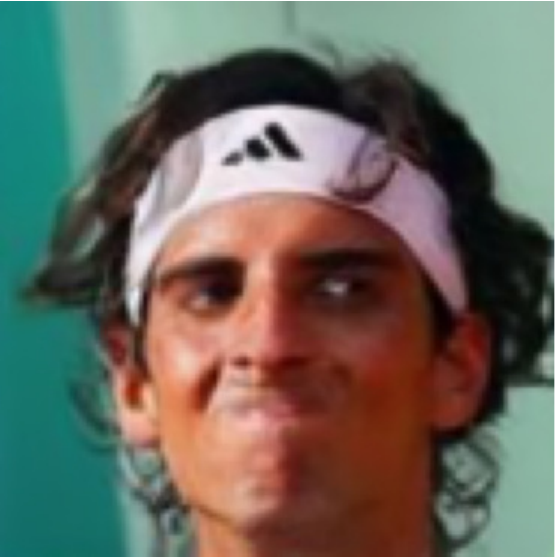}
       \makebox[\linewidth][c]{\footnotesize PSNR:  28.59}
    \end{minipage}
    \begin{minipage}[b]{0.13\textwidth}
    \captionsetup{skip=-0.01cm}
        \includegraphics[width=\linewidth]{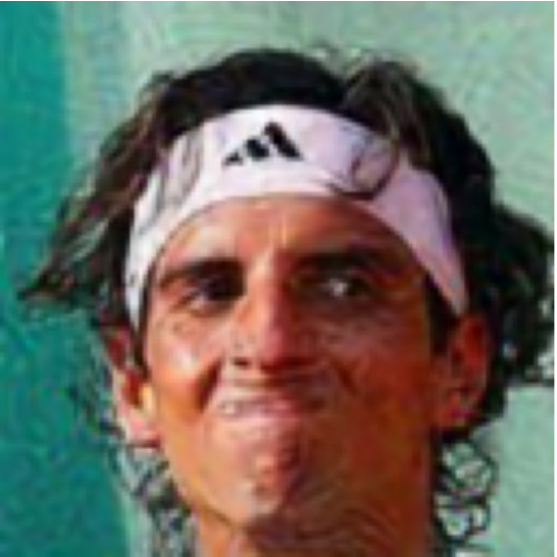}
       \makebox[\linewidth][c]{\footnotesize PSNR:  28.43}
    \end{minipage}
    \begin{minipage}[b]{0.13\textwidth}
    \captionsetup{skip=-0.01cm}
        \includegraphics[width=\linewidth]{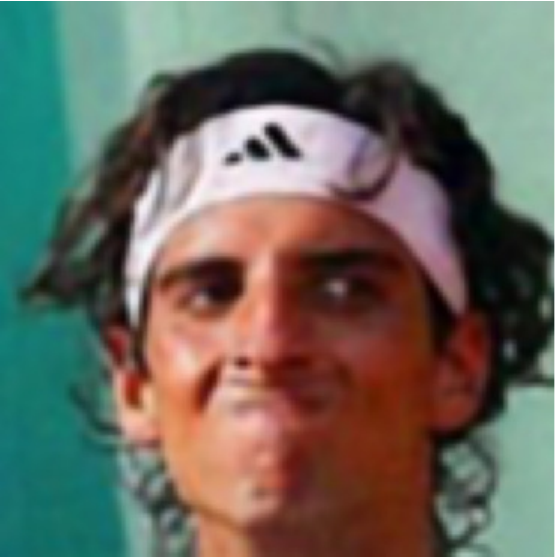}
       \makebox[\linewidth][c]{\footnotesize PSNR:  29.62}
    \end{minipage}
    \begin{minipage}[b]{0.13\textwidth}
    \captionsetup{skip=-0.01cm}
        \includegraphics[width=\linewidth]{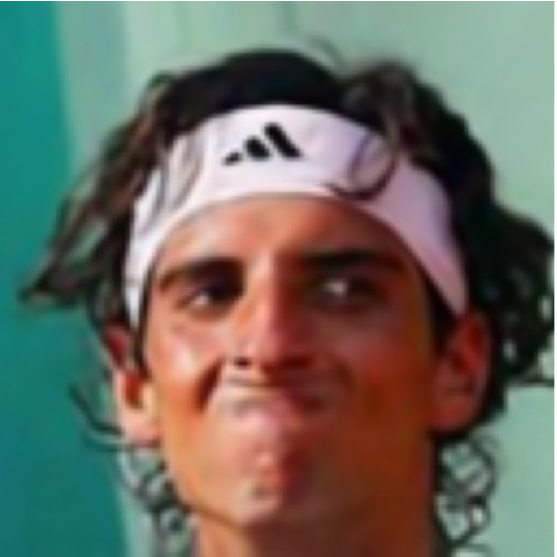}
       \makebox[\linewidth][c]{\footnotesize PSNR:  \textbf{30.56}}
    \end{minipage}

    \vspace{0.2cm}

    \begin{minipage}[b]{0.13\textwidth}
    \captionsetup{skip=-0.01cm}
        \includegraphics[width=\linewidth]{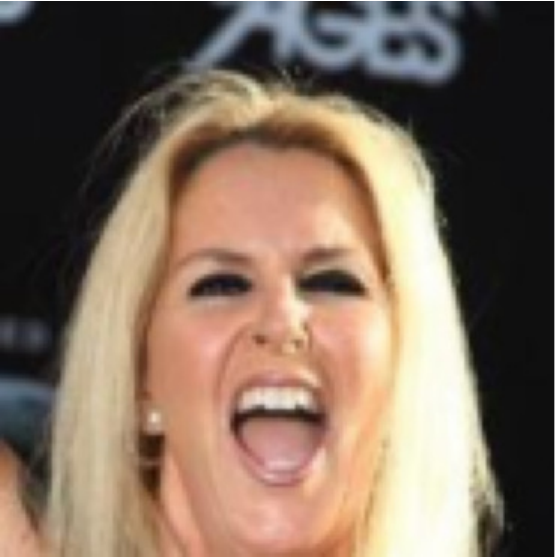}
       \makebox[\linewidth][c]{\footnotesize }
    \end{minipage}
    \begin{minipage}[b]{0.13\textwidth}
    \captionsetup{skip=-0.01cm}
        \includegraphics[width=\linewidth]{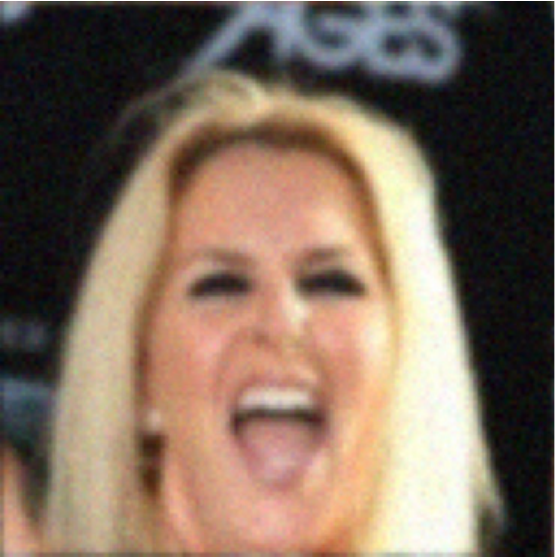}
       \makebox[\linewidth][c]{\footnotesize PSNR: 27.05}
    \end{minipage}
    \begin{minipage}[b]{0.13\textwidth}
    \captionsetup{skip=-0.01cm}
        \includegraphics[width=\linewidth]{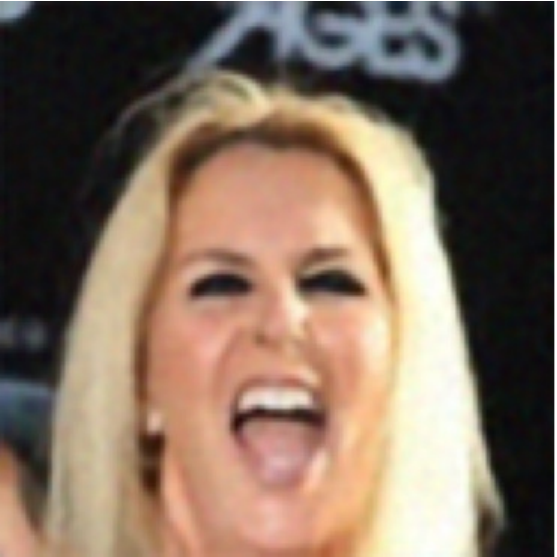}
       \makebox[\linewidth][c]{\footnotesize PSNR:  33.10}
    \end{minipage}
    \begin{minipage}[b]{0.13\textwidth}
    \captionsetup{skip=-0.01cm}
        \includegraphics[width=\linewidth]{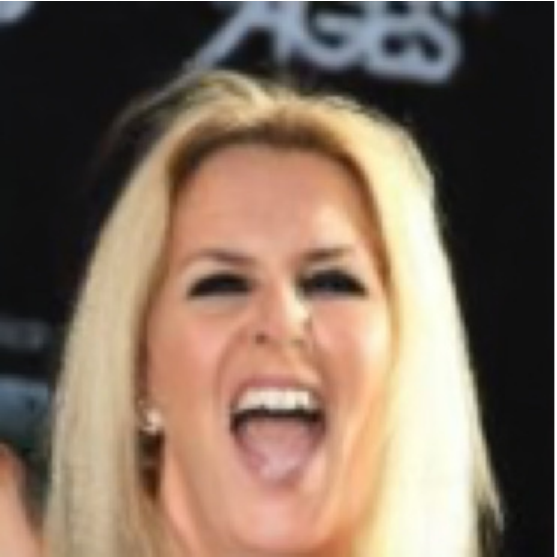}
       \makebox[\linewidth][c]{\footnotesize PSNR: 31.79}
    \end{minipage}
    \begin{minipage}[b]{0.13\textwidth}
    \captionsetup{skip=-0.01cm}
        \includegraphics[width=\linewidth]{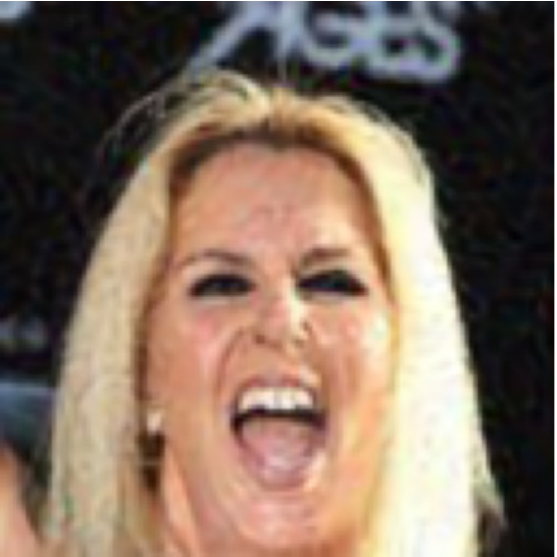}
       \makebox[\linewidth][c]{\footnotesize PSNR:  31.02}
    \end{minipage}
    \begin{minipage}[b]{0.13\textwidth}
    \captionsetup{skip=-0.01cm}
        \includegraphics[width=\linewidth]{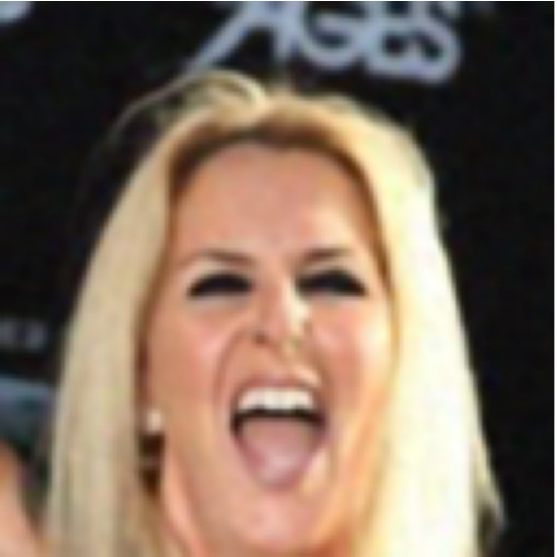}
       \makebox[\linewidth][c]{\footnotesize PSNR:  33.59}
    \end{minipage}
    \begin{minipage}[b]{0.13\textwidth}
    \captionsetup{skip=-0.01cm}
        \includegraphics[width=\linewidth]{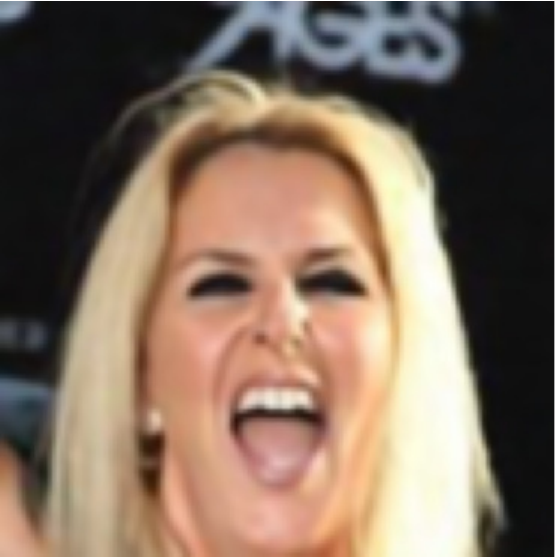}
       \makebox[\linewidth][c]{\footnotesize PSNR: \textbf{34.15}}
    \end{minipage}

\vspace{0.2cm}

    \begin{minipage}[b]{0.13\textwidth}
    \captionsetup{skip=-0.01cm}
        \includegraphics[width=\linewidth]{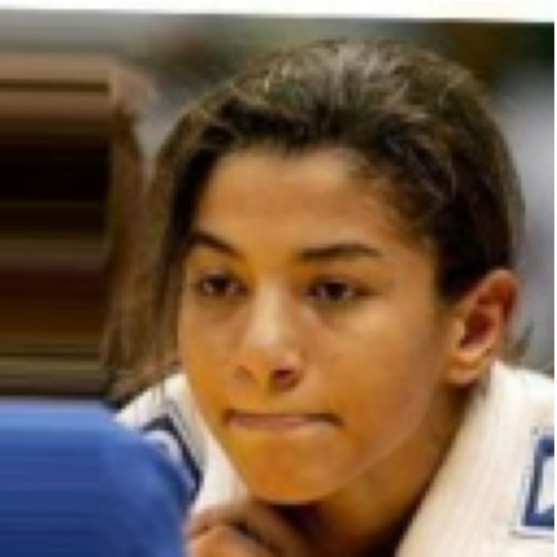}
       \makebox[\linewidth][c]{\footnotesize  }
    \end{minipage}
    \begin{minipage}[b]{0.13\textwidth}
    \captionsetup{skip=-0.01cm}
        \includegraphics[width=\linewidth]{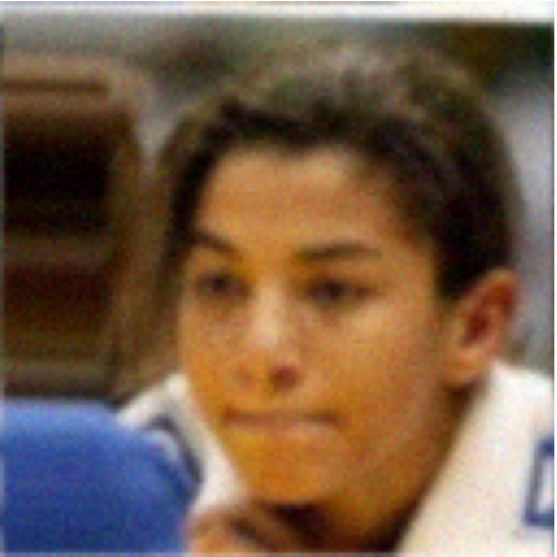}
       \makebox[\linewidth][c]{\footnotesize PSNR:  25.85}
    \end{minipage}
    \begin{minipage}[b]{0.13\textwidth}
    \captionsetup{skip=-0.01cm}
        \includegraphics[width=\linewidth]{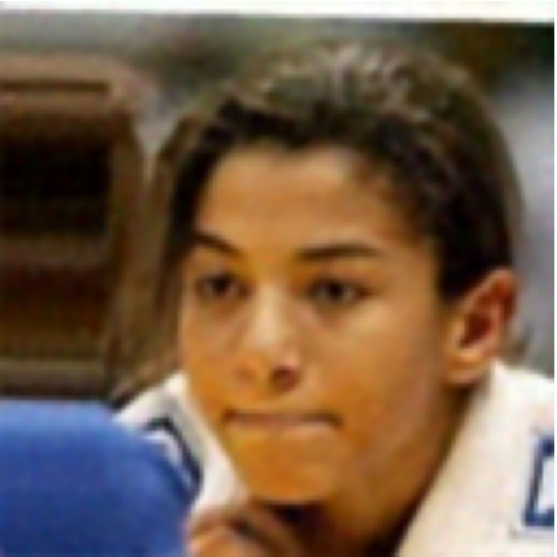}
       \makebox[\linewidth][c]{\footnotesize PSNR:  32.56}
    \end{minipage}
    \begin{minipage}[b]{0.13\textwidth}
    \captionsetup{skip=-0.01cm}
        \includegraphics[width=\linewidth]{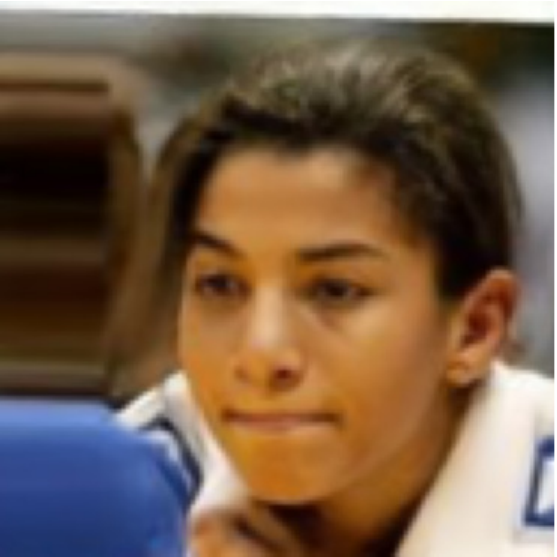}
       \makebox[\linewidth][c]{\footnotesize PSNR:  31.48}
    \end{minipage}
    \begin{minipage}[b]{0.13\textwidth}
    \captionsetup{skip=-0.01cm}
        \includegraphics[width=\linewidth]{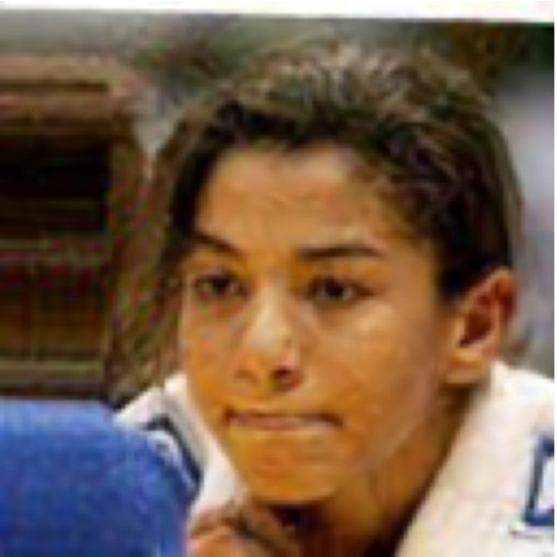}
       \makebox[\linewidth][c]{\footnotesize PSNR:  30.87}
    \end{minipage}
    \begin{minipage}[b]{0.13\textwidth}
    \captionsetup{skip=-0.01cm}
        \includegraphics[width=\linewidth]{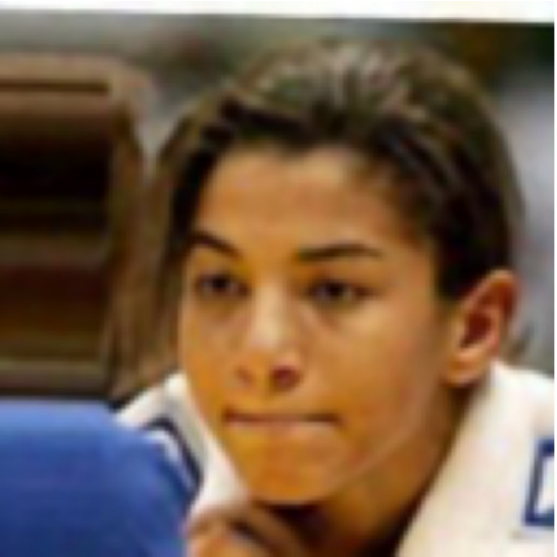}
       \makebox[\linewidth][c]{\footnotesize PSNR:  33.00}
    \end{minipage}
    \begin{minipage}[b]{0.13\textwidth}
    \captionsetup{skip=-0.01cm}
        \includegraphics[width=\linewidth]{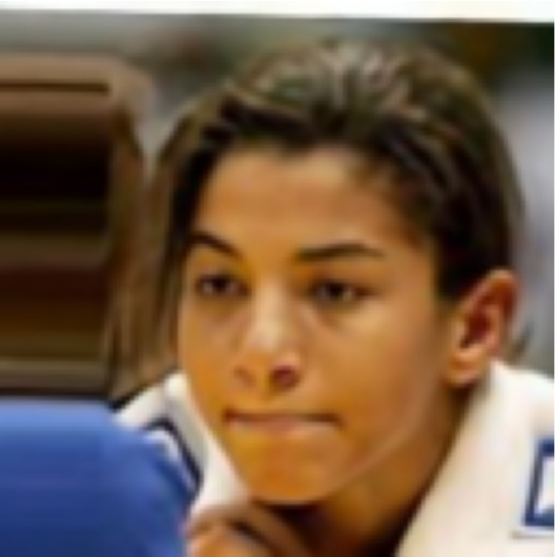}
       \makebox[\linewidth][c]{\footnotesize PSNR:  \textbf{33.86}}
    \end{minipage}

    \caption{Comparison of image deblurring results on CelebA.}
    \label{fig:celeba_deblur}
\end{figure}

     \vspace{5cm}

\begin{figure}[!ht]
    \centering

    \begin{minipage}[b]{0.13\textwidth}
        \centering
        {\footnotesize \text{Clean}}
    \end{minipage}
    \begin{minipage}[b]{0.13\textwidth}
        \centering
        \small Masked
    \end{minipage}
    \begin{minipage}[b]{0.13\textwidth}
        \centering
        \small PnP-GS
    \end{minipage}
    \begin{minipage}[b]{0.13\textwidth}
        \centering
        \small OT-ODE
    \end{minipage}
    \begin{minipage}[b]{0.13\textwidth}
        \centering
        \small Flow-Priors
    \end{minipage}
    \begin{minipage}[b]{0.13\textwidth}
        \centering
        \small PnP-Flow
    \end{minipage}
    \begin{minipage}[b]{0.13\textwidth}
        \centering
        \small Ours
    \end{minipage}

    \begin{minipage}[b]{0.13\textwidth}
    \captionsetup{skip=-0.01cm}
        \includegraphics[width=\linewidth]{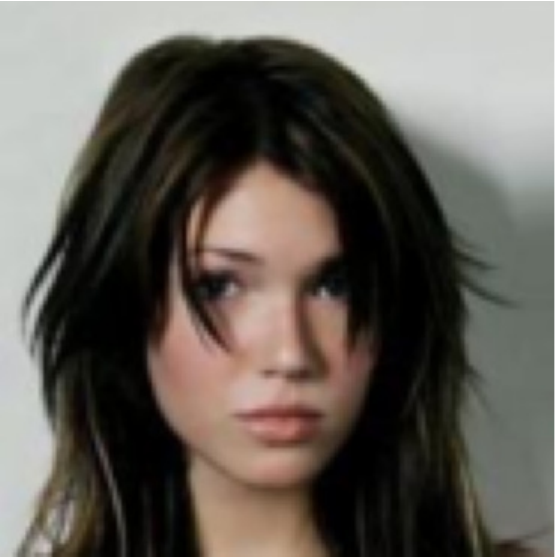}
        \makebox[\linewidth][c]{\footnotesize }
    \end{minipage}
    \begin{minipage}[b]{0.13\textwidth}
    \captionsetup{skip=-0.01cm}
        \includegraphics[width=\linewidth]{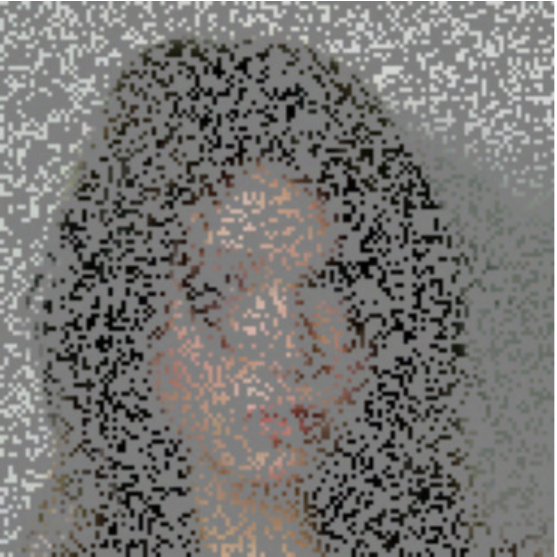}
        \makebox[\linewidth][c]{\footnotesize PSNR: 11.09}
    \end{minipage}
    \begin{minipage}[b]{0.13\textwidth}
    \captionsetup{skip=-0.01cm}
        \includegraphics[width=\linewidth]{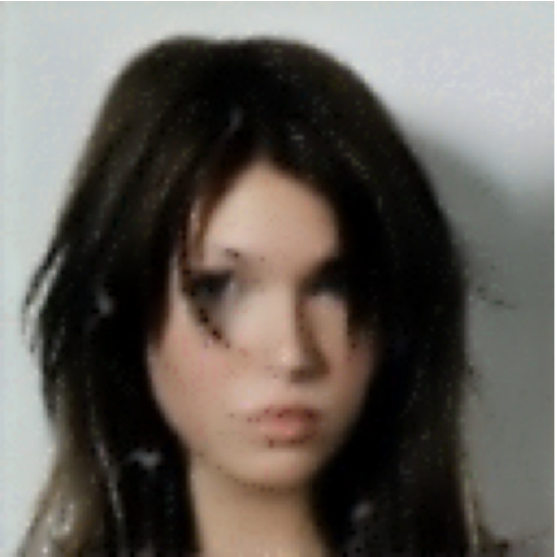}
       \makebox[\linewidth][c]{\footnotesize PSNR: 29.63}
    \end{minipage}
    \begin{minipage}[b]{0.13\textwidth}
    \captionsetup{skip=-0.01cm}
        \includegraphics[width=\linewidth]{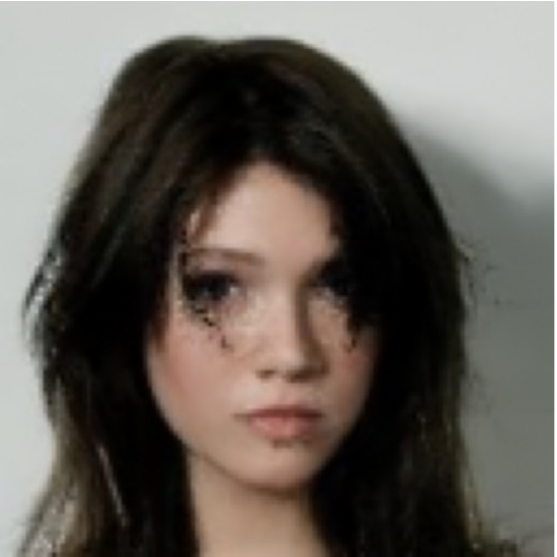}
       \makebox[\linewidth][c]{\footnotesize PSNR:  28.99}
    \end{minipage}
    \begin{minipage}[b]{0.13\textwidth}
    \captionsetup{skip=-0.01cm}
        \includegraphics[width=\linewidth]{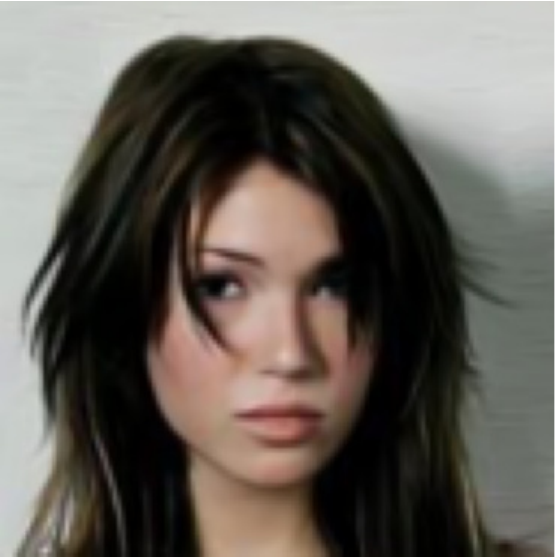}
       \makebox[\linewidth][c]{\footnotesize PSNR:  34.41}
    \end{minipage}
    \begin{minipage}[b]{0.13\textwidth}
    \captionsetup{skip=-0.01cm}
        \includegraphics[width=\linewidth]{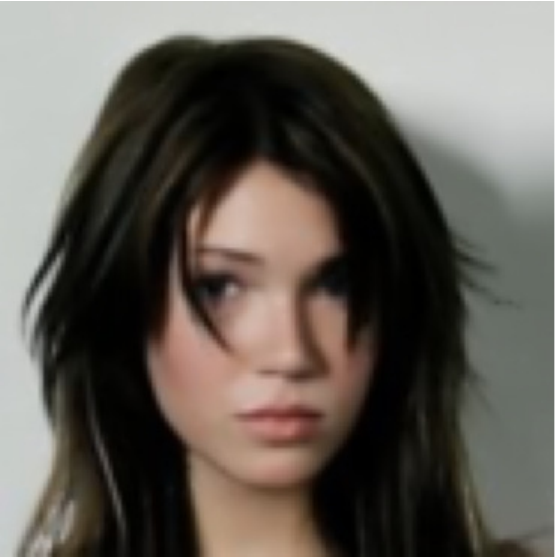}
       \makebox[\linewidth][c]{\footnotesize PSNR:  35.28}
    \end{minipage}
    \begin{minipage}[b]{0.13\textwidth}
    \captionsetup{skip=-0.01cm}
        \includegraphics[width=\linewidth]{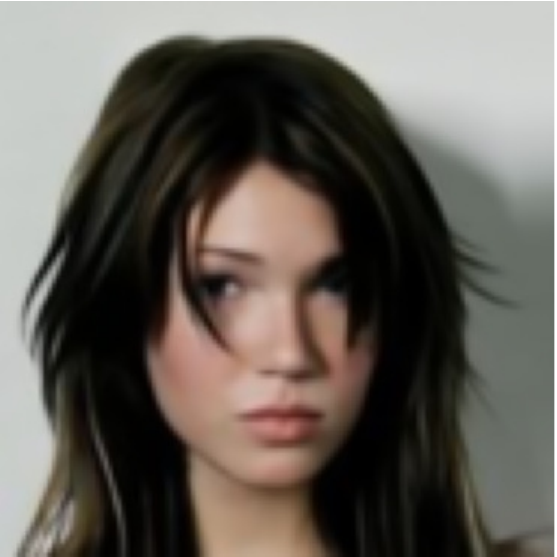}
       \makebox[\linewidth][c]{\footnotesize PSNR:  \textbf{36.40}}
    \end{minipage}

         \vspace{0.2cm}

    \begin{minipage}[b]{0.13\textwidth}
    \captionsetup{skip=-0.01cm}
        \includegraphics[width=\linewidth]{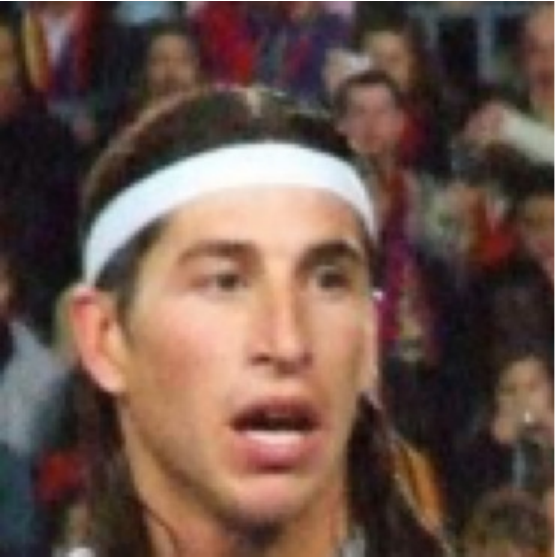}
       \makebox[\linewidth][c]{\footnotesize }
    \end{minipage}
    \begin{minipage}[b]{0.13\textwidth}
    \captionsetup{skip=-0.01cm}
        \includegraphics[width=\linewidth]{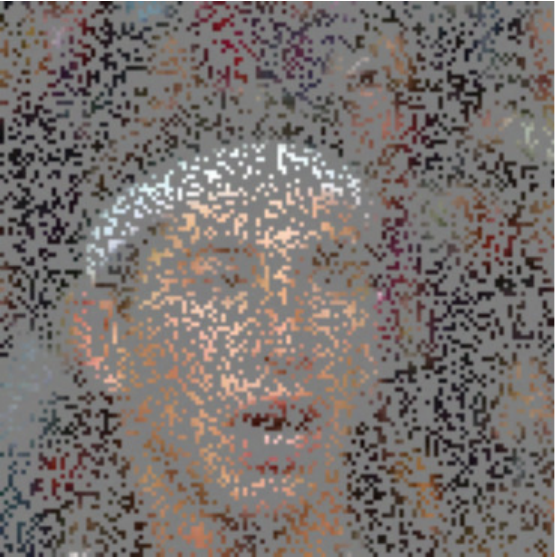}
       \makebox[\linewidth][c]{\footnotesize PSNR: 12.34}
    \end{minipage}
    \begin{minipage}[b]{0.13\textwidth}
    \captionsetup{skip=-0.01cm}
        \includegraphics[width=\linewidth]{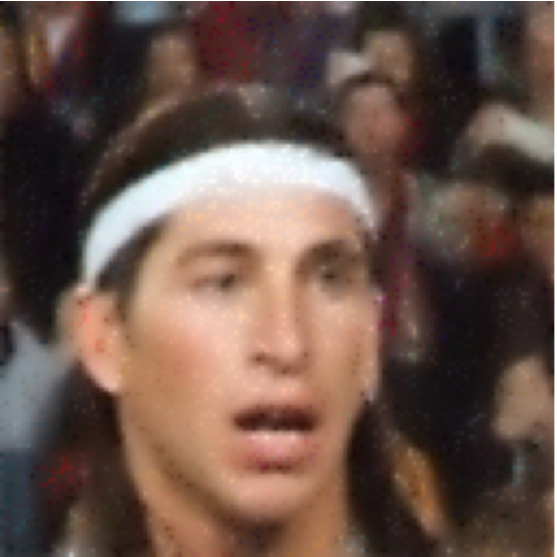}
       \makebox[\linewidth][c]{\footnotesize PSNR:  28.05}
    \end{minipage}
    \begin{minipage}[b]{0.13\textwidth}
    \captionsetup{skip=-0.01cm}
        \includegraphics[width=\linewidth]{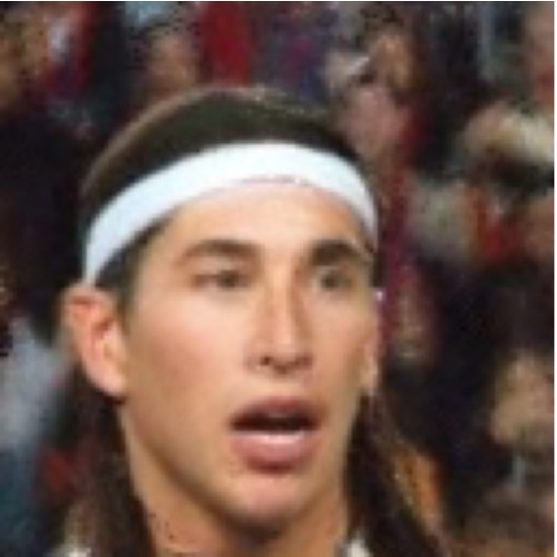}
       \makebox[\linewidth][c]{\footnotesize PSNR:  26.95}
    \end{minipage}
    \begin{minipage}[b]{0.13\textwidth}
    \captionsetup{skip=-0.01cm}
        \includegraphics[width=\linewidth]{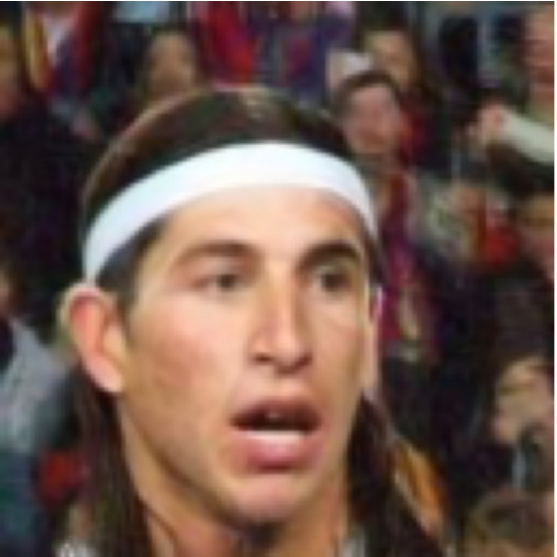}
       \makebox[\linewidth][c]{\footnotesize PSNR:  30.92}
    \end{minipage}
    \begin{minipage}[b]{0.13\textwidth}
    \captionsetup{skip=-0.01cm}
        \includegraphics[width=\linewidth]{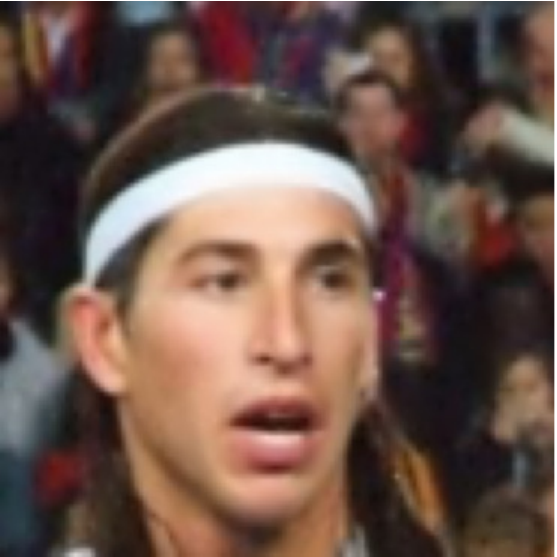}
       \makebox[\linewidth][c]{\footnotesize PSNR:  31.96}
    \end{minipage}
    \begin{minipage}[b]{0.13\textwidth}
    \captionsetup{skip=-0.01cm}
        \includegraphics[width=\linewidth]{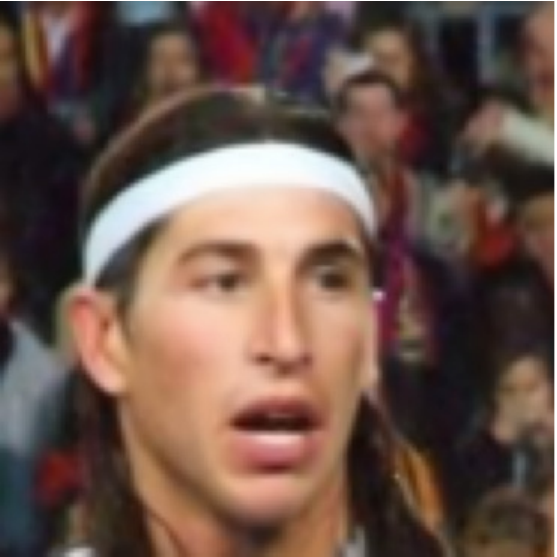}
       \makebox[\linewidth][c]{\footnotesize PSNR: \textbf{33.24}}
    \end{minipage}

     \vspace{0.2cm}

    \begin{minipage}[b]{0.13\textwidth}
    \captionsetup{skip=-0.01cm}
        \includegraphics[width=\linewidth]{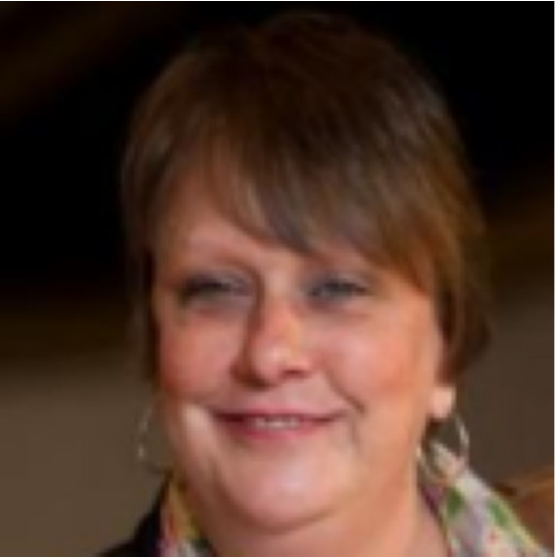}
       \makebox[\linewidth][c]{\footnotesize }
    \end{minipage}
    \begin{minipage}[b]{0.13\textwidth}
    \captionsetup{skip=-0.01cm}
        \includegraphics[width=\linewidth]{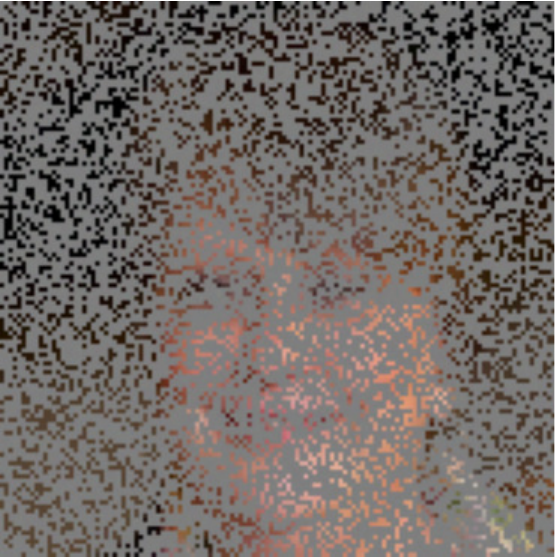}
       \makebox[\linewidth][c]{\footnotesize PSNR:  10.74}
    \end{minipage}
    \begin{minipage}[b]{0.13\textwidth}
    \captionsetup{skip=-0.01cm}
        \includegraphics[width=\linewidth]{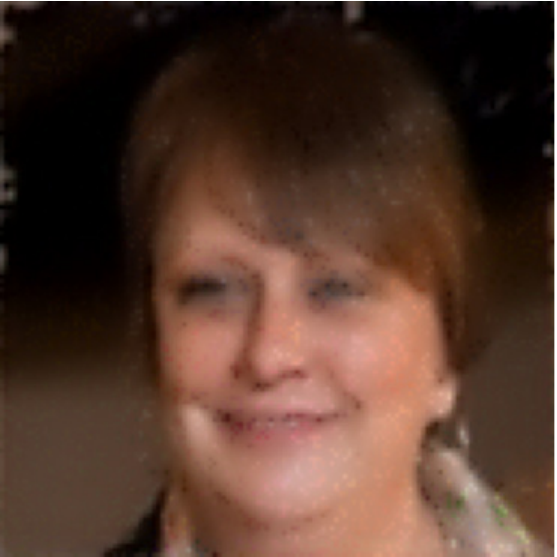}
       \makebox[\linewidth][c]{\footnotesize PSNR: 29.71}
    \end{minipage}
    \begin{minipage}[b]{0.13\textwidth}
    \captionsetup{skip=-0.01cm}
        \includegraphics[width=\linewidth]{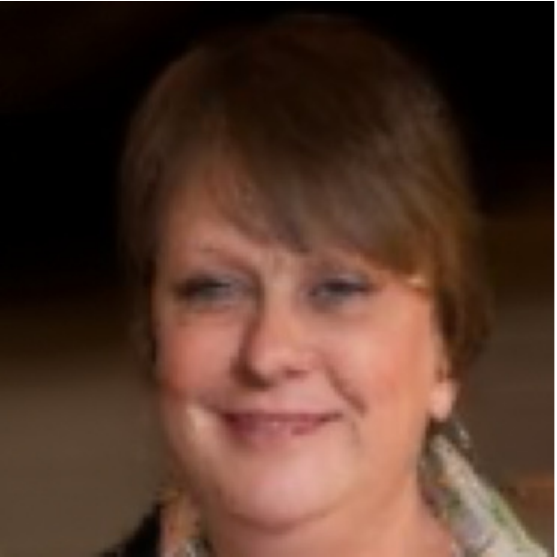}
       \makebox[\linewidth][c]{\footnotesize PSNR:  31.77}
    \end{minipage}
    \begin{minipage}[b]{0.13\textwidth}
    \captionsetup{skip=-0.01cm}
        \includegraphics[width=\linewidth]{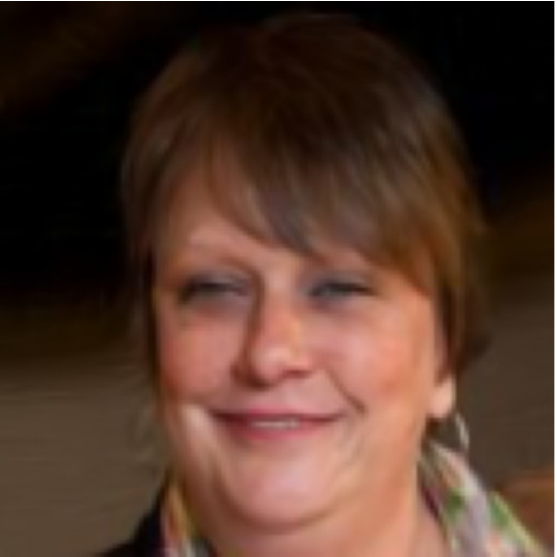}
       \makebox[\linewidth][c]{\footnotesize PSNR:  35.65}
    \end{minipage}
    \begin{minipage}[b]{0.13\textwidth}
    \captionsetup{skip=-0.01cm}
        \includegraphics[width=\linewidth]{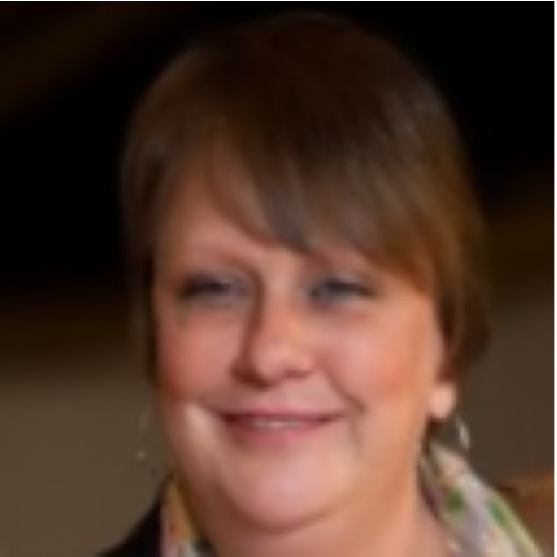}
       \makebox[\linewidth][c]{\footnotesize PSNR:  36.86}
    \end{minipage}
    \begin{minipage}[b]{0.13\textwidth}
    \captionsetup{skip=-0.01cm}
        \includegraphics[width=\linewidth]{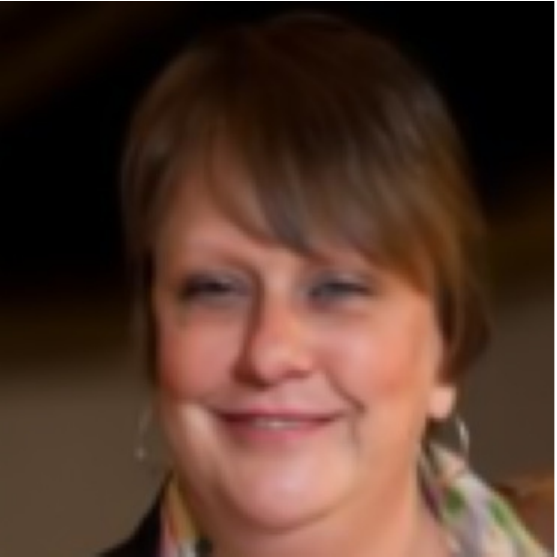}
       \makebox[\linewidth][c]{\footnotesize PSNR:  \textbf{38.27}}
    \end{minipage}

    \caption{Comparison of random inpainting results on CelebA.}
    \label{fig:celeba_inpaint}

\end{figure}

     \vspace{5cm}

\begin{figure}[!ht]
    \centering

    \begin{minipage}[b]{0.13\textwidth}
        \centering
        {\footnotesize \text{Clean}}
    \end{minipage}
    \begin{minipage}[b]{0.13\textwidth}
        \centering
        \small Low-resolution
    \end{minipage}
    \begin{minipage}[b]{0.13\textwidth}
        \centering
        \small PnP-GS
    \end{minipage}
    \begin{minipage}[b]{0.13\textwidth}
        \centering
        \small OT-ODE
    \end{minipage}
    \begin{minipage}[b]{0.13\textwidth}
        \centering
        \small Flow-Priors
    \end{minipage}
    \begin{minipage}[b]{0.13\textwidth}
        \centering
        \small PnP-Flow
    \end{minipage}
    \begin{minipage}[b]{0.13\textwidth}
        \centering
        \small Ours
    \end{minipage}

    \begin{minipage}[b]{0.13\textwidth}
    \captionsetup{skip=-0.01cm}
        \includegraphics[width=\linewidth]{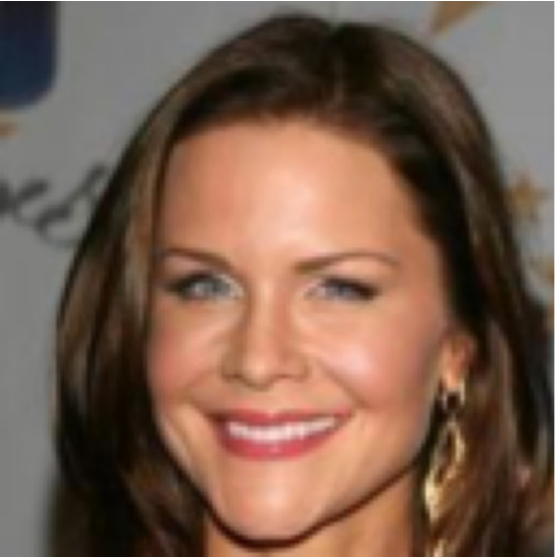}
        \makebox[\linewidth][c]{\footnotesize }
    \end{minipage}
    \begin{minipage}[b]{0.13\textwidth}
    \captionsetup{skip=-0.01cm}
        \includegraphics[width=\linewidth]{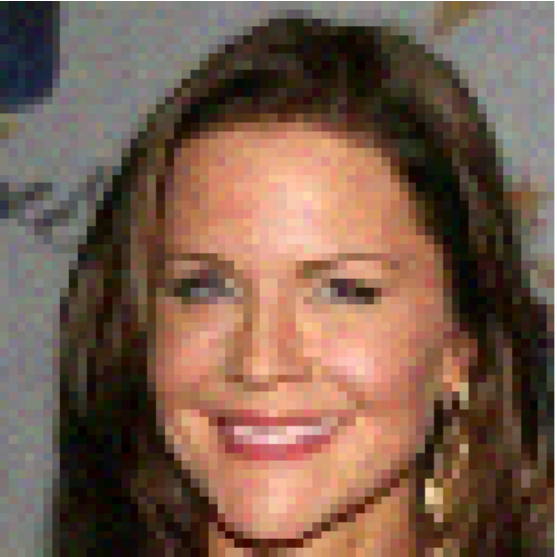}
        \makebox[\linewidth][c]{\footnotesize PSNR: 7.94}
    \end{minipage}
    \begin{minipage}[b]{0.13\textwidth}
    \captionsetup{skip=-0.01cm}
        \includegraphics[width=\linewidth]{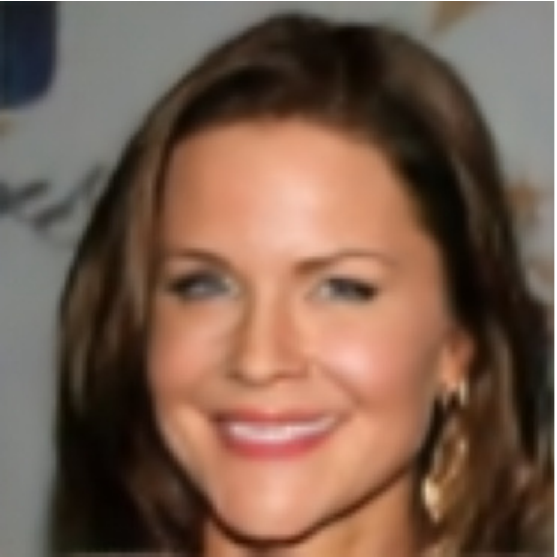}
       \makebox[\linewidth][c]{\footnotesize PSNR: 31.67}
    \end{minipage}
    \begin{minipage}[b]{0.13\textwidth}
    \captionsetup{skip=-0.01cm}
        \includegraphics[width=\linewidth]{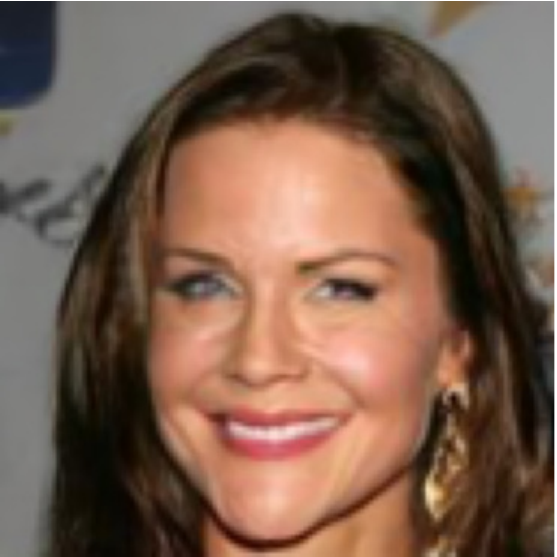}
       \makebox[\linewidth][c]{\footnotesize PSNR:  31.46}
    \end{minipage}
    \begin{minipage}[b]{0.13\textwidth}
    \captionsetup{skip=-0.01cm}
        \includegraphics[width=\linewidth]{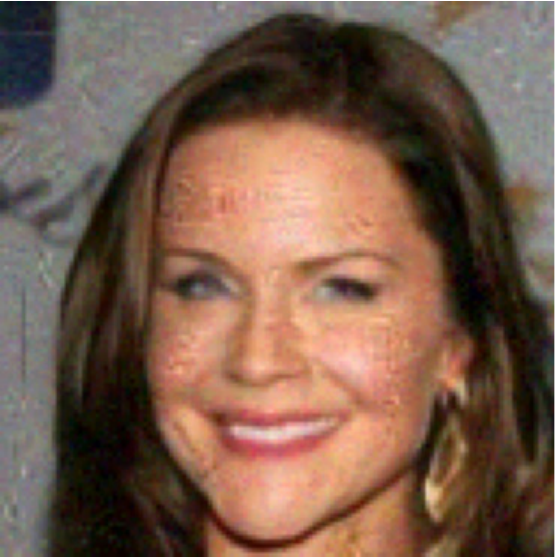}
       \makebox[\linewidth][c]{\footnotesize PSNR:  28.53}
    \end{minipage}
    \begin{minipage}[b]{0.13\textwidth}
    \captionsetup{skip=-0.01cm}
        \includegraphics[width=\linewidth]{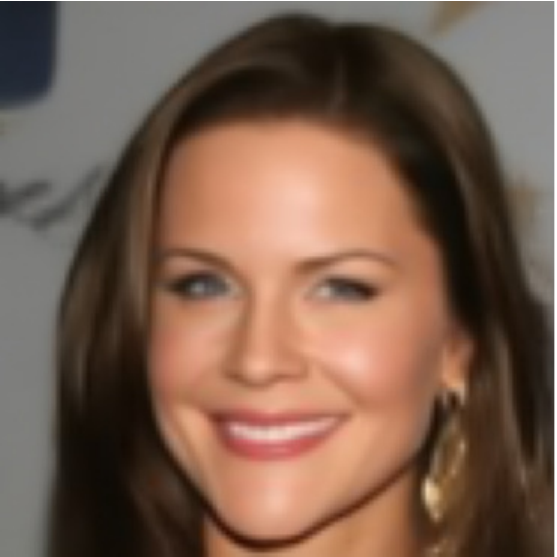}
       \makebox[\linewidth][c]{\footnotesize PSNR:  31.63}
    \end{minipage}
    \begin{minipage}[b]{0.13\textwidth}
    \captionsetup{skip=-0.01cm}
        \includegraphics[width=\linewidth]{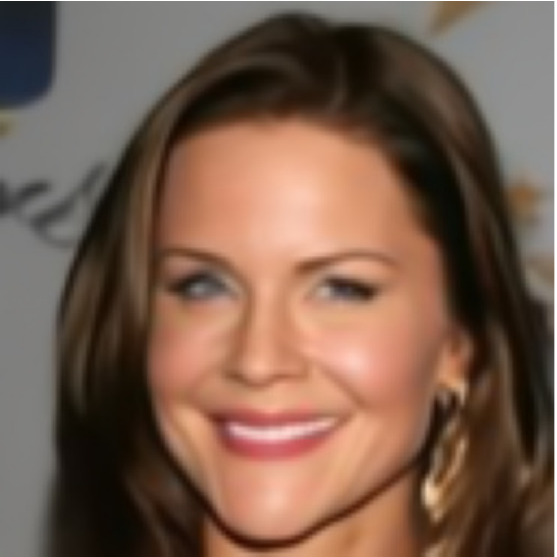}
       \makebox[\linewidth][c]{\footnotesize PSNR:  \textbf{33.42}}
    \end{minipage}

         \vspace{0.2cm}

    \begin{minipage}[b]{0.13\textwidth}
    \captionsetup{skip=-0.01cm}
        \includegraphics[width=\linewidth]{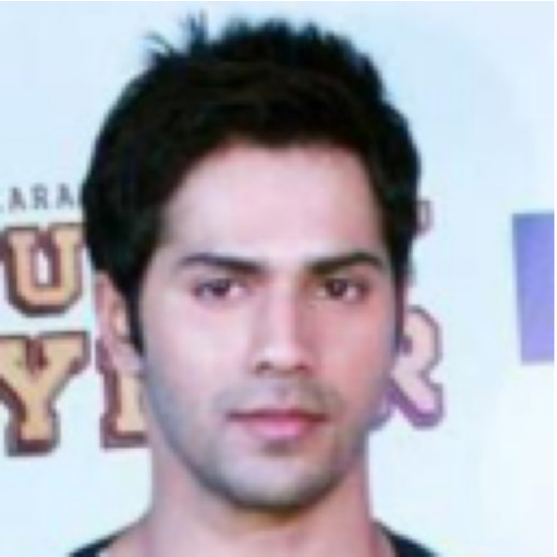}
       \makebox[\linewidth][c]{\footnotesize }
    \end{minipage}
    \begin{minipage}[b]{0.13\textwidth}
    \captionsetup{skip=-0.01cm}
        \includegraphics[width=\linewidth]{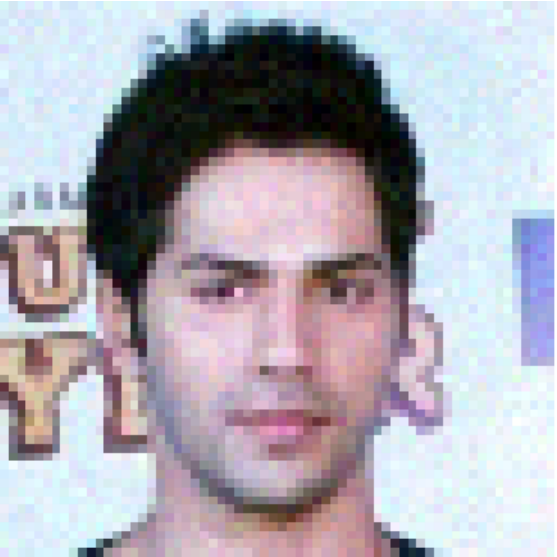}
       \makebox[\linewidth][c]{\footnotesize PSNR: 8.49}
    \end{minipage}
    \begin{minipage}[b]{0.13\textwidth}
    \captionsetup{skip=-0.01cm}
        \includegraphics[width=\linewidth]{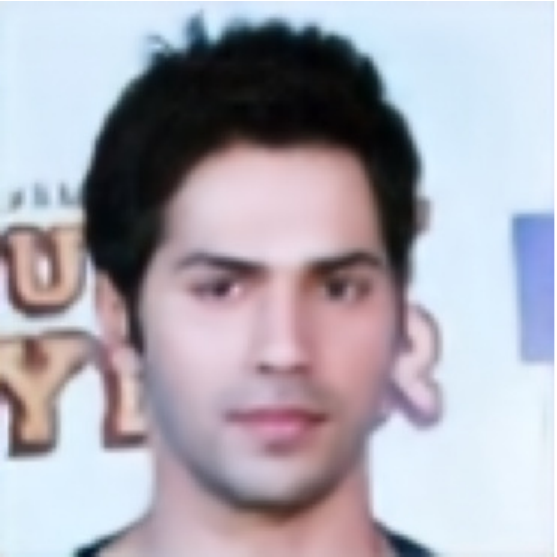}
       \makebox[\linewidth][c]{\footnotesize PSNR: 29.85}
    \end{minipage}
    \begin{minipage}[b]{0.13\textwidth}
    \captionsetup{skip=-0.01cm}
        \includegraphics[width=\linewidth]{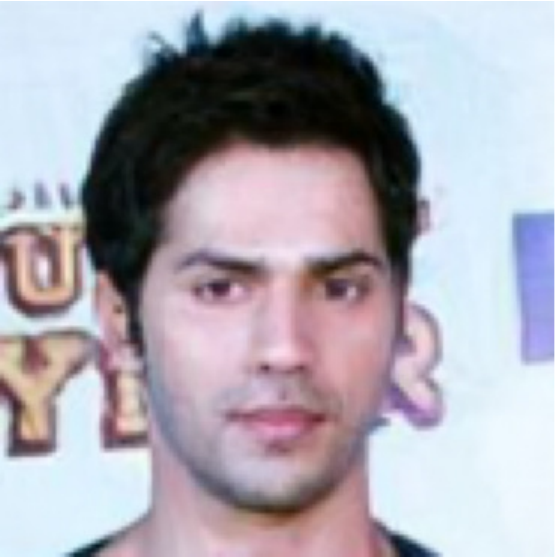}
       \makebox[\linewidth][c]{\footnotesize PSNR:  29.81}
    \end{minipage}
    \begin{minipage}[b]{0.13\textwidth}
    \captionsetup{skip=-0.01cm}
        \includegraphics[width=\linewidth]{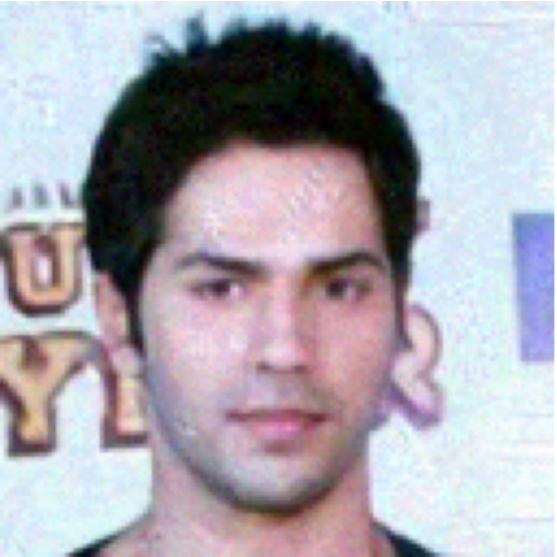}
       \makebox[\linewidth][c]{\footnotesize PSNR:  27.98}
    \end{minipage}
    \begin{minipage}[b]{0.13\textwidth}
    \captionsetup{skip=-0.01cm}
        \includegraphics[width=\linewidth]{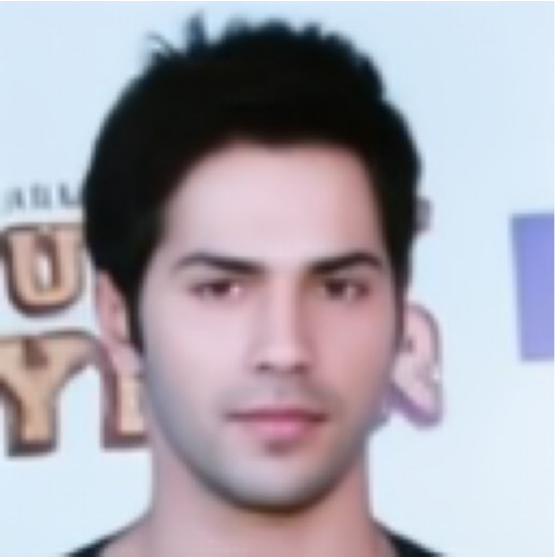}
       \makebox[\linewidth][c]{\footnotesize PSNR:  29.59}
    \end{minipage}
    \begin{minipage}[b]{0.13\textwidth}
    \captionsetup{skip=-0.01cm}
        \includegraphics[width=\linewidth]{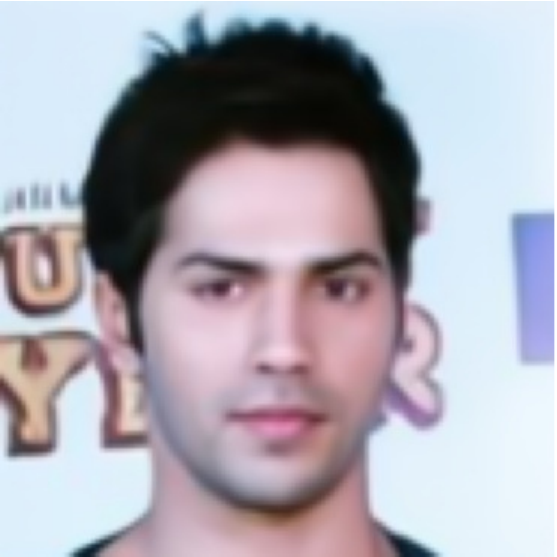}
       \makebox[\linewidth][c]{\footnotesize PSNR: \textbf{31.37}}
    \end{minipage}

     \vspace{0.2cm}

    \begin{minipage}[b]{0.13\textwidth}
    \captionsetup{skip=-0.01cm}
        \includegraphics[width=\linewidth]{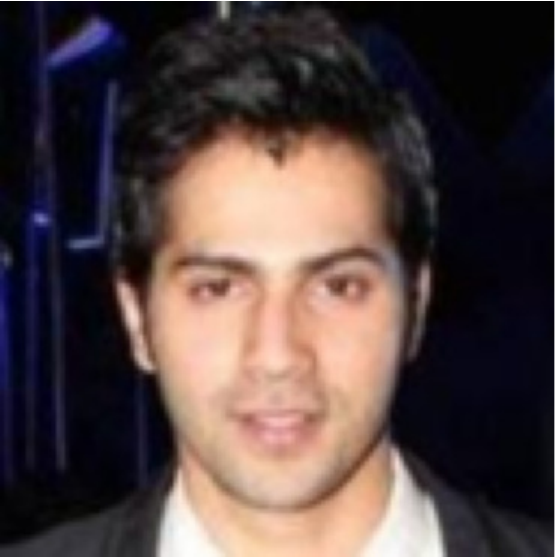}
       \makebox[\linewidth][c]{\footnotesize }
    \end{minipage}
    \begin{minipage}[b]{0.13\textwidth}
    \captionsetup{skip=-0.01cm}
        \includegraphics[width=\linewidth]{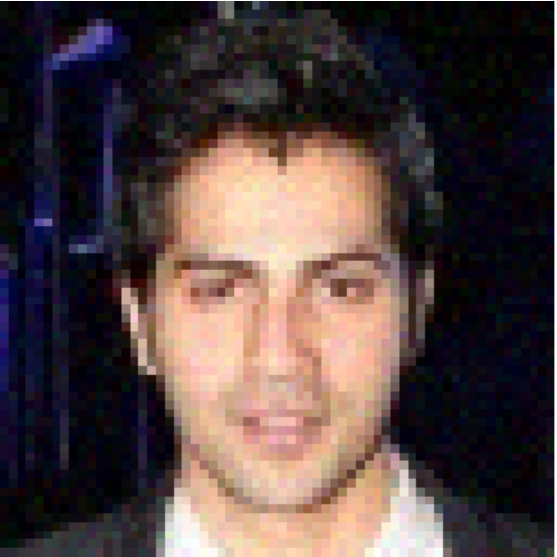}
       \makebox[\linewidth][c]{\footnotesize PSNR:  5.80}
    \end{minipage}
    \begin{minipage}[b]{0.13\textwidth}
    \captionsetup{skip=-0.01cm}
        \includegraphics[width=\linewidth]{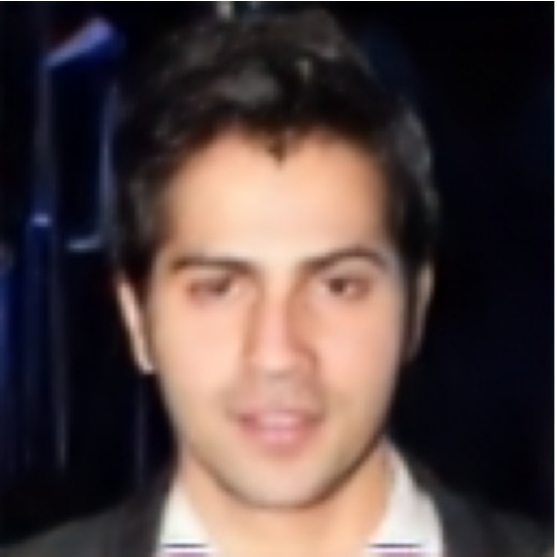}
       \makebox[\linewidth][c]{\footnotesize PSNR:  30.53}
    \end{minipage}
    \begin{minipage}[b]{0.13\textwidth}
    \captionsetup{skip=-0.01cm}
        \includegraphics[width=\linewidth]{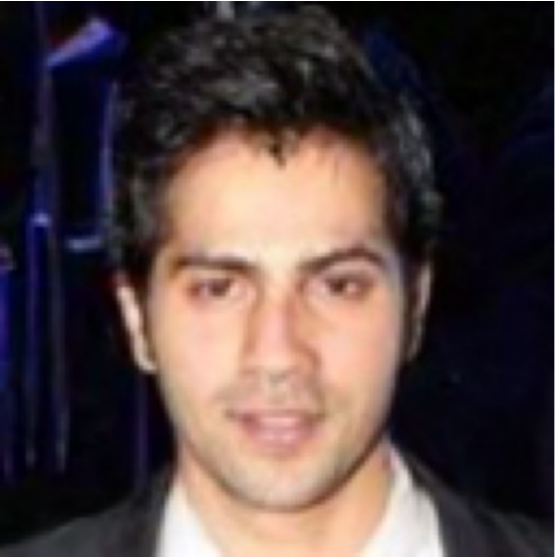}
       \makebox[\linewidth][c]{\footnotesize PSNR:  32.57}
    \end{minipage}
    \begin{minipage}[b]{0.13\textwidth}
    \captionsetup{skip=-0.01cm}
        \includegraphics[width=\linewidth]{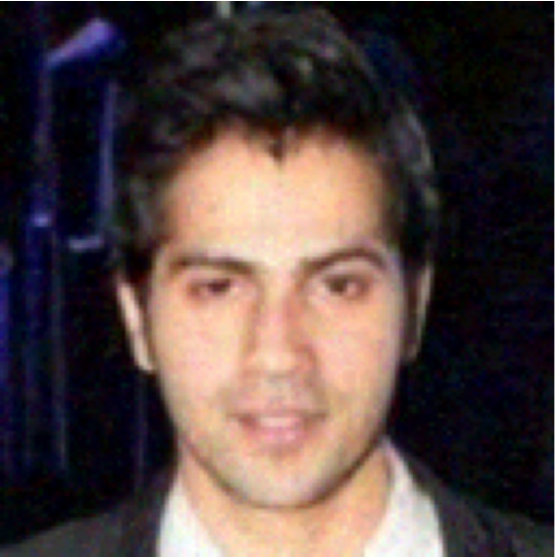}
       \makebox[\linewidth][c]{\footnotesize PSNR:  29.50}
    \end{minipage}
    \begin{minipage}[b]{0.13\textwidth}
    \captionsetup{skip=-0.01cm}
        \includegraphics[width=\linewidth]{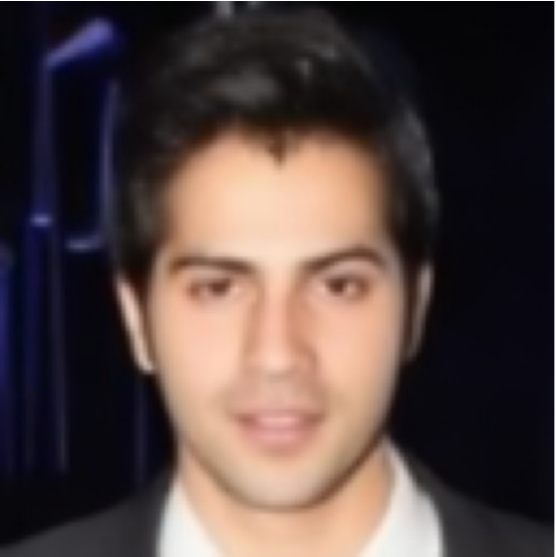}
       \makebox[\linewidth][c]{\footnotesize PSNR:  32.24}
    \end{minipage}
    \begin{minipage}[b]{0.13\textwidth}
    \captionsetup{skip=-0.01cm}
        \includegraphics[width=\linewidth]{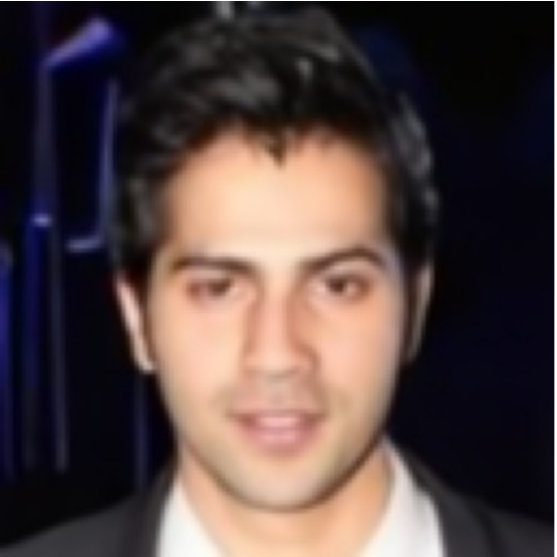}
       \makebox[\linewidth][c]{\footnotesize PSNR:  \textbf{34.35}}
    \end{minipage}

    \caption{Comparison of super-resolution results on CelebA.}
    \label{fig:celeba_sr}

\end{figure}

     \vspace{5cm}

\begin{figure}[!ht]
    \centering

    \begin{minipage}[b]{0.13\textwidth}
        \centering
        {\footnotesize \text{Clean}}
    \end{minipage}
    \begin{minipage}[b]{0.13\textwidth}
        \centering
        \small Masked
    \end{minipage}
    \begin{minipage}[b]{0.13\textwidth}
        \centering
        \small PnP-GS
    \end{minipage}
    \begin{minipage}[b]{0.13\textwidth}
        \centering
        \small OT-ODE
    \end{minipage}
    \begin{minipage}[b]{0.13\textwidth}
        \centering
        \small Flow-Priors
    \end{minipage}
    \begin{minipage}[b]{0.13\textwidth}
        \centering
        \small PnP-Flow
    \end{minipage}
    \begin{minipage}[b]{0.13\textwidth}
        \centering
        \small Ours
    \end{minipage}

    \begin{minipage}[b]{0.13\textwidth}
    \captionsetup{skip=-0.01cm}
        \includegraphics[width=\linewidth]{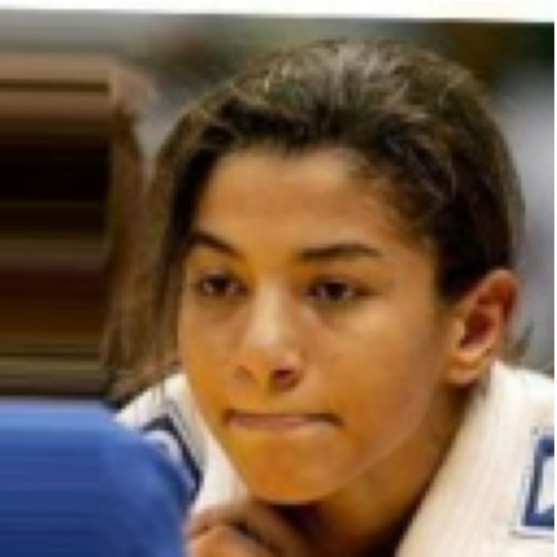}
        \makebox[\linewidth][c]{\footnotesize }
    \end{minipage}
    \begin{minipage}[b]{0.13\textwidth}
    \captionsetup{skip=-0.01cm}
        \includegraphics[width=\linewidth]{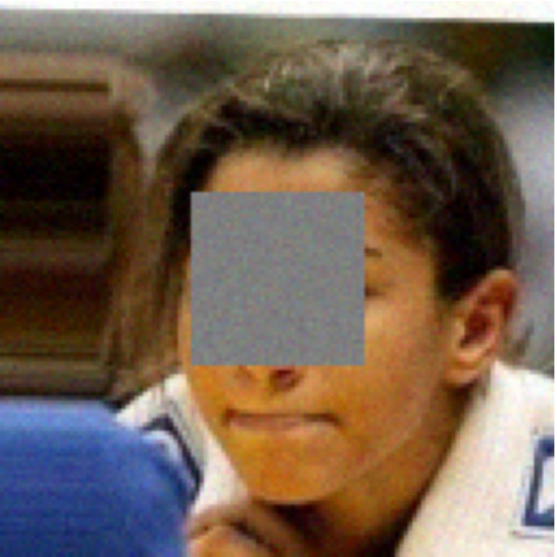}
        \makebox[\linewidth][c]{\footnotesize PSNR: 22.36}
    \end{minipage}
    \begin{minipage}[b]{0.13\textwidth}
    \captionsetup{skip=-0.01cm}
        \centering
       \makebox[\linewidth][c]{\footnotesize PSNR: N/A}
    \end{minipage}
    \begin{minipage}[b]{0.13\textwidth}
    \captionsetup{skip=-0.01cm}
        \includegraphics[width=\linewidth]{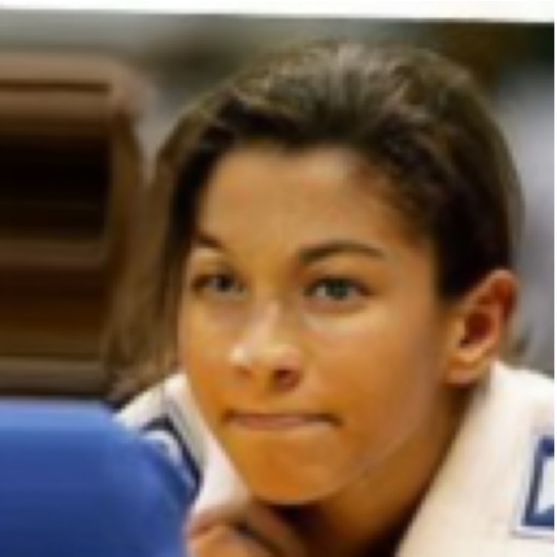}
       \makebox[\linewidth][c]{\footnotesize PSNR:  30.46}
    \end{minipage}
    \begin{minipage}[b]{0.13\textwidth}
    \captionsetup{skip=-0.01cm}
        \includegraphics[width=\linewidth]{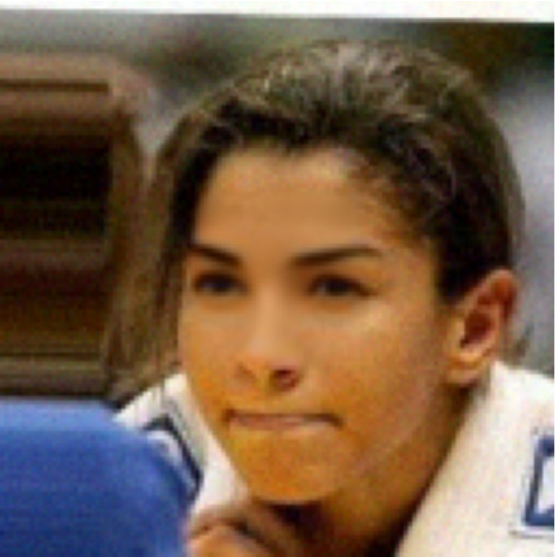}
       \makebox[\linewidth][c]{\footnotesize PSNR:  30.17}
    \end{minipage}
    \begin{minipage}[b]{0.13\textwidth}
    \captionsetup{skip=-0.01cm}
        \includegraphics[width=\linewidth]{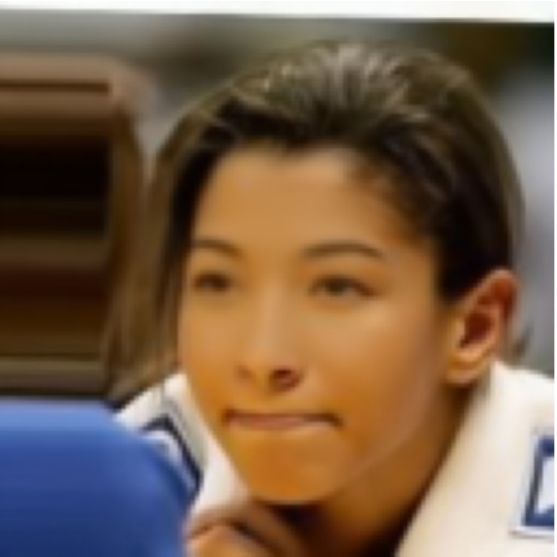}
       \makebox[\linewidth][c]{\footnotesize PSNR:  30.23}
    \end{minipage}
    \begin{minipage}[b]{0.13\textwidth}
    \captionsetup{skip=-0.01cm}
        \includegraphics[width=\linewidth]{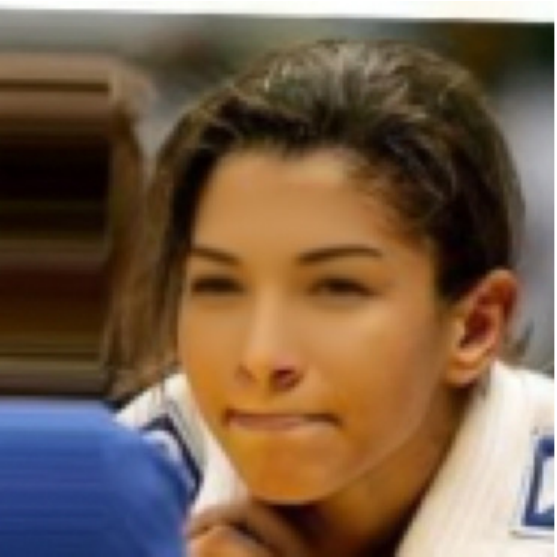}
       \makebox[\linewidth][c]{\footnotesize PSNR:  \textbf{31.44}}
    \end{minipage}

         \vspace{0.2cm}

    \begin{minipage}[b]{0.13\textwidth}
    \captionsetup{skip=-0.01cm}
        \includegraphics[width=\linewidth]{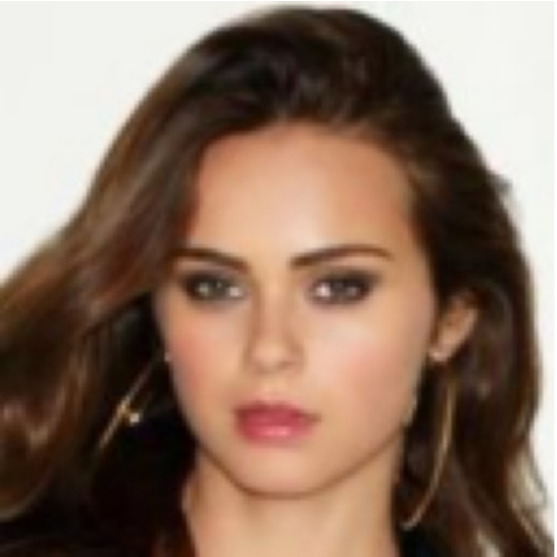}
       \makebox[\linewidth][c]{\footnotesize }
    \end{minipage}
    \begin{minipage}[b]{0.13\textwidth}
    \captionsetup{skip=-0.01cm}
        \includegraphics[width=\linewidth]{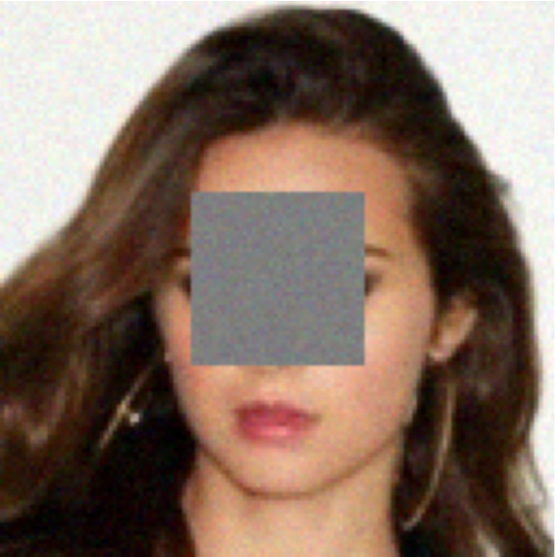}
       \makebox[\linewidth][c]{\footnotesize PSNR: 22.11}
    \end{minipage}
    \begin{minipage}[b]{0.13\textwidth}
    \captionsetup{skip=-0.01cm}
        \centering
       \makebox[\linewidth][c]{\footnotesize PSNR:  N/A}
    \end{minipage}
    \begin{minipage}[b]{0.13\textwidth}
    \captionsetup{skip=-0.01cm}
        \includegraphics[width=\linewidth]{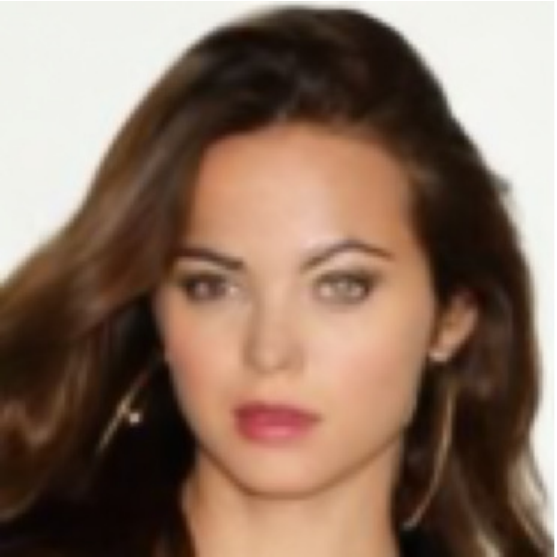}
       \makebox[\linewidth][c]{\footnotesize PSNR:  30.96}
    \end{minipage}
    \begin{minipage}[b]{0.13\textwidth}
    \captionsetup{skip=-0.01cm}
        \includegraphics[width=\linewidth]{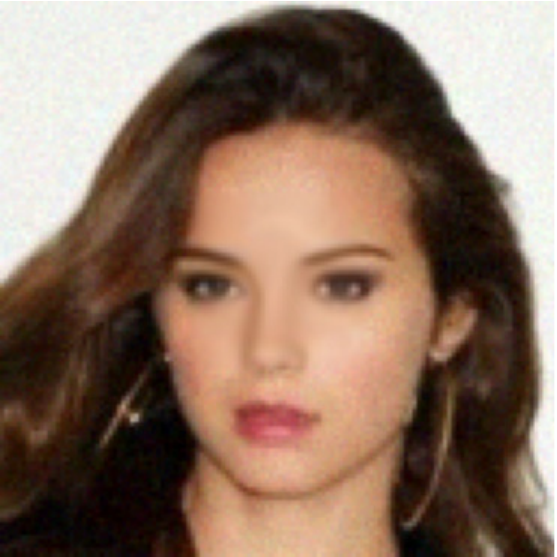}
       \makebox[\linewidth][c]{\footnotesize PSNR:  31.83}
    \end{minipage}
    \begin{minipage}[b]{0.13\textwidth}
    \captionsetup{skip=-0.01cm}
        \includegraphics[width=\linewidth]{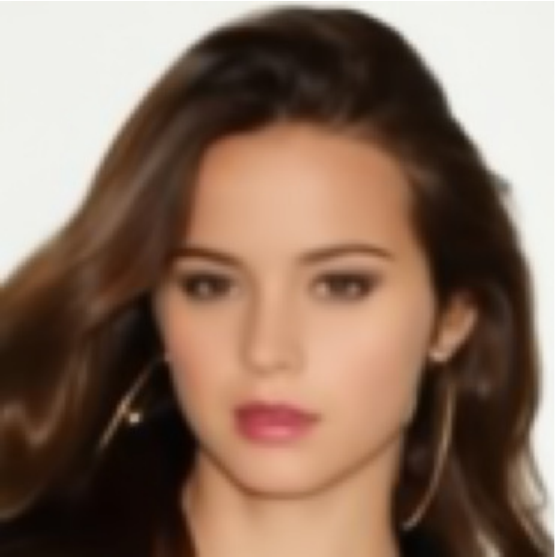}
       \makebox[\linewidth][c]{\footnotesize PSNR:  33.30}
    \end{minipage}
    \begin{minipage}[b]{0.13\textwidth}
    \captionsetup{skip=-0.01cm}
        \includegraphics[width=\linewidth]{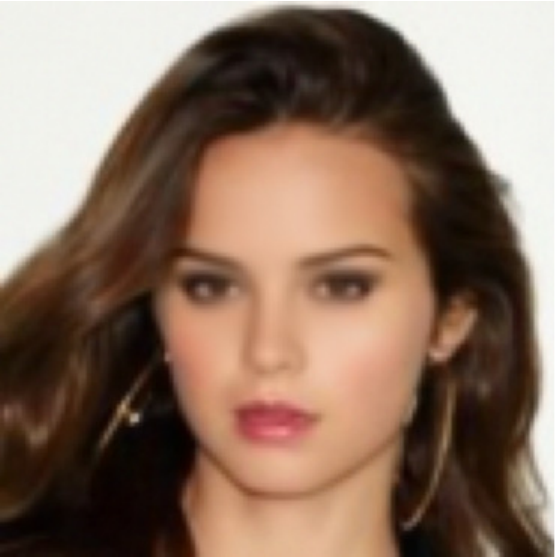}
       \makebox[\linewidth][c]{\footnotesize PSNR: \textbf{34.82}}
    \end{minipage}

     \vspace{0.2cm}

    \begin{minipage}[b]{0.13\textwidth}
    \captionsetup{skip=-0.01cm}
        \includegraphics[width=\linewidth]{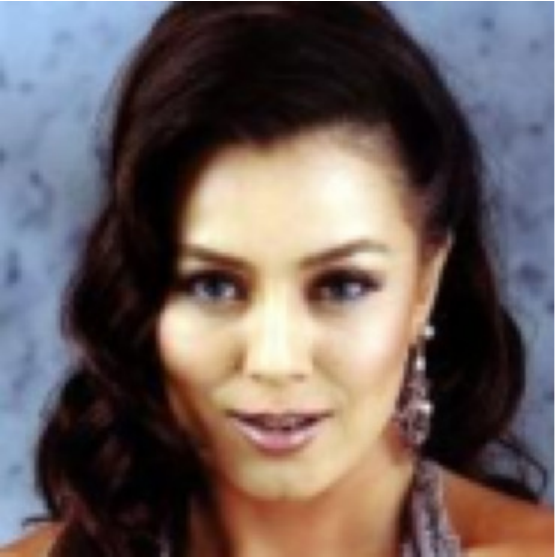}
       \makebox[\linewidth][c]{\footnotesize }
    \end{minipage}
    \begin{minipage}[b]{0.13\textwidth}
    \captionsetup{skip=-0.01cm}
        \includegraphics[width=\linewidth]{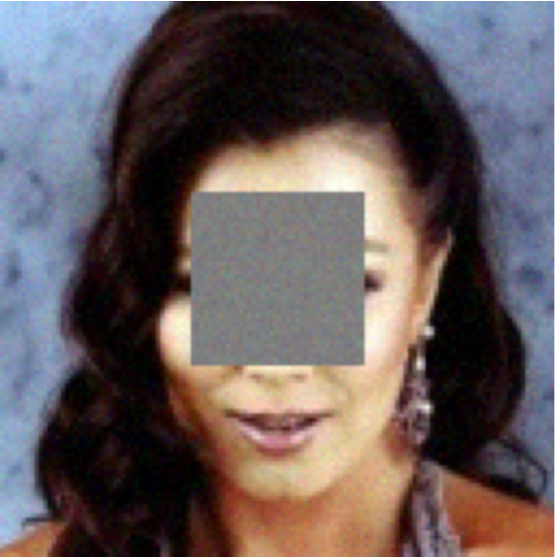}
       \makebox[\linewidth][c]{\footnotesize PSNR:  19.36}
    \end{minipage}
    \begin{minipage}[b]{0.13\textwidth}
    \captionsetup{skip=-0.01cm}
        \centering
       \makebox[\linewidth][c]{\footnotesize PSNR: N/A}
    \end{minipage}
    \begin{minipage}[b]{0.13\textwidth}
    \captionsetup{skip=-0.01cm}
        \includegraphics[width=\linewidth]{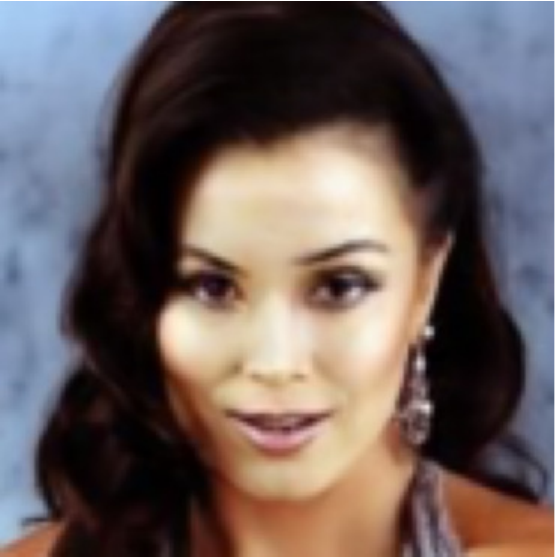}
       \makebox[\linewidth][c]{\footnotesize PSNR:  29.68}
    \end{minipage}
    \begin{minipage}[b]{0.13\textwidth}
    \captionsetup{skip=-0.01cm}
        \includegraphics[width=\linewidth]{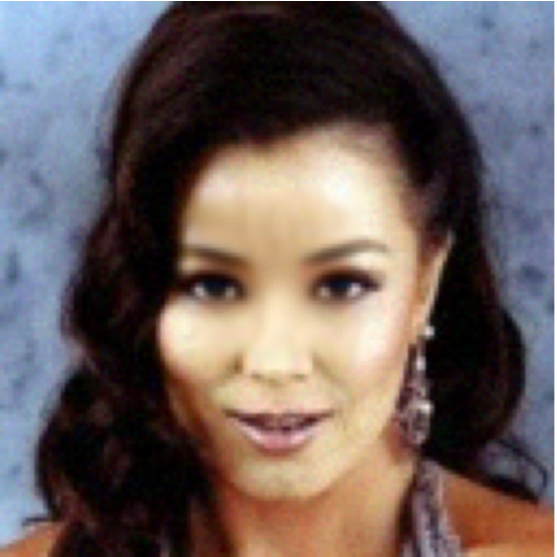}
       \makebox[\linewidth][c]{\footnotesize PSNR:  28.52}
    \end{minipage}
    \begin{minipage}[b]{0.13\textwidth}
    \captionsetup{skip=-0.01cm}
        \includegraphics[width=\linewidth]{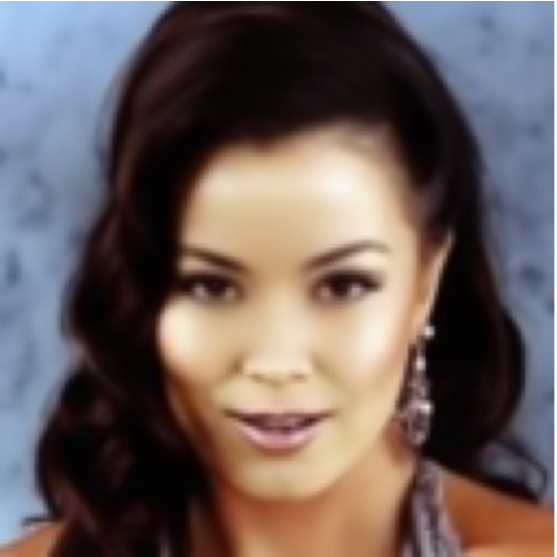}
       \makebox[\linewidth][c]{\footnotesize PSNR: 28.94}
    \end{minipage}
    \begin{minipage}[b]{0.13\textwidth}
    \captionsetup{skip=-0.01cm}
        \includegraphics[width=\linewidth]{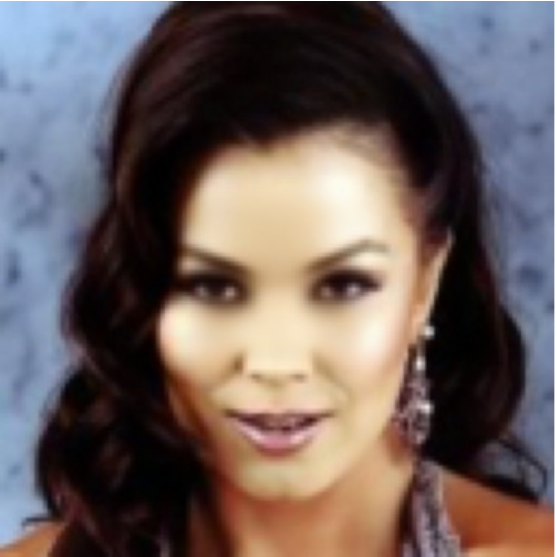}
       \makebox[\linewidth][c]{\footnotesize PSNR:  \textbf{32.02}}
    \end{minipage}

    \caption{Comparison of box inpainting results on CelebA.}
    \label{fig:celeba_inpaint-2}

\end{figure}

\begin{figure}[!ht]
    \centering

    \begin{minipage}[b]{0.13\textwidth}
        \centering
        {\footnotesize \text{Clean}}
    \end{minipage}
    \begin{minipage}[b]{0.13\textwidth}
        \centering
        \small Noisy
    \end{minipage}
    \begin{minipage}[b]{0.13\textwidth}
        \centering
        \small PnP-GS
    \end{minipage}
    \begin{minipage}[b]{0.13\textwidth}
        \centering
        \small OT-ODE
    \end{minipage}
    \begin{minipage}[b]{0.13\textwidth}
        \centering
        \small Flow-Priors
    \end{minipage}
    \begin{minipage}[b]{0.13\textwidth}
        \centering
        \small PnP-Flow
    \end{minipage}
    \begin{minipage}[b]{0.13\textwidth}
        \centering
        \small Ours
    \end{minipage}

    \begin{minipage}[b]{0.13\textwidth}
    \captionsetup{skip=-0.01cm}
        \includegraphics[width=\linewidth]{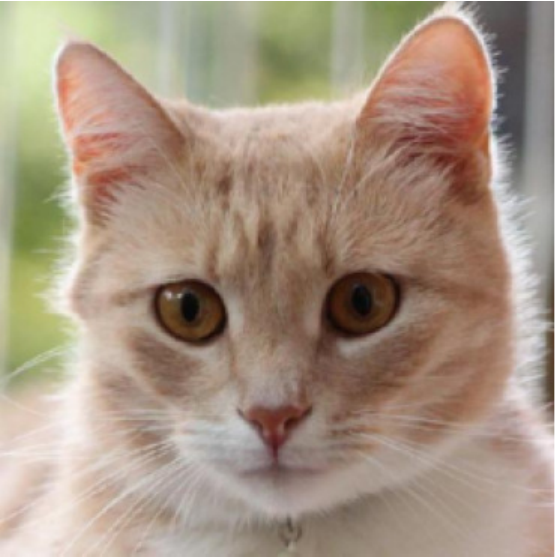}
       \makebox[\linewidth][c]{\footnotesize }
    \end{minipage}
    \begin{minipage}[b]{0.13\textwidth}
    \captionsetup{skip=-0.01cm}
        \includegraphics[width=\linewidth]{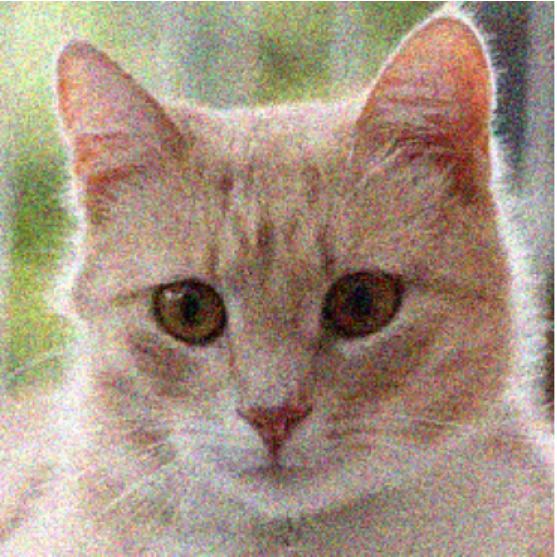}
       \makebox[\linewidth][c]{\footnotesize PSNR: 20.00}
    \end{minipage}
    \begin{minipage}[b]{0.13\textwidth}
    \captionsetup{skip=-0.01cm}
        \includegraphics[width=\linewidth]{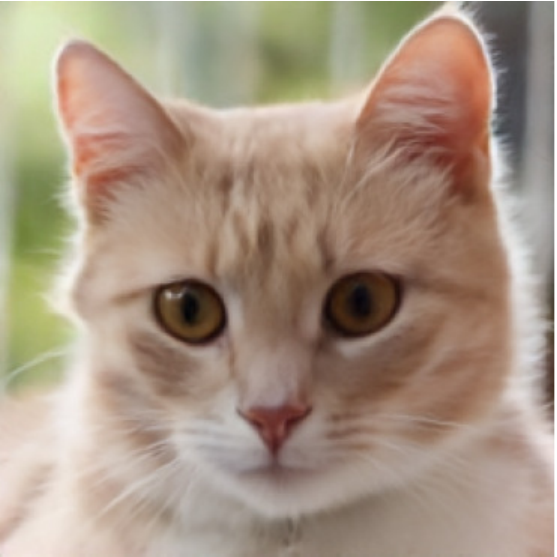}
       \makebox[\linewidth][c]{\footnotesize PSNR:  \textbf{33.43}}
    \end{minipage}
    \begin{minipage}[b]{0.13\textwidth}
    \captionsetup{skip=-0.01cm}
        \includegraphics[width=\linewidth]{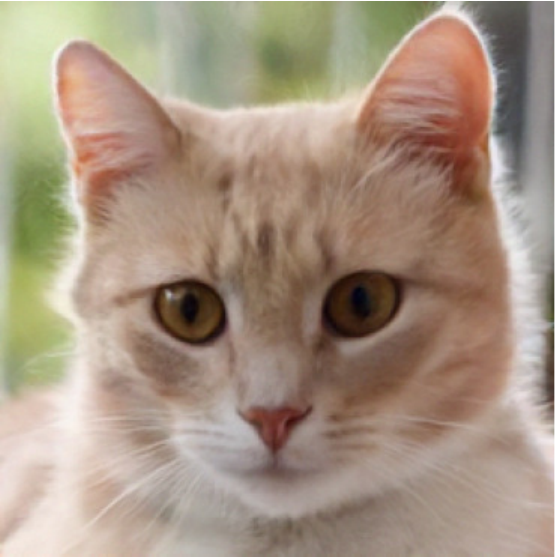}
       \makebox[\linewidth][c]{\footnotesize PSNR:  31.03}
    \end{minipage}
    \begin{minipage}[b]{0.13\textwidth}
    \captionsetup{skip=-0.01cm}
        \includegraphics[width=\linewidth]{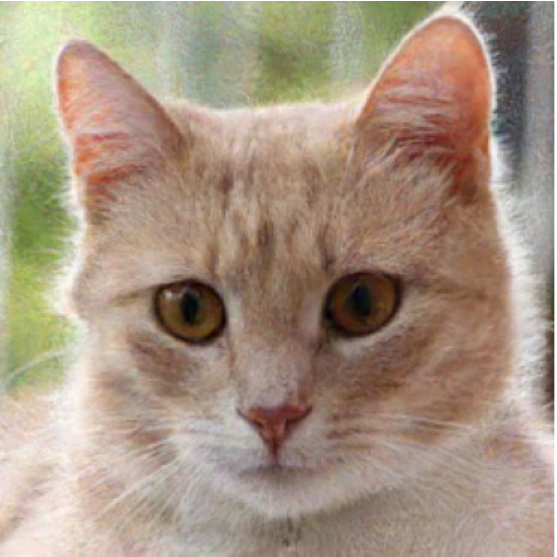}
       \makebox[\linewidth][c]{\footnotesize PSNR:  29.82}
    \end{minipage}
    \begin{minipage}[b]{0.13\textwidth}
    \captionsetup{skip=-0.01cm}
        \includegraphics[width=\linewidth]{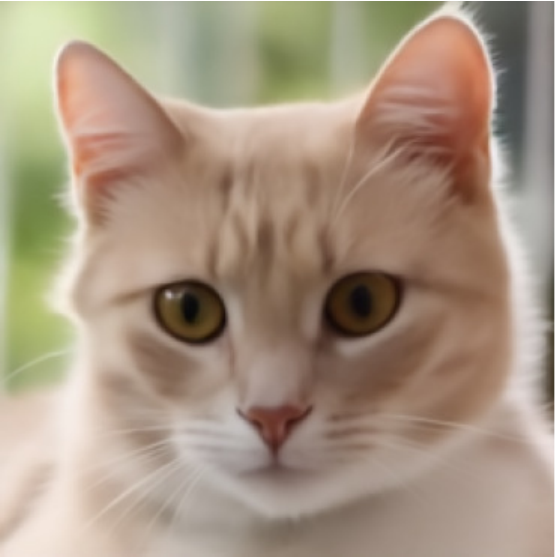}
       \makebox[\linewidth][c]{\footnotesize PSNR:  32.39}
    \end{minipage}
    \begin{minipage}[b]{0.13\textwidth}
    \captionsetup{skip=-0.01cm}
        \includegraphics[width=\linewidth]{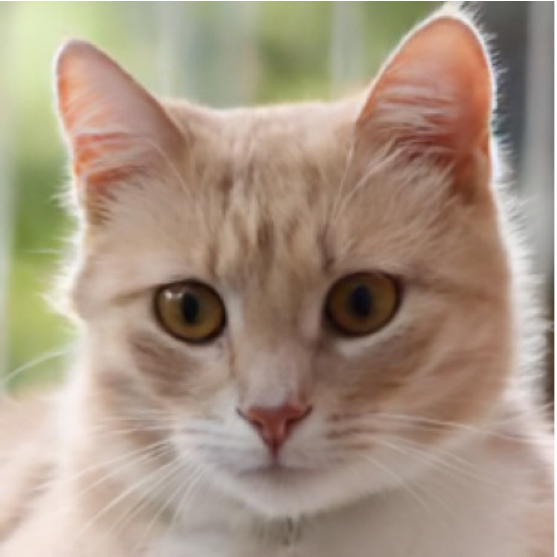}
       \makebox[\linewidth][c]{\footnotesize PSNR: 33.31}
    \end{minipage}

    \begin{minipage}[b]{0.13\textwidth}
    \captionsetup{skip=-0.01cm}
        \includegraphics[width=\linewidth]{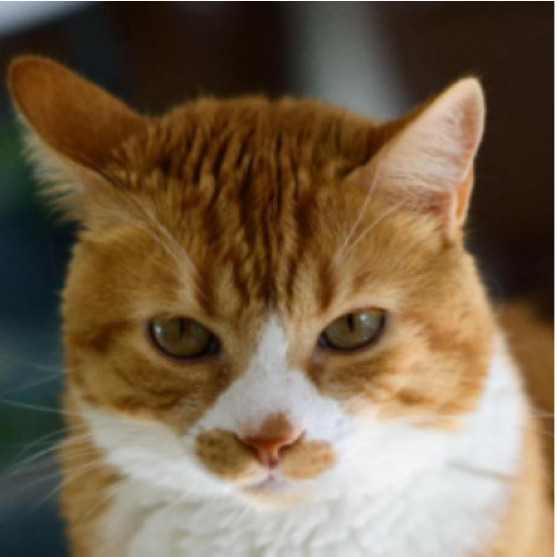}
       \makebox[\linewidth][c]{\footnotesize }
    \end{minipage}
    \begin{minipage}[b]{0.13\textwidth}
    \captionsetup{skip=-0.01cm}
        \includegraphics[width=\linewidth]{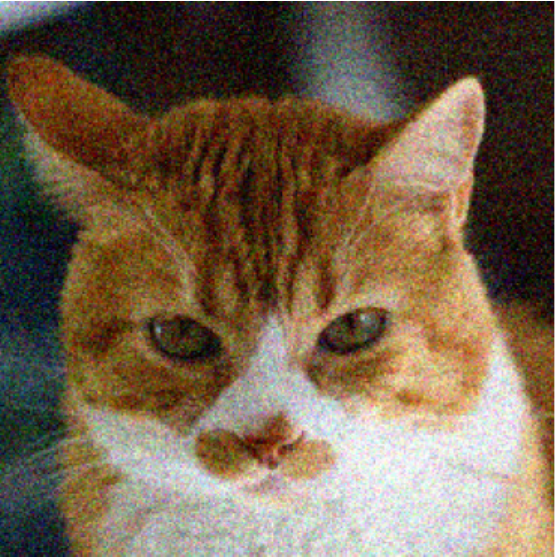}
       \makebox[\linewidth][c]{\footnotesize PSNR:  20.02}
    \end{minipage}
    \begin{minipage}[b]{0.13\textwidth}
    \captionsetup{skip=-0.01cm}
        \includegraphics[width=\linewidth]{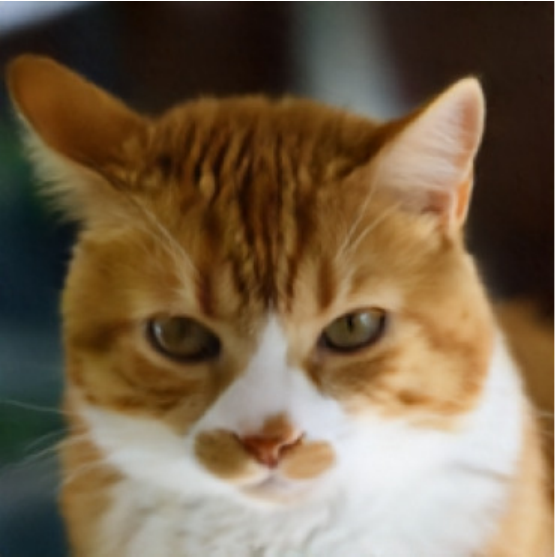}
       \makebox[\linewidth][c]{\footnotesize PSNR:  \textbf{34.11}}
    \end{minipage}
    \begin{minipage}[b]{0.13\textwidth}
    \captionsetup{skip=-0.01cm}
        \includegraphics[width=\linewidth]{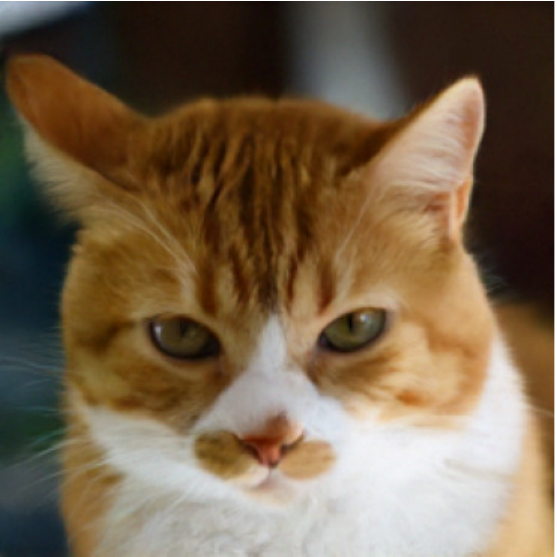}
       \makebox[\linewidth][c]{\footnotesize PSNR:  31.41}
    \end{minipage}
    \begin{minipage}[b]{0.13\textwidth}
    \captionsetup{skip=-0.01cm}
        \includegraphics[width=\linewidth]{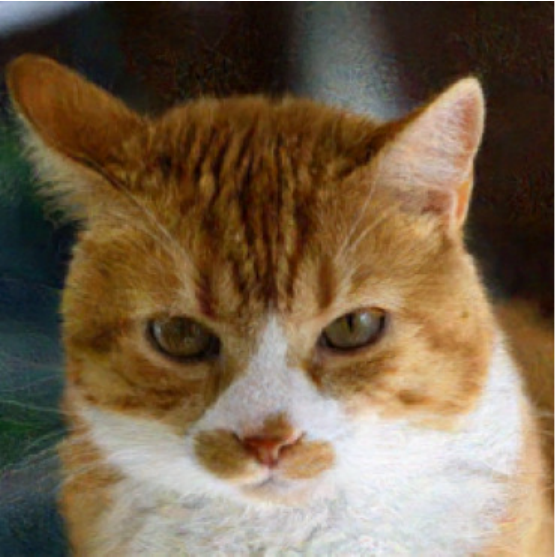}
       \makebox[\linewidth][c]{\footnotesize PSNR: 30.87}
    \end{minipage}
    \begin{minipage}[b]{0.13\textwidth}
    \captionsetup{skip=-0.01cm}
        \includegraphics[width=\linewidth]{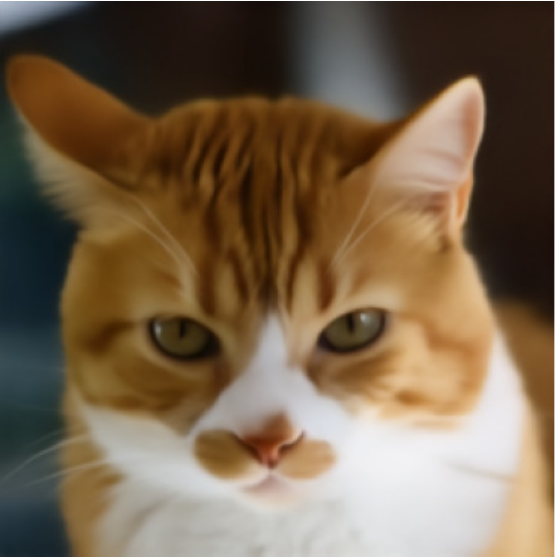}
       \makebox[\linewidth][c]{\footnotesize PSNR:  32.76}
    \end{minipage}
    \begin{minipage}[b]{0.13\textwidth}
    \captionsetup{skip=-0.01cm}
        \includegraphics[width=\linewidth]{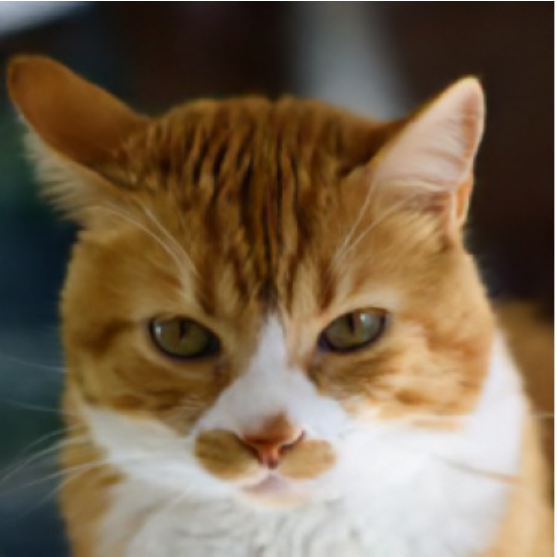}
       \makebox[\linewidth][c]{\footnotesize PSNR:  33.75}
    \end{minipage}

    \begin{minipage}[b]{0.13\textwidth}
    \captionsetup{skip=-0.01cm}
        \includegraphics[width=\linewidth]{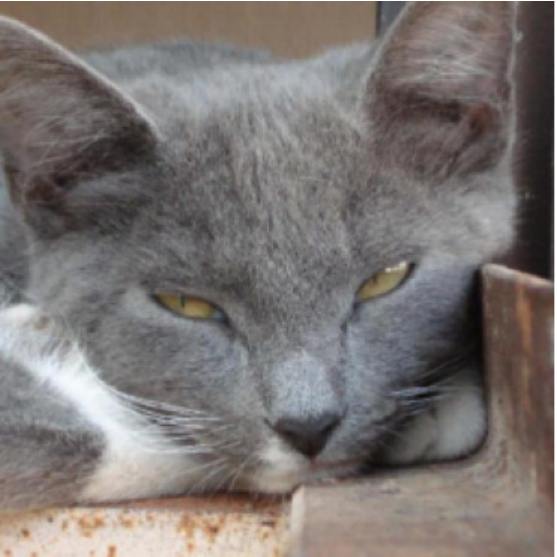}
       \makebox[\linewidth][c]{\footnotesize }
    \end{minipage}
    \begin{minipage}[b]{0.13\textwidth}
    \captionsetup{skip=-0.01cm}
        \includegraphics[width=\linewidth]{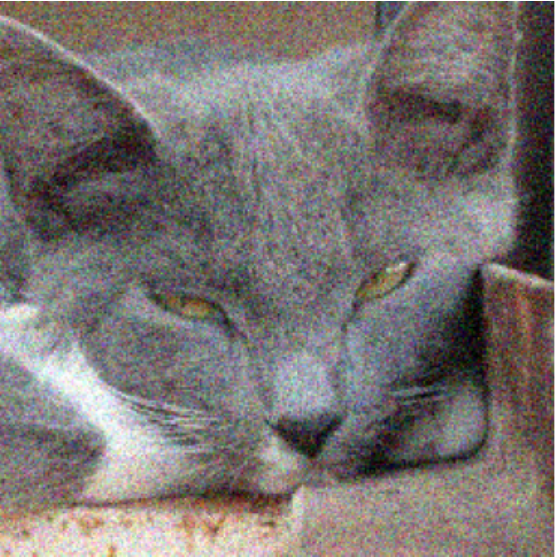}
       \makebox[\linewidth][c]{\footnotesize PSNR: 20.02}
    \end{minipage}
    \begin{minipage}[b]{0.13\textwidth}
    \captionsetup{skip=-0.01cm}
        \includegraphics[width=\linewidth]{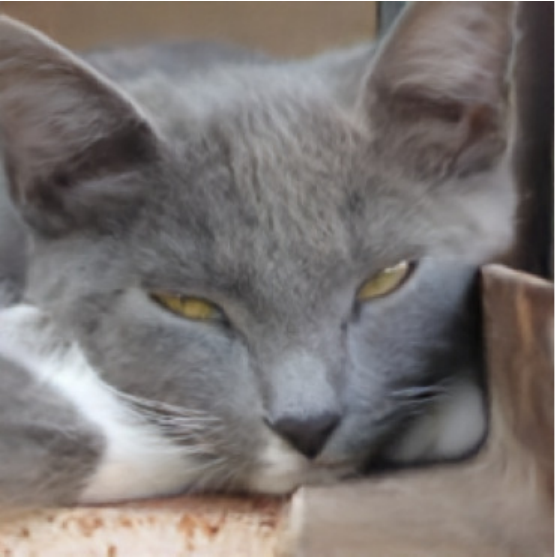}
       \makebox[\linewidth][c]{\footnotesize PSNR: 31.87}
    \end{minipage}
    \begin{minipage}[b]{0.13\textwidth}
    \captionsetup{skip=-0.01cm}
        \includegraphics[width=\linewidth]{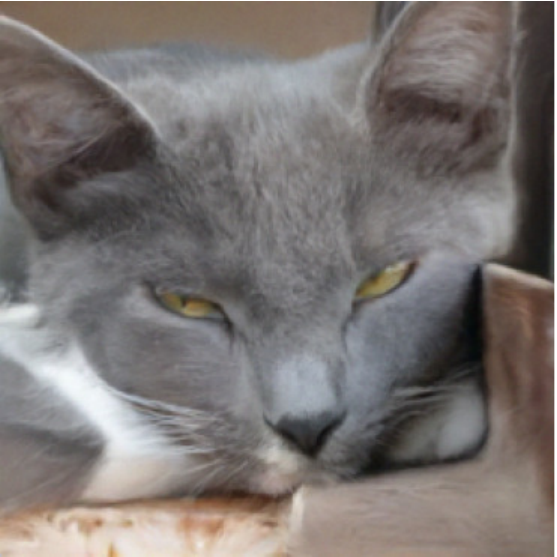}
       \makebox[\linewidth][c]{\footnotesize PSNR: 29.53}
    \end{minipage}
    \begin{minipage}[b]{0.13\textwidth}
    \captionsetup{skip=-0.01cm}
        \includegraphics[width=\linewidth]{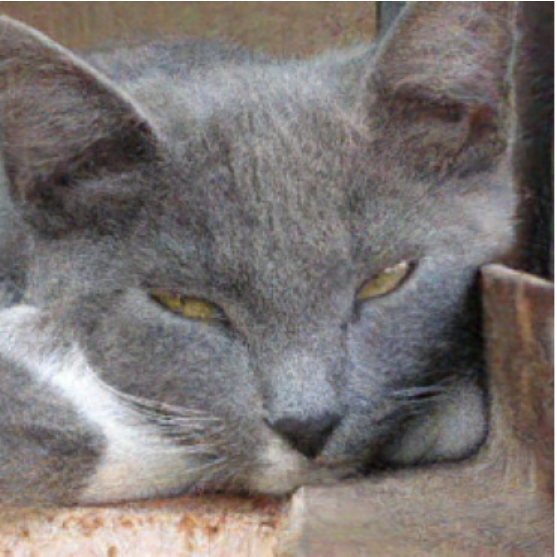}
       \makebox[\linewidth][c]{\footnotesize PSNR: 28.46}
    \end{minipage}
    \begin{minipage}[b]{0.13\textwidth}
    \captionsetup{skip=-0.01cm}
        \includegraphics[width=\linewidth]{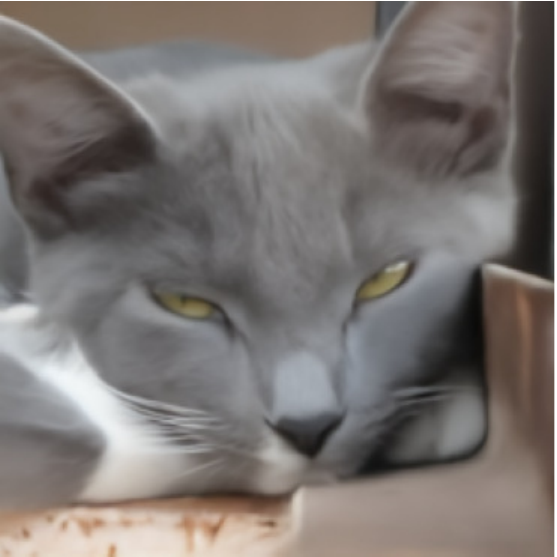}
       \makebox[\linewidth][c]{\footnotesize PSNR: 30.71}
    \end{minipage}
    \begin{minipage}[b]{0.13\textwidth}
        \captionsetup{skip=-0.01cm}
        \includegraphics[width=\linewidth]{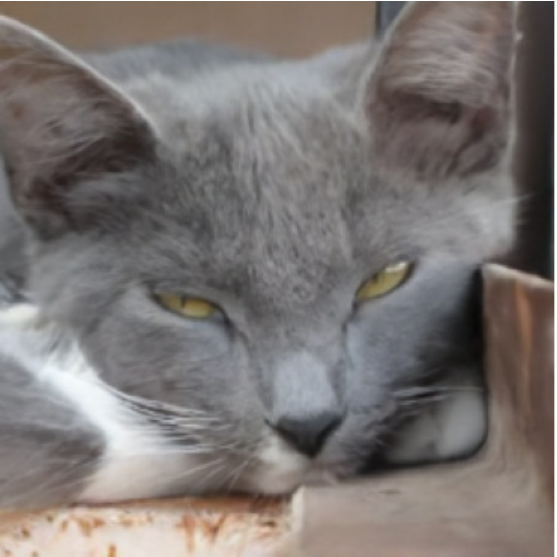}
       \makebox[\linewidth][c]{\footnotesize PSNR: \textbf{31.89}}
    \end{minipage}

    \caption{Comparison of image denoising results on AFHQ-Cat.}
    \label{fig:afhq_denoise}

\end{figure}

\begin{figure}[!ht]
    \centering

    \begin{minipage}[b]{0.13\textwidth}
        \centering
        {\footnotesize \text{Clean}}
    \end{minipage}
    \begin{minipage}[b]{0.13\textwidth}
        \centering
        \small Blurry
    \end{minipage}
    \begin{minipage}[b]{0.13\textwidth}
        \centering
        \small PnP-GS
    \end{minipage}
    \begin{minipage}[b]{0.13\textwidth}
        \centering
        \small OT-ODE
    \end{minipage}
    \begin{minipage}[b]{0.13\textwidth}
        \centering
        \small Flow-Priors
    \end{minipage}
    \begin{minipage}[b]{0.13\textwidth}
        \centering
        \small PnP-Flow
    \end{minipage}
    \begin{minipage}[b]{0.13\textwidth}
        \centering
        \small Ours
    \end{minipage}

    \begin{minipage}[b]{0.13\textwidth}
    \captionsetup{skip=-0.01cm}
        \includegraphics[width=\linewidth]{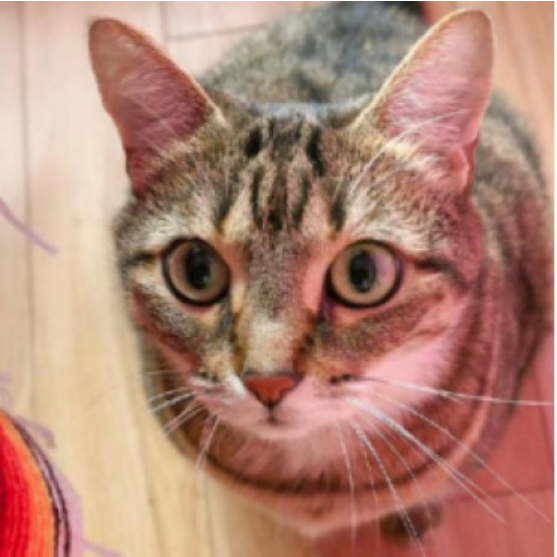}
        \makebox[\linewidth][c]{\footnotesize }
    \end{minipage}
    \begin{minipage}[b]{0.13\textwidth}
    \captionsetup{skip=-0.01cm}
        \includegraphics[width=\linewidth]{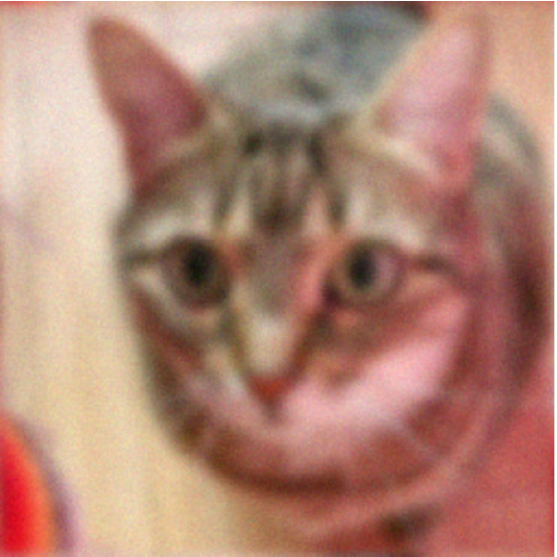}
        \makebox[\linewidth][c]{\footnotesize PSNR: 23.52}
    \end{minipage}
    \begin{minipage}[b]{0.13\textwidth}
    \captionsetup{skip=-0.01cm}
        \includegraphics[width=\linewidth]{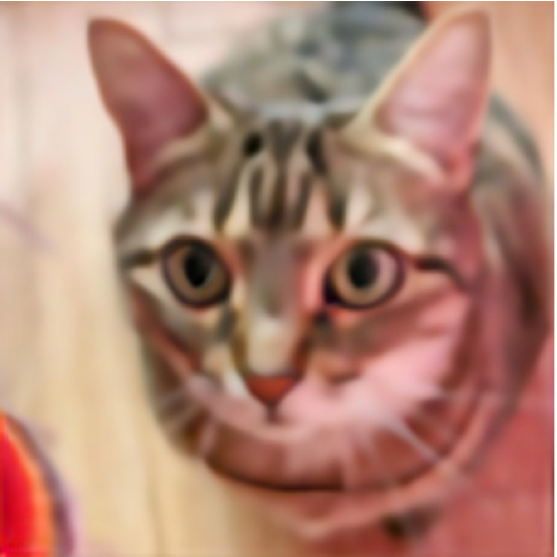}
       \makebox[\linewidth][c]{\footnotesize PSNR: 26.81}
    \end{minipage}
    \begin{minipage}[b]{0.13\textwidth}
    \captionsetup{skip=-0.01cm}
        \includegraphics[width=\linewidth]{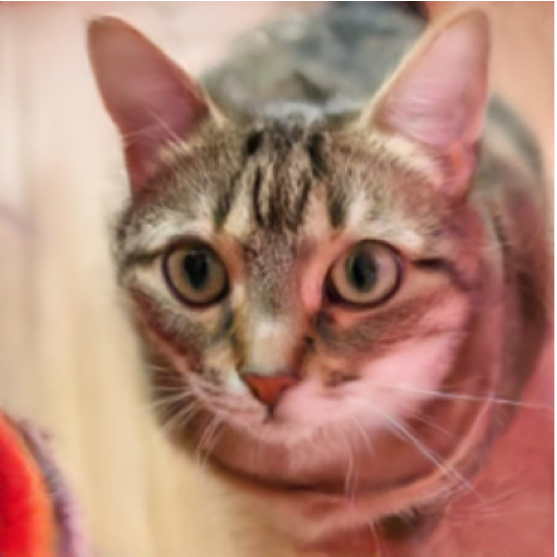}
       \makebox[\linewidth][c]{\footnotesize PSNR:  26.02}
    \end{minipage}
    \begin{minipage}[b]{0.13\textwidth}
    \captionsetup{skip=-0.01cm}
        \includegraphics[width=\linewidth]{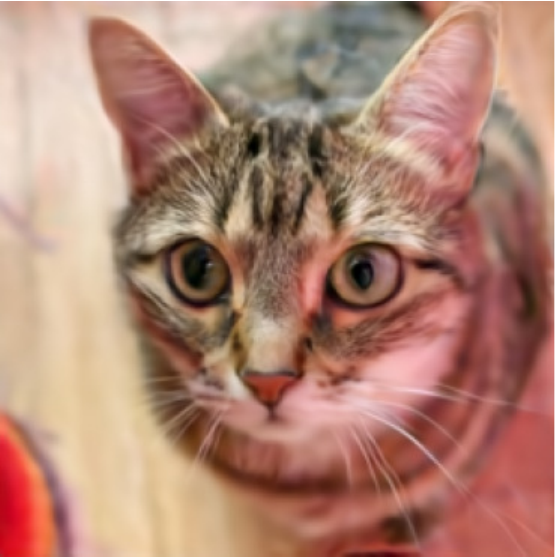}
       \makebox[\linewidth][c]{\footnotesize PSNR:  25.29}
    \end{minipage}
    \begin{minipage}[b]{0.13\textwidth}
    \captionsetup{skip=-0.01cm}
        \includegraphics[width=\linewidth]{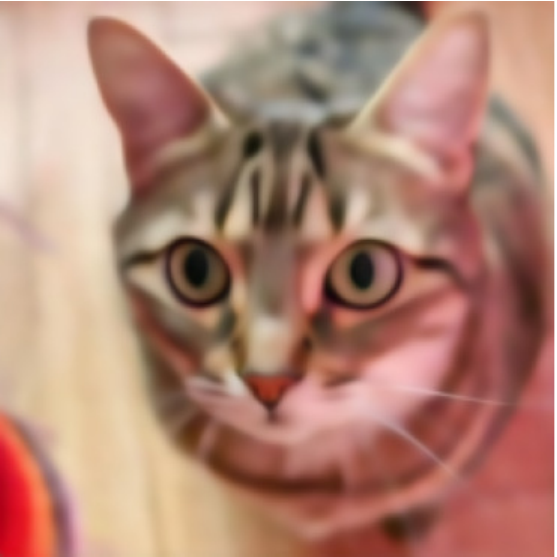}
       \makebox[\linewidth][c]{\footnotesize PSNR:  26.78}
    \end{minipage}
    \begin{minipage}[b]{0.13\textwidth}
    \captionsetup{skip=-0.01cm}
        \includegraphics[width=\linewidth]{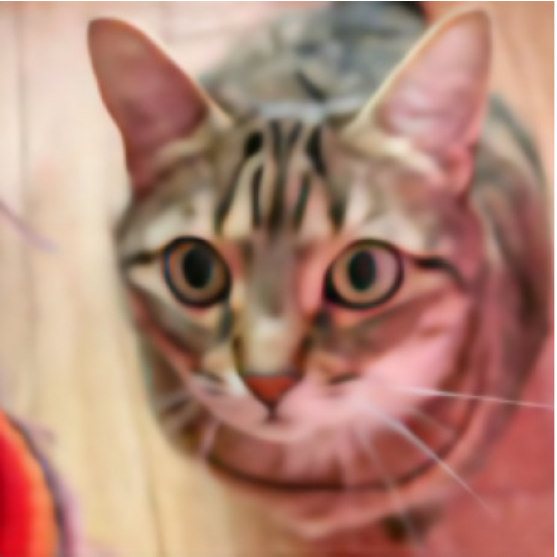}
       \makebox[\linewidth][c]{\footnotesize PSNR:  \textbf{27.40}}
    \end{minipage}

    \begin{minipage}[b]{0.13\textwidth}
    \captionsetup{skip=-0.01cm}
        \includegraphics[width=\linewidth]{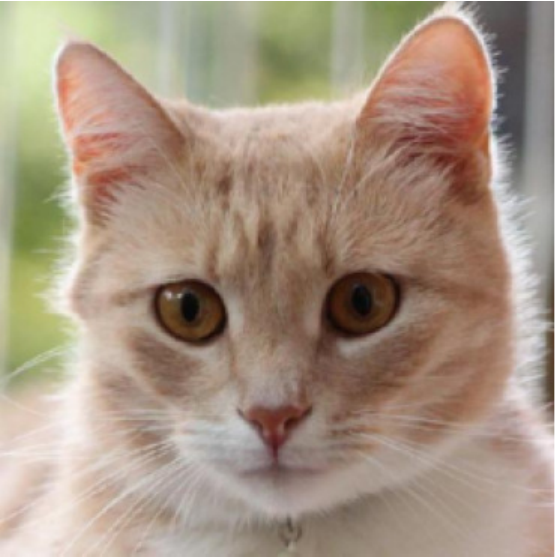}
       \makebox[\linewidth][c]{\footnotesize }
    \end{minipage}
    \begin{minipage}[b]{0.13\textwidth}
    \captionsetup{skip=-0.01cm}
        \includegraphics[width=\linewidth]{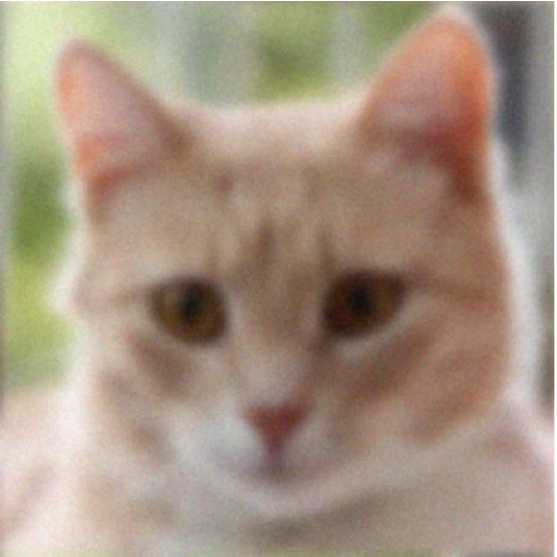}
       \makebox[\linewidth][c]{\footnotesize PSNR: 25.85}
    \end{minipage}
    \begin{minipage}[b]{0.13\textwidth}
    \captionsetup{skip=-0.01cm}
        \includegraphics[width=\linewidth]{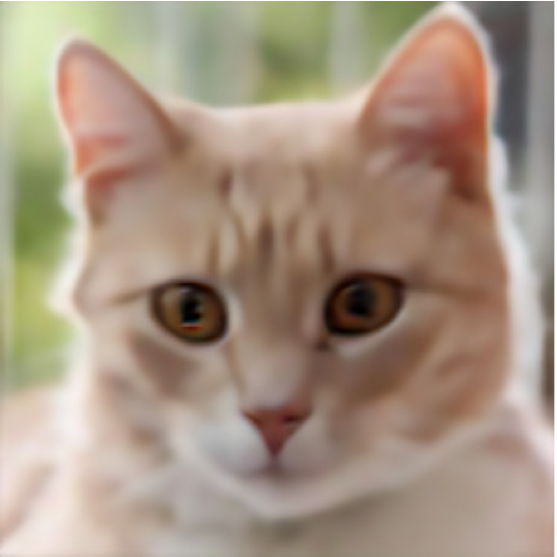}
       \makebox[\linewidth][c]{\footnotesize PSNR:  29.98}
    \end{minipage}
    \begin{minipage}[b]{0.13\textwidth}
    \captionsetup{skip=-0.01cm}
        \includegraphics[width=\linewidth]{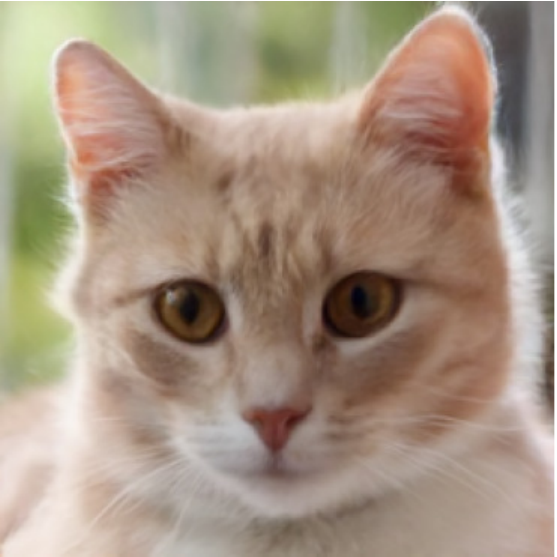}
       \makebox[\linewidth][c]{\footnotesize PSNR:  29.28}
    \end{minipage}
    \begin{minipage}[b]{0.13\textwidth}
    \captionsetup{skip=-0.01cm}
        \includegraphics[width=\linewidth]{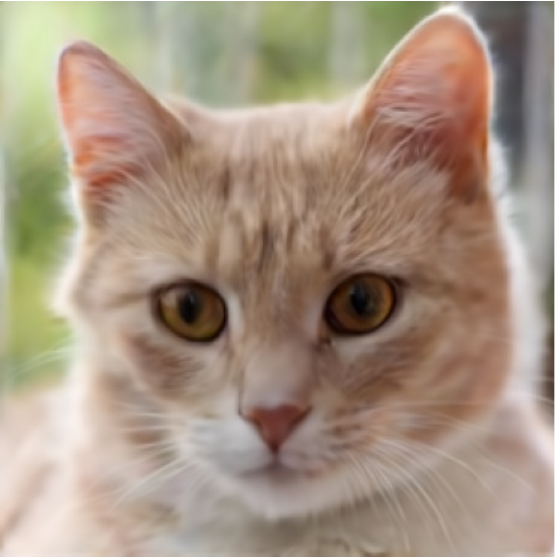}
       \makebox[\linewidth][c]{\footnotesize PSNR: 28.72}
    \end{minipage}
    \begin{minipage}[b]{0.13\textwidth}
    \captionsetup{skip=-0.01cm}
        \includegraphics[width=\linewidth]{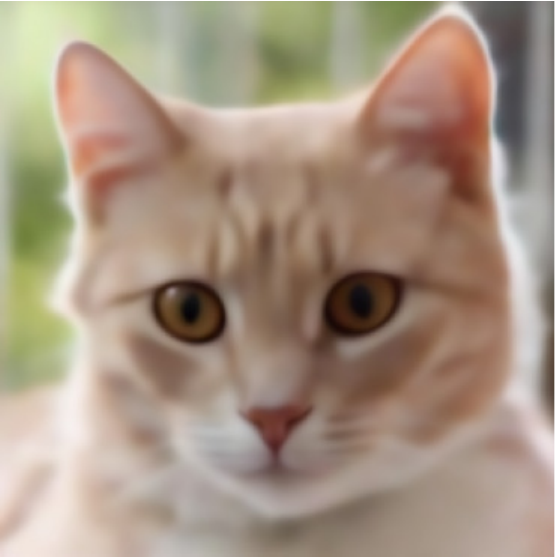}
       \makebox[\linewidth][c]{\footnotesize PSNR: 30.37}
    \end{minipage}
    \begin{minipage}[b]{0.13\textwidth}
    \captionsetup{skip=-0.01cm}
        \includegraphics[width=\linewidth]{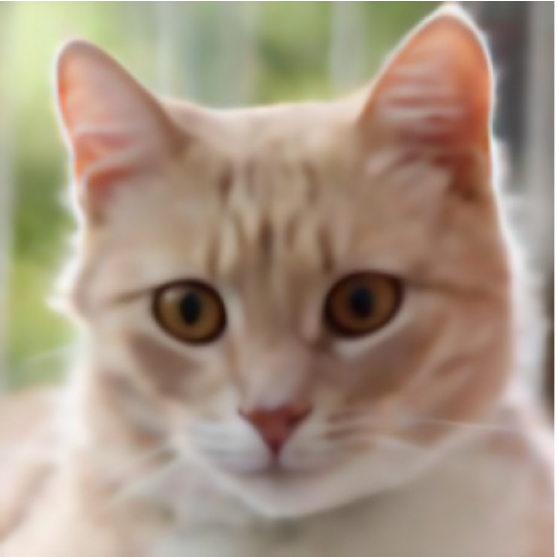}
       \makebox[\linewidth][c]{\footnotesize PSNR: \textbf{31.03}}
    \end{minipage}

    \begin{minipage}[b]{0.13\textwidth}
    \captionsetup{skip=-0.01cm}
        \includegraphics[width=\linewidth]{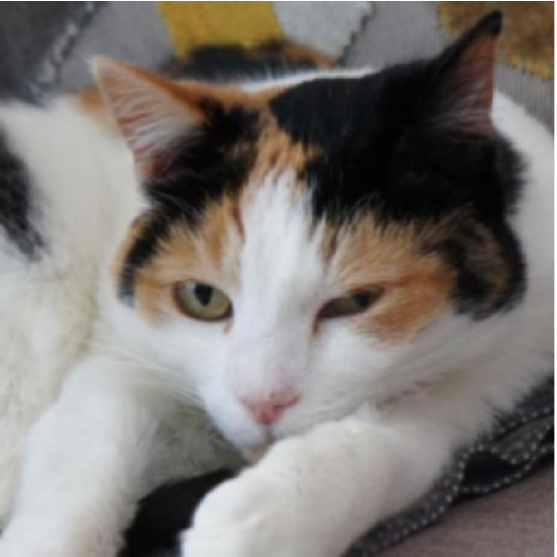}
       \makebox[\linewidth][c]{\footnotesize }
    \end{minipage}
    \begin{minipage}[b]{0.13\textwidth}
    \captionsetup{skip=-0.01cm}
        \includegraphics[width=\linewidth]{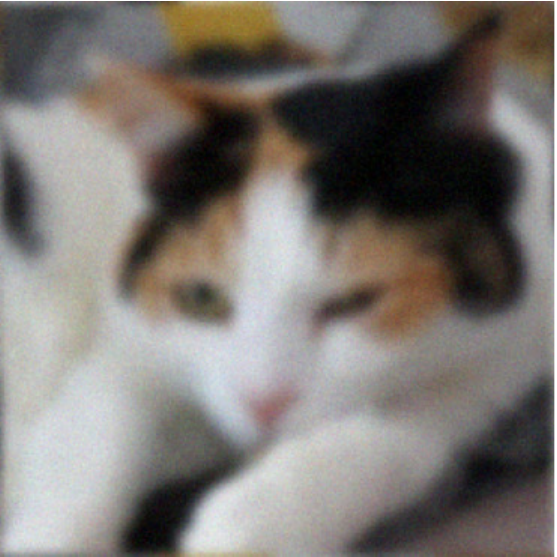}
       \makebox[\linewidth][c]{\footnotesize PSNR: 25.17}
    \end{minipage}
    \begin{minipage}[b]{0.13\textwidth}
    \captionsetup{skip=-0.01cm}
        \includegraphics[width=\linewidth]{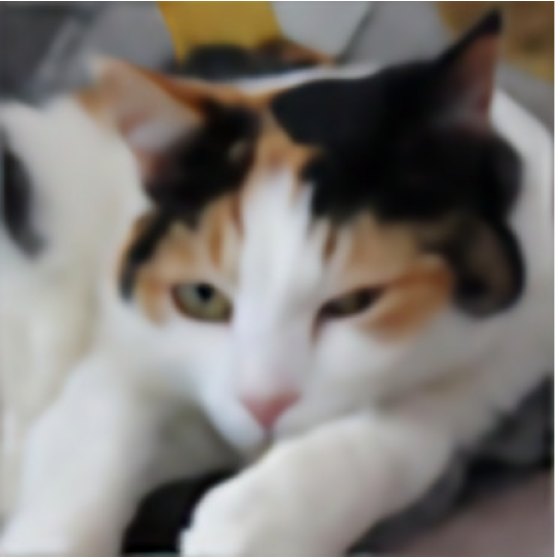}
       \makebox[\linewidth][c]{\footnotesize PSNR: 31.08}
    \end{minipage}
    \begin{minipage}[b]{0.13\textwidth}
    \captionsetup{skip=-0.01cm}
        \includegraphics[width=\linewidth]{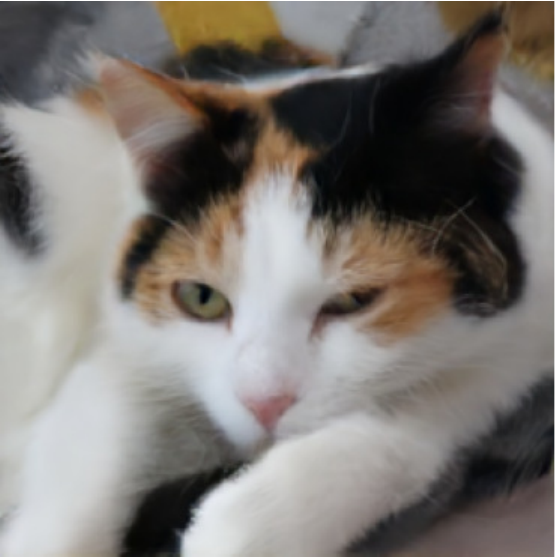}
       \makebox[\linewidth][c]{\footnotesize PSNR: 29.80}
    \end{minipage}
    \begin{minipage}[b]{0.13\textwidth}
    \captionsetup{skip=-0.01cm}
        \includegraphics[width=\linewidth]{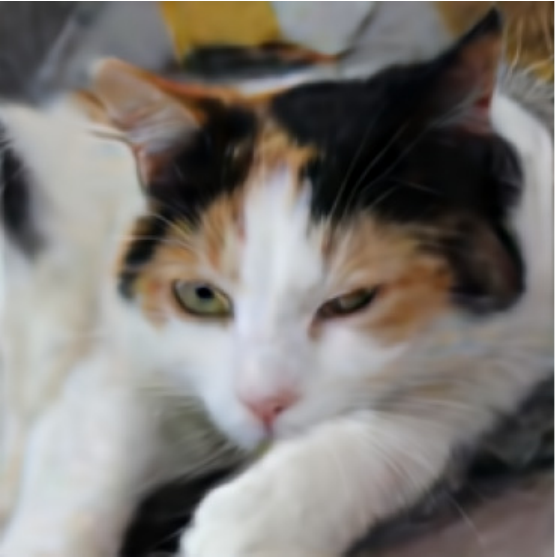}
       \makebox[\linewidth][c]{\footnotesize PSNR: 29.50}
    \end{minipage}
    \begin{minipage}[b]{0.13\textwidth}
    \captionsetup{skip=-0.01cm}
        \includegraphics[width=\linewidth]{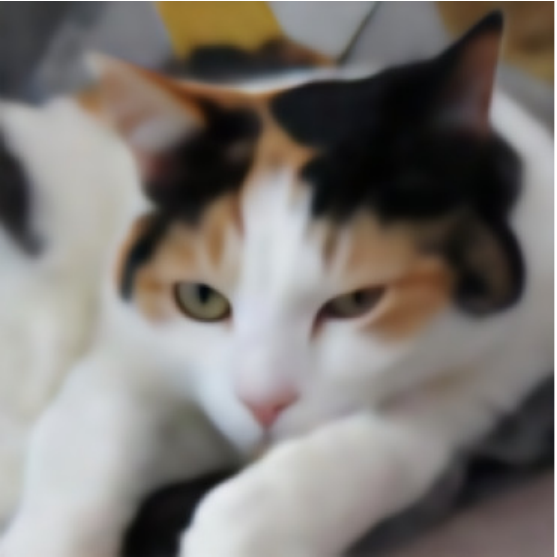}
       \makebox[\linewidth][c]{\footnotesize PSNR: 31.33}
    \end{minipage}
    \begin{minipage}[b]{0.13\textwidth}
    \captionsetup{skip=-0.01cm}
        \includegraphics[width=\linewidth]{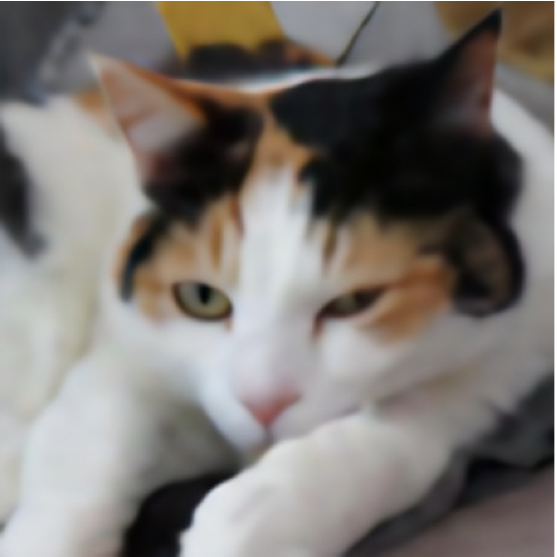}
       \makebox[\linewidth][c]{\footnotesize PSNR:  \textbf{32.11}}
    \end{minipage}

    \caption{Comparison of image deblurring results on AFHQ-Cat.}
    \label{fig:afhq_deblur}

\end{figure}

\vspace{4cm}

\begin{figure}[!ht]
    \centering

    \begin{minipage}[b]{0.13\textwidth}
        \centering
        {\footnotesize \text{Clean}}
    \end{minipage}
    \begin{minipage}[b]{0.13\textwidth}
        \centering
        \small Masked
    \end{minipage}
    \begin{minipage}[b]{0.13\textwidth}
        \centering
        \small PnP-GS
    \end{minipage}
    \begin{minipage}[b]{0.13\textwidth}
        \centering
        \small OT-ODE
    \end{minipage}
    \begin{minipage}[b]{0.13\textwidth}
        \centering
        \small Flow-Priors
    \end{minipage}
    \begin{minipage}[b]{0.13\textwidth}
        \centering
        \small PnP-Flow
    \end{minipage}
    \begin{minipage}[b]{0.13\textwidth}
        \centering
        \small Ours
    \end{minipage}

    \begin{minipage}[b]{0.13\textwidth}
    \captionsetup{skip=-0.01cm}
        \includegraphics[width=\linewidth]{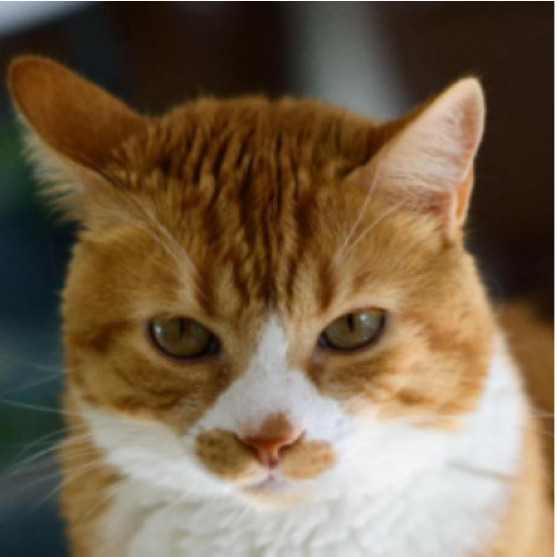}
        \makebox[\linewidth][c]{\footnotesize }
    \end{minipage}
    \begin{minipage}[b]{0.13\textwidth}
    \captionsetup{skip=-0.01cm}
        \includegraphics[width=\linewidth]{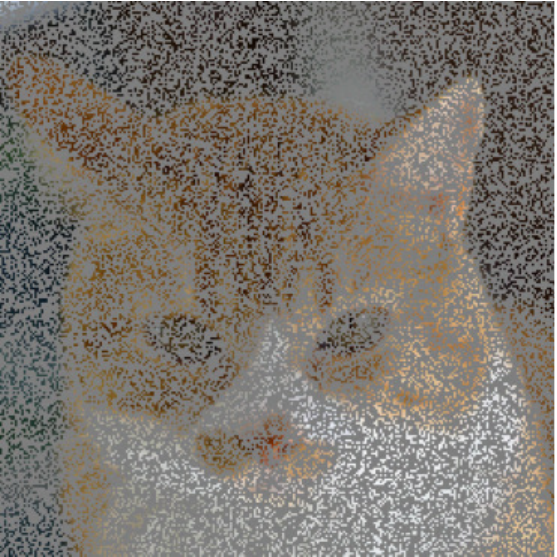}
        \makebox[\linewidth][c]{\footnotesize PSNR: 11.95}
    \end{minipage}
    \begin{minipage}[b]{0.13\textwidth}
    \captionsetup{skip=-0.01cm}
        \includegraphics[width=\linewidth]{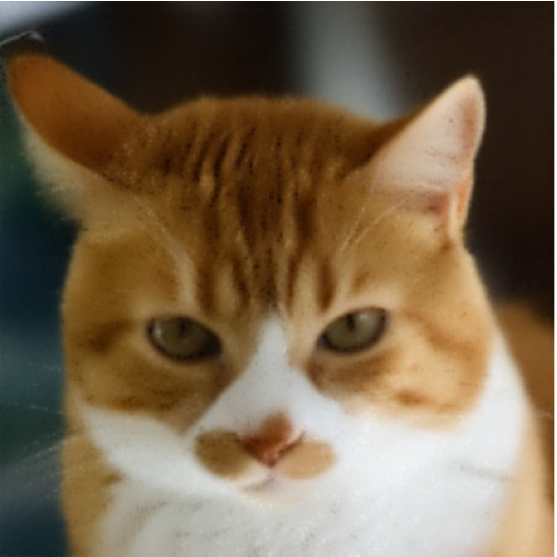}
       \makebox[\linewidth][c]{\footnotesize PSNR: 32.11}
    \end{minipage}
    \begin{minipage}[b]{0.13\textwidth}
    \captionsetup{skip=-0.01cm}
        \includegraphics[width=\linewidth]{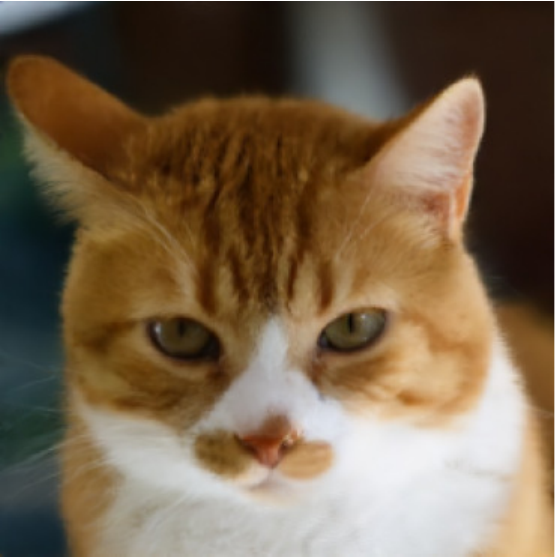}
       \makebox[\linewidth][c]{\footnotesize PSNR: 32.15}
    \end{minipage}
    \begin{minipage}[b]{0.13\textwidth}
    \captionsetup{skip=-0.01cm}
        \includegraphics[width=\linewidth]{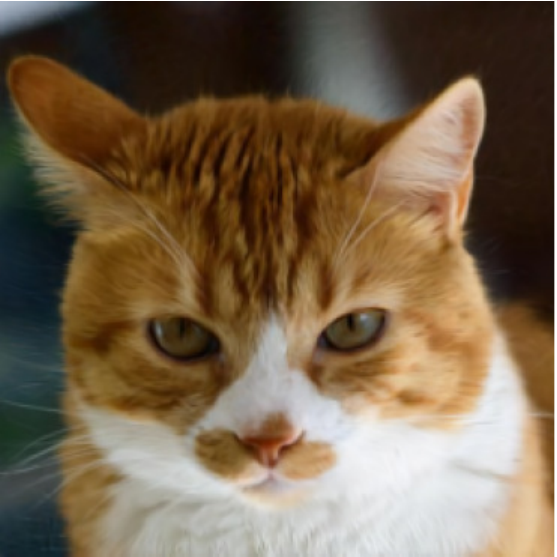}
       \makebox[\linewidth][c]{\footnotesize PSNR:  35.72}
    \end{minipage}
    \begin{minipage}[b]{0.13\textwidth}
    \captionsetup{skip=-0.01cm}
        \includegraphics[width=\linewidth]{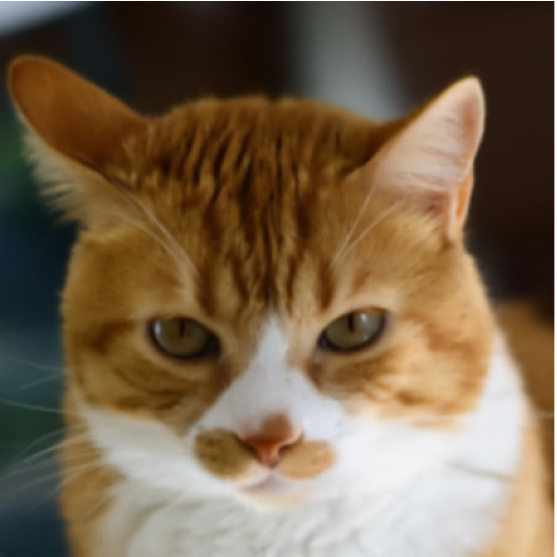}
       \makebox[\linewidth][c]{\footnotesize PSNR:  36.74}
    \end{minipage}
    \begin{minipage}[b]{0.13\textwidth}
    \captionsetup{skip=-0.01cm}
        \includegraphics[width=\linewidth]{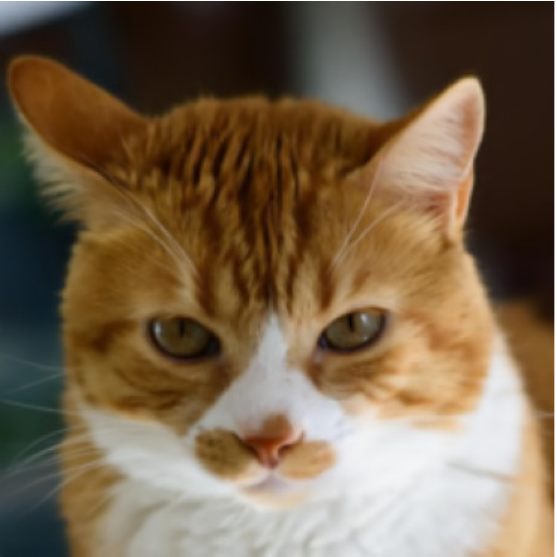}
       \makebox[\linewidth][c]{\footnotesize PSNR:  \textbf{37.81}}
    \end{minipage}

    \begin{minipage}[b]{0.13\textwidth}
    \captionsetup{skip=-0.01cm}
        \includegraphics[width=\linewidth]{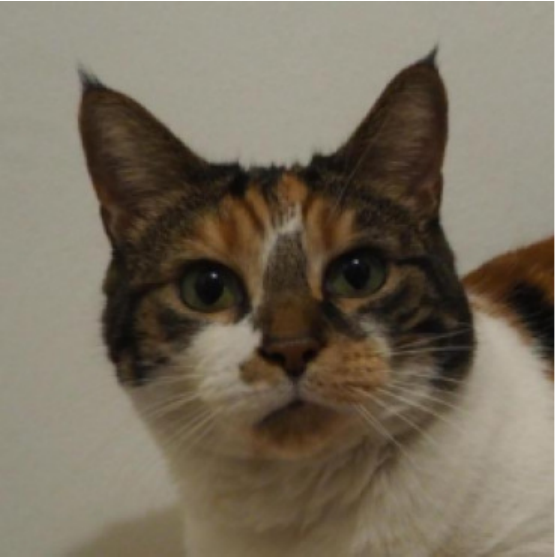}
       \makebox[\linewidth][c]{\footnotesize }
    \end{minipage}
    \begin{minipage}[b]{0.13\textwidth}
    \captionsetup{skip=-0.01cm}
        \includegraphics[width=\linewidth]{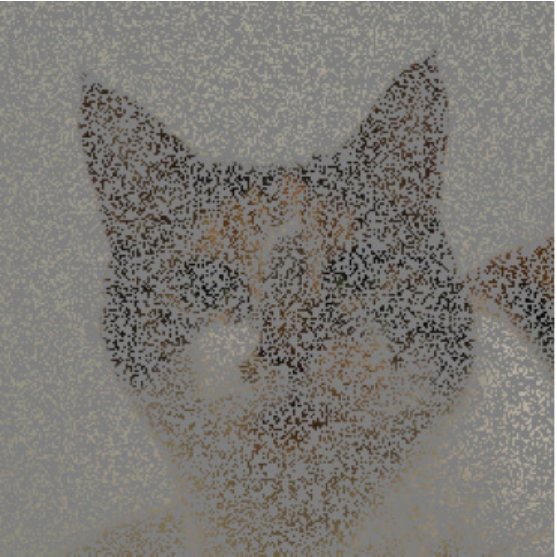}
       \makebox[\linewidth][c]{\footnotesize PSNR: 15.02}
    \end{minipage}
    \begin{minipage}[b]{0.13\textwidth}
    \captionsetup{skip=-0.01cm}
        \includegraphics[width=\linewidth]{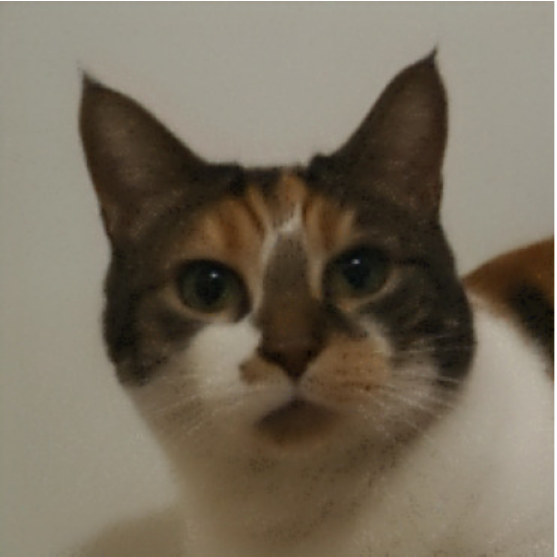}
       \makebox[\linewidth][c]{\footnotesize PSNR:  34.28}
    \end{minipage}
    \begin{minipage}[b]{0.13\textwidth}
    \captionsetup{skip=-0.01cm}
        \includegraphics[width=\linewidth]{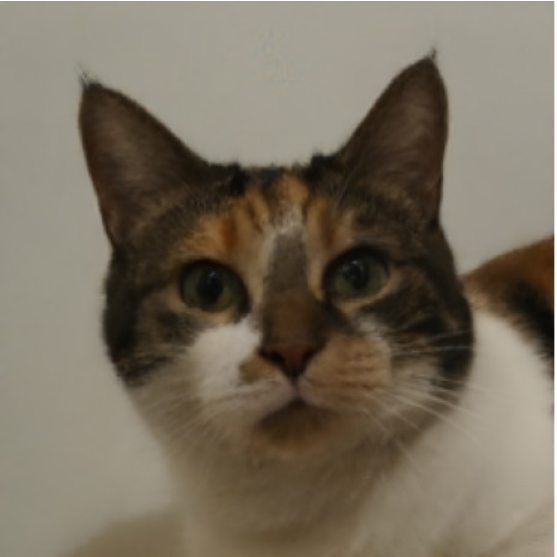}
       \makebox[\linewidth][c]{\footnotesize PSNR:  33.93}
    \end{minipage}
    \begin{minipage}[b]{0.13\textwidth}
    \captionsetup{skip=-0.01cm}
        \includegraphics[width=\linewidth]{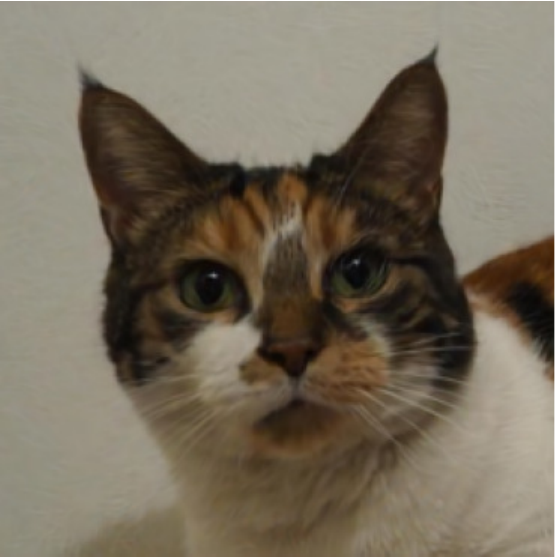}
       \makebox[\linewidth][c]{\footnotesize PSNR:  37.28}
    \end{minipage}
    \begin{minipage}[b]{0.13\textwidth}
    \captionsetup{skip=-0.01cm}
        \includegraphics[width=\linewidth]{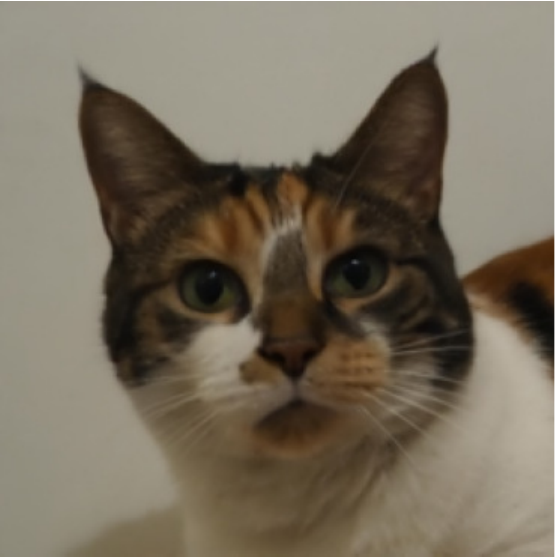}
       \makebox[\linewidth][c]{\footnotesize PSNR: 38.13}
    \end{minipage}
    \begin{minipage}[b]{0.13\textwidth}
    \captionsetup{skip=-0.01cm}
        \includegraphics[width=\linewidth]{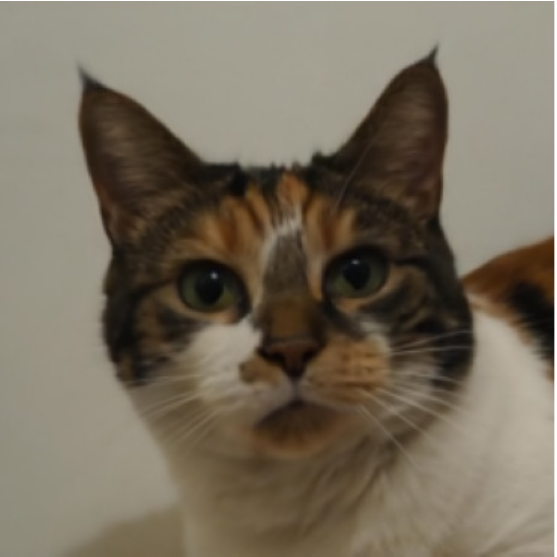}
       \makebox[\linewidth][c]{\footnotesize PSNR: \textbf{38.95}}
    \end{minipage}

    \begin{minipage}[b]{0.13\textwidth}
    \captionsetup{skip=-0.01cm}
        \includegraphics[width=\linewidth]{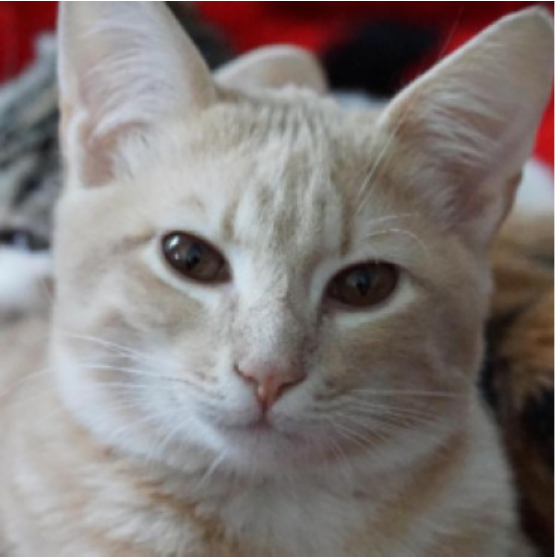}
       \makebox[\linewidth][c]{\footnotesize }
    \end{minipage}
    \begin{minipage}[b]{0.13\textwidth}
    \captionsetup{skip=-0.01cm}
        \includegraphics[width=\linewidth]{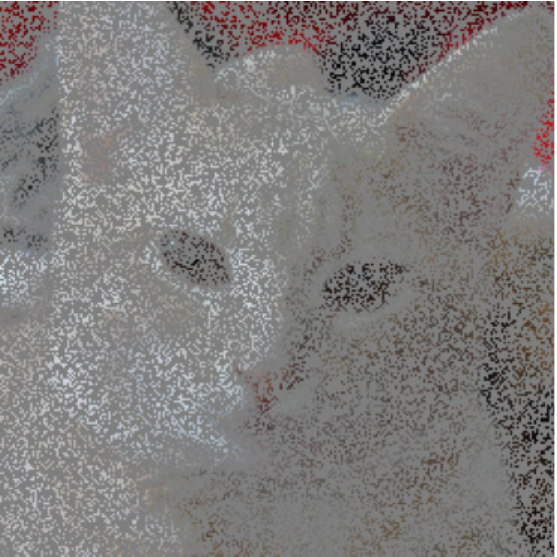}
       \makebox[\linewidth][c]{\footnotesize PSNR:  14.77}
    \end{minipage}
    \begin{minipage}[b]{0.13\textwidth}
    \captionsetup{skip=-0.01cm}
        \includegraphics[width=\linewidth]{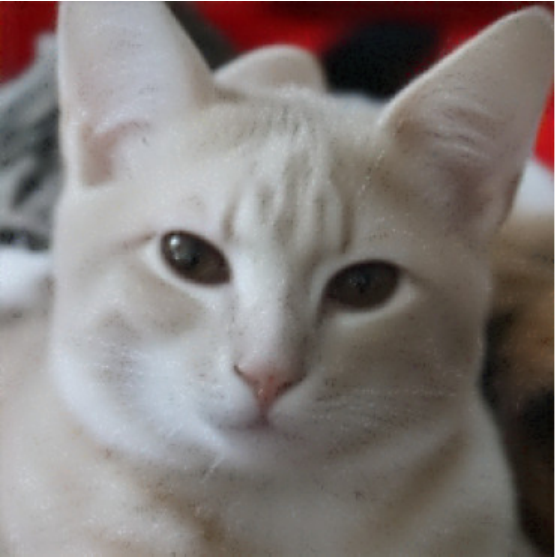}
       \makebox[\linewidth][c]{\footnotesize PSNR:  32.74}
    \end{minipage}
    \begin{minipage}[b]{0.13\textwidth}
    \captionsetup{skip=-0.01cm}
        \includegraphics[width=\linewidth]{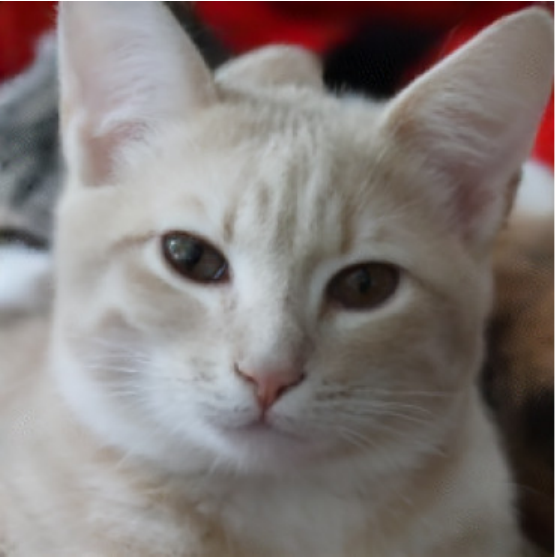}
       \makebox[\linewidth][c]{\footnotesize PSNR:  32.25}
    \end{minipage}
    \begin{minipage}[b]{0.13\textwidth}
    \captionsetup{skip=-0.01cm}
        \includegraphics[width=\linewidth]{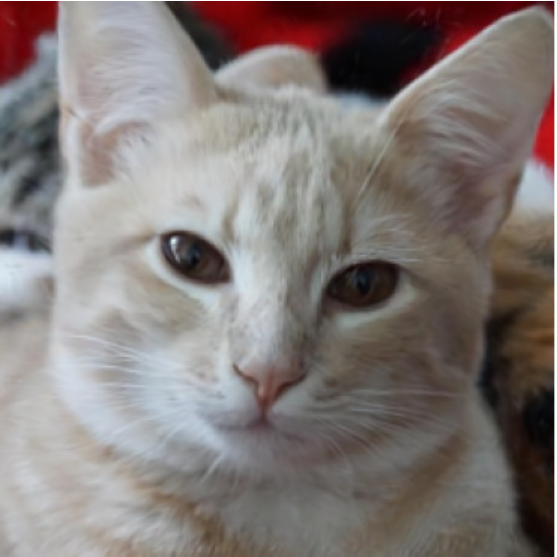}
       \makebox[\linewidth][c]{\footnotesize PSNR:  35.50}
    \end{minipage}
    \begin{minipage}[b]{0.13\textwidth}
    \captionsetup{skip=-0.01cm}
        \includegraphics[width=\linewidth]{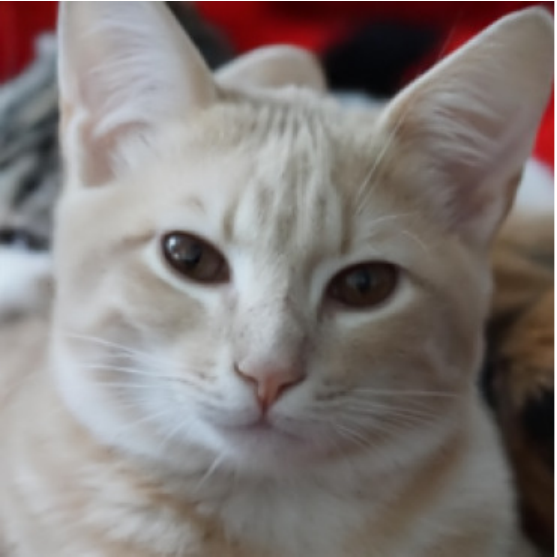}
       \makebox[\linewidth][c]{\footnotesize PSNR:  37.38}
    \end{minipage}
    \begin{minipage}[b]{0.13\textwidth}
    \captionsetup{skip=-0.01cm}
        \includegraphics[width=\linewidth]{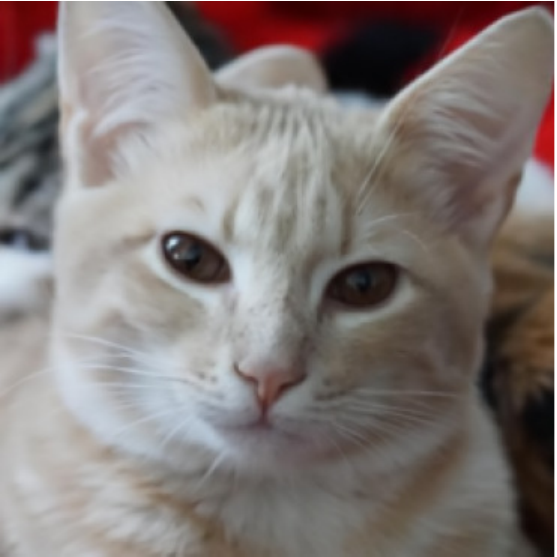}
       \makebox[\linewidth][c]{\footnotesize PSNR:  \textbf{38.57}}
    \end{minipage}

    \caption{Comparison of random inpainting results on AFHQ-Cat.}
    \label{fig:afhq_inpaint}

\end{figure}

     \vspace{5cm}

\begin{figure}[!ht]
    \centering

    \begin{minipage}[b]{0.13\textwidth}
        \centering
        {\footnotesize \text{Clean}}
    \end{minipage}
    \begin{minipage}[b]{0.13\textwidth}
        \centering
        \small Low-resolution
    \end{minipage}
    \begin{minipage}[b]{0.13\textwidth}
        \centering
        \small PnP-GS
    \end{minipage}
    \begin{minipage}[b]{0.13\textwidth}
        \centering
        \small OT-ODE
    \end{minipage}
    \begin{minipage}[b]{0.13\textwidth}
        \centering
        \small Flow-Priors
    \end{minipage}
    \begin{minipage}[b]{0.13\textwidth}
        \centering
        \small PnP-Flow
    \end{minipage}
    \begin{minipage}[b]{0.13\textwidth}
        \centering
        \small Ours
    \end{minipage}

    \begin{minipage}[b]{0.13\textwidth}
    \captionsetup{skip=-0.01cm}
        \includegraphics[width=\linewidth]{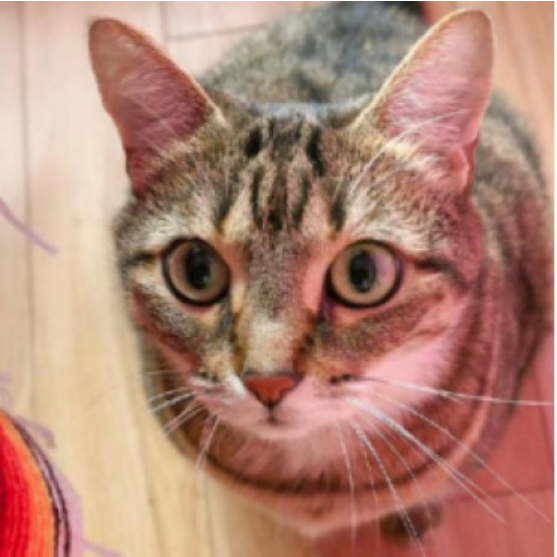}
        \makebox[\linewidth][c]{\footnotesize }
    \end{minipage}
    \begin{minipage}[b]{0.13\textwidth}
    \captionsetup{skip=-0.01cm}
        \includegraphics[width=\linewidth]{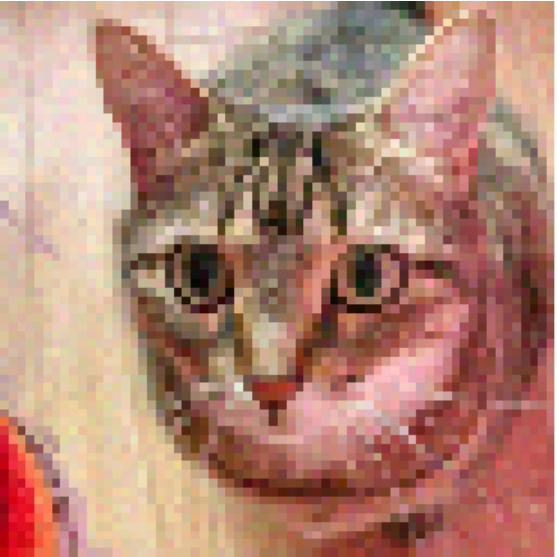}
        \makebox[\linewidth][c]{\footnotesize PSNR: 11.83}
    \end{minipage}
    \begin{minipage}[b]{0.13\textwidth}
    \captionsetup{skip=-0.01cm}
        \includegraphics[width=\linewidth]{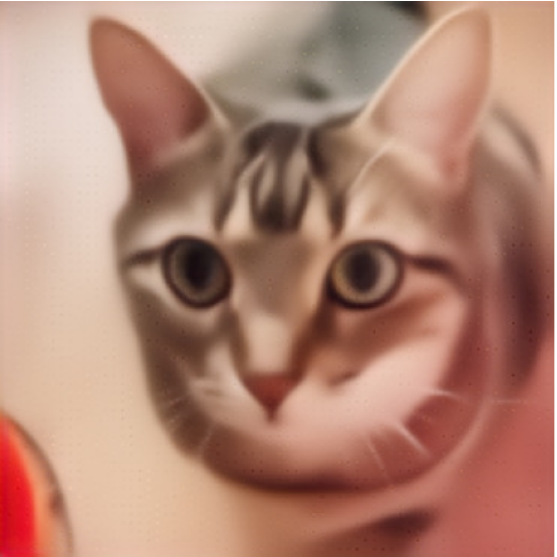}
       \makebox[\linewidth][c]{\footnotesize PSNR: 24.55}
    \end{minipage}
    \begin{minipage}[b]{0.13\textwidth}
    \captionsetup{skip=-0.01cm}
        \includegraphics[width=\linewidth]{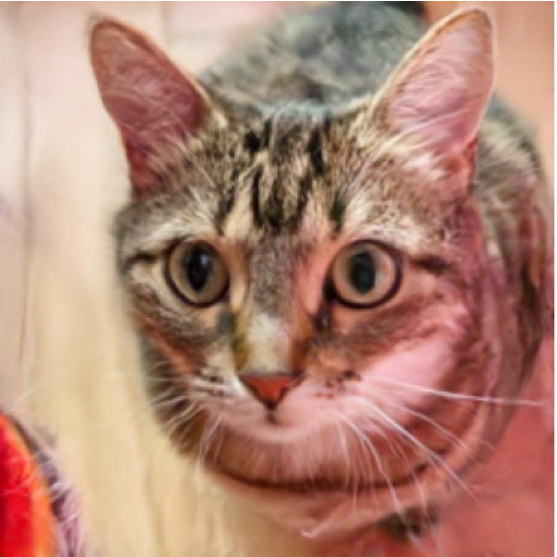}
       \makebox[\linewidth][c]{\footnotesize PSNR:  24.72}
    \end{minipage}
    \begin{minipage}[b]{0.13\textwidth}
    \captionsetup{skip=-0.01cm}
        \includegraphics[width=\linewidth]{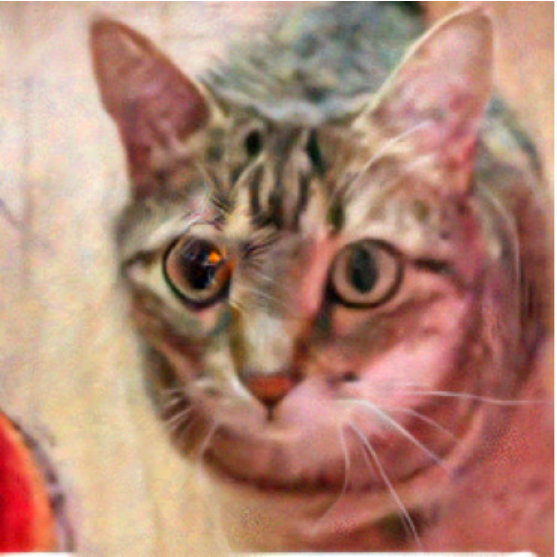}
       \makebox[\linewidth][c]{\footnotesize PSNR:  23.15}
    \end{minipage}
    \begin{minipage}[b]{0.13\textwidth}
    \captionsetup{skip=-0.01cm}
        \includegraphics[width=\linewidth]{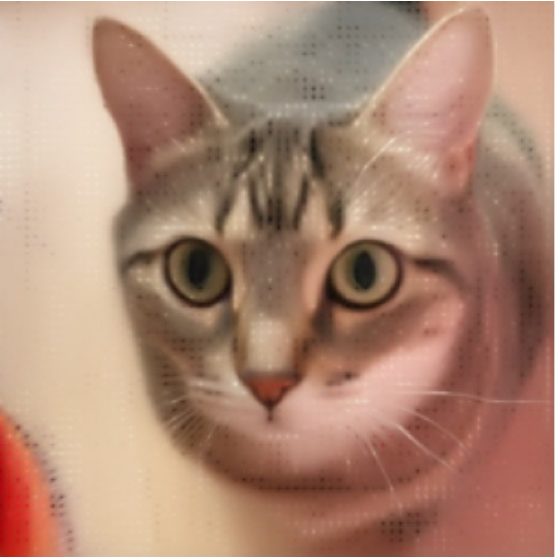}
       \makebox[\linewidth][c]{\footnotesize PSNR: 24.53}
    \end{minipage}
    \begin{minipage}[b]{0.13\textwidth}
    \captionsetup{skip=-0.01cm}
        \includegraphics[width=\linewidth]{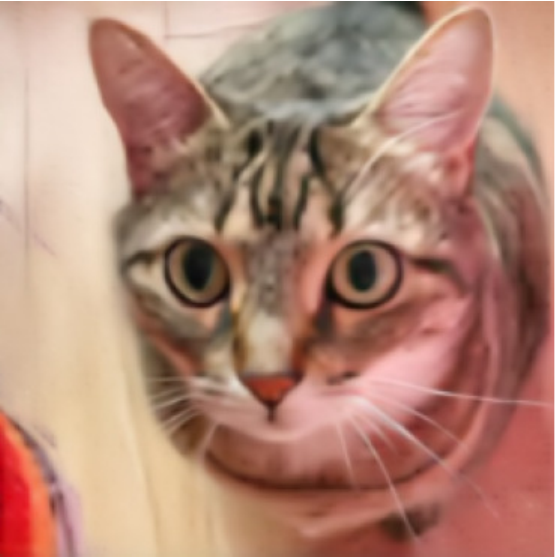}
       \makebox[\linewidth][c]{\footnotesize PSNR:  \textbf{26.74}}
    \end{minipage}

    \begin{minipage}[b]{0.13\textwidth}
    \captionsetup{skip=-0.01cm}
        \includegraphics[width=\linewidth]{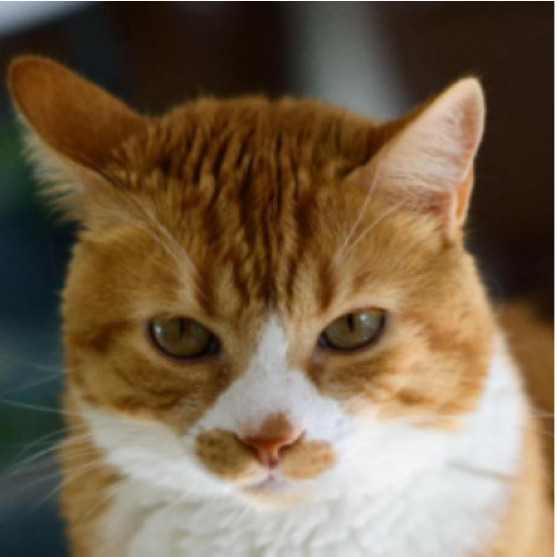}
       \makebox[\linewidth][c]{\footnotesize }
    \end{minipage}
    \begin{minipage}[b]{0.13\textwidth}
    \captionsetup{skip=-0.01cm}
        \includegraphics[width=\linewidth]{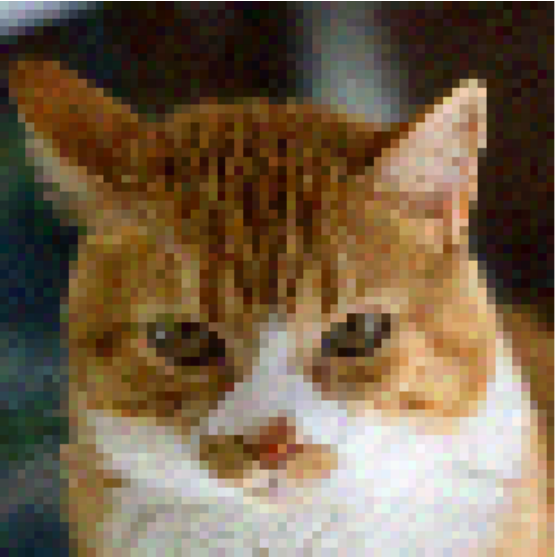}
       \makebox[\linewidth][c]{\footnotesize PSNR: 9.40}
    \end{minipage}
    \begin{minipage}[b]{0.13\textwidth}
    \captionsetup{skip=-0.01cm}
        \includegraphics[width=\linewidth]{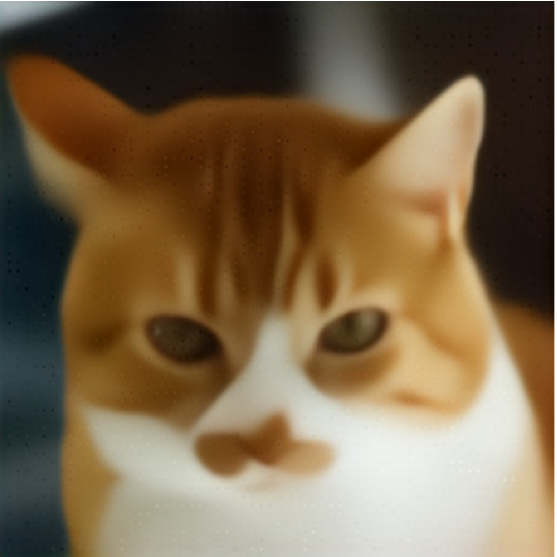}
       \makebox[\linewidth][c]{\footnotesize PSNR:  27.56}
    \end{minipage}
    \begin{minipage}[b]{0.13\textwidth}
    \captionsetup{skip=-0.01cm}
        \includegraphics[width=\linewidth]{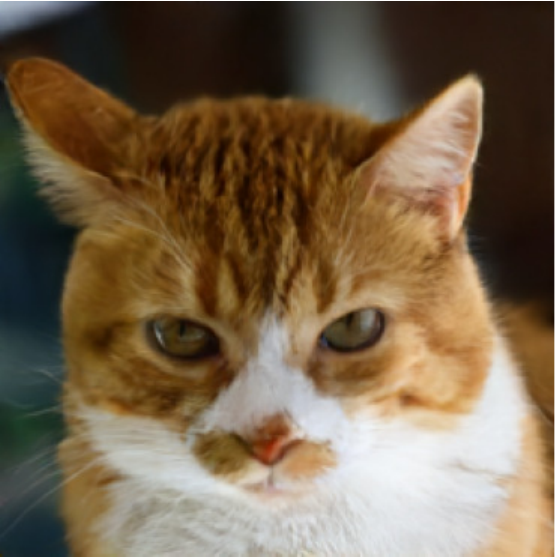}
       \makebox[\linewidth][c]{\footnotesize PSNR:  28.60}
    \end{minipage}
    \begin{minipage}[b]{0.13\textwidth}
    \captionsetup{skip=-0.01cm}
        \includegraphics[width=\linewidth]{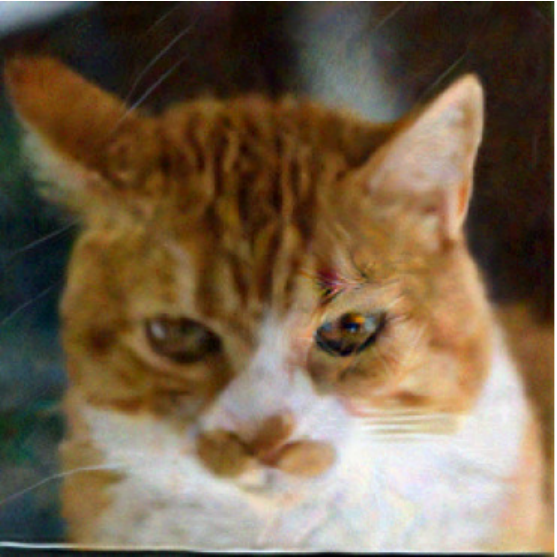}
       \makebox[\linewidth][c]{\footnotesize PSNR: 23.16}
    \end{minipage}
    \begin{minipage}[b]{0.13\textwidth}
    \captionsetup{skip=-0.01cm}
        \includegraphics[width=\linewidth]{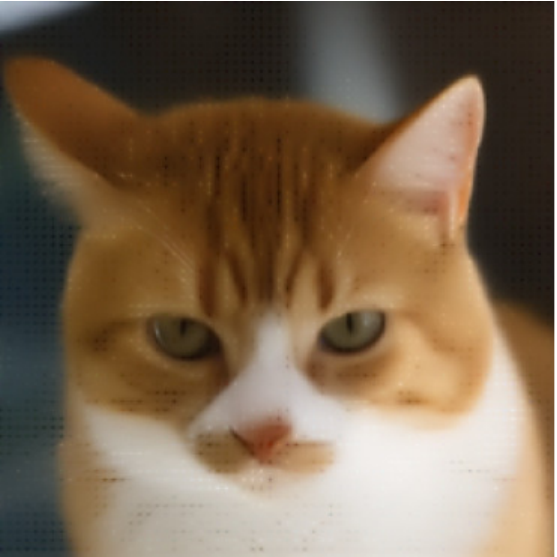}
       \makebox[\linewidth][c]{\footnotesize PSNR: 27.25}
    \end{minipage}
    \begin{minipage}[b]{0.13\textwidth}
    \captionsetup{skip=-0.01cm}
        \includegraphics[width=\linewidth]{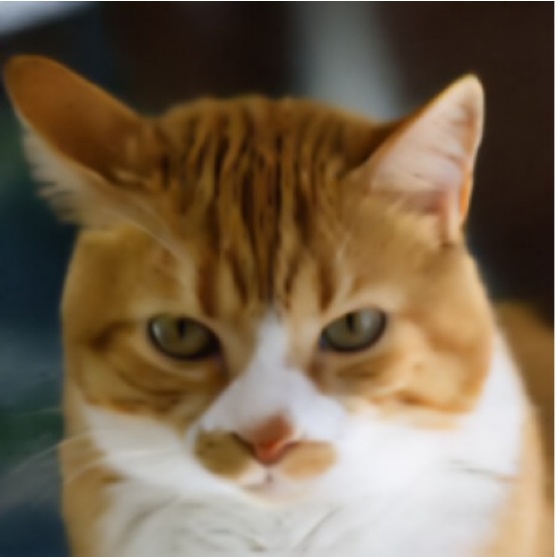}
       \makebox[\linewidth][c]{\footnotesize PSNR: \textbf{30.61}}
    \end{minipage}

    \begin{minipage}[b]{0.13\textwidth}
    \captionsetup{skip=-0.01cm}
        \includegraphics[width=\linewidth]{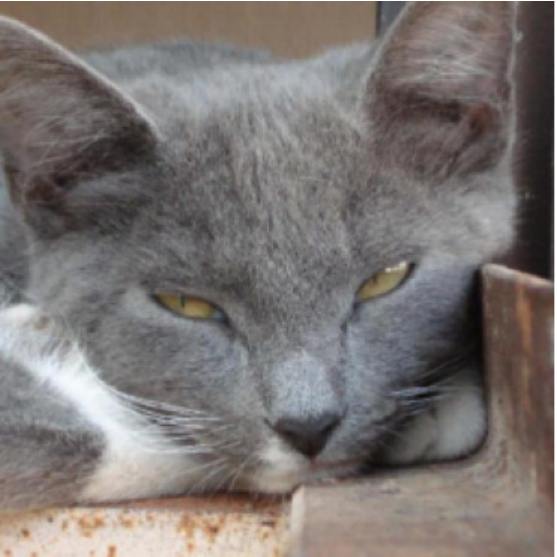}
       \makebox[\linewidth][c]{\footnotesize }
    \end{minipage}
    \begin{minipage}[b]{0.13\textwidth}
    \captionsetup{skip=-0.01cm}
        \includegraphics[width=\linewidth]{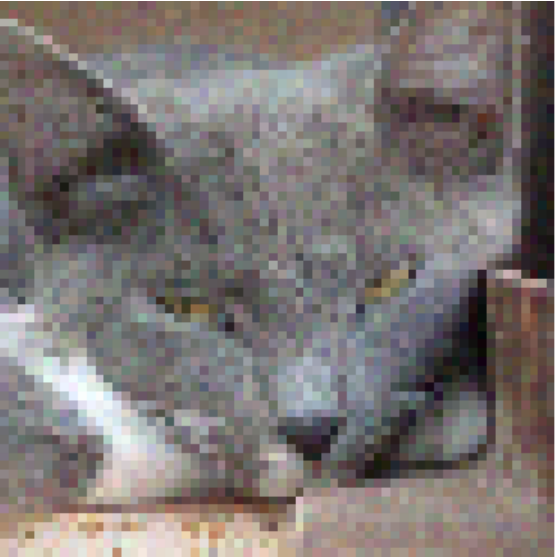}
       \makebox[\linewidth][c]{\footnotesize PSNR: 13.42}
    \end{minipage}
    \begin{minipage}[b]{0.13\textwidth}
    \captionsetup{skip=-0.01cm}
        \includegraphics[width=\linewidth]{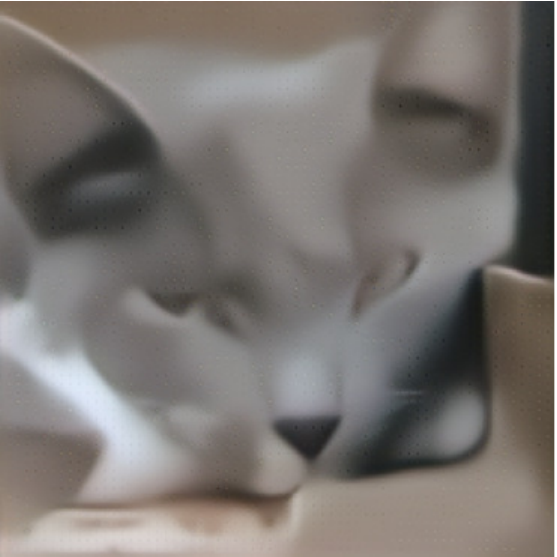}
       \makebox[\linewidth][c]{\footnotesize PSNR: 25.66}
    \end{minipage}
    \begin{minipage}[b]{0.13\textwidth}
    \captionsetup{skip=-0.01cm}
        \includegraphics[width=\linewidth]{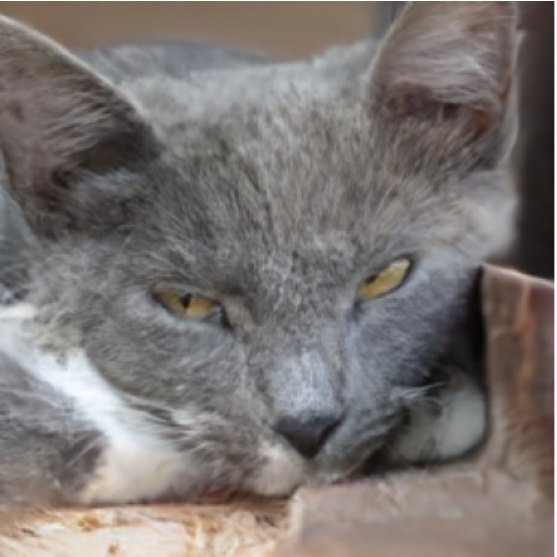}
       \makebox[\linewidth][c]{\footnotesize PSNR: 26.05}
    \end{minipage}
    \begin{minipage}[b]{0.13\textwidth}
    \captionsetup{skip=-0.01cm}
        \includegraphics[width=\linewidth]{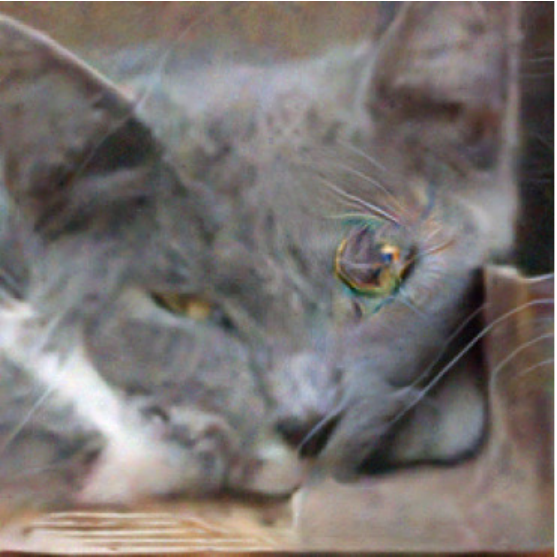}
       \makebox[\linewidth][c]{\footnotesize PSNR: 23.38}
    \end{minipage}
    \begin{minipage}[b]{0.13\textwidth}
    \captionsetup{skip=-0.01cm}
        \includegraphics[width=\linewidth]{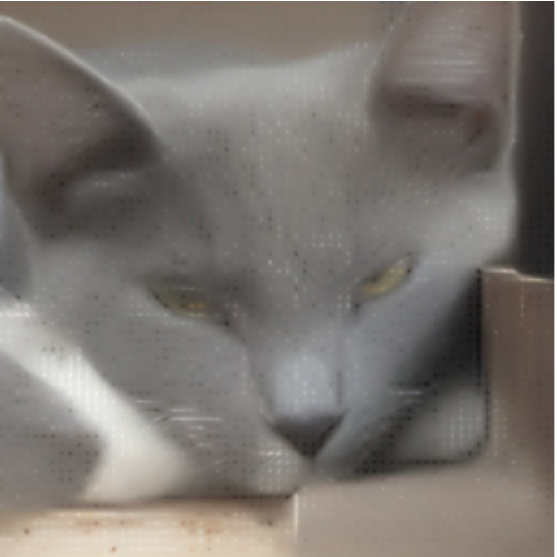}
       \makebox[\linewidth][c]{\footnotesize PSNR: 26.24}
    \end{minipage}
    \begin{minipage}[b]{0.13\textwidth}
        \captionsetup{skip=-0.01cm}
        \includegraphics[width=\linewidth]{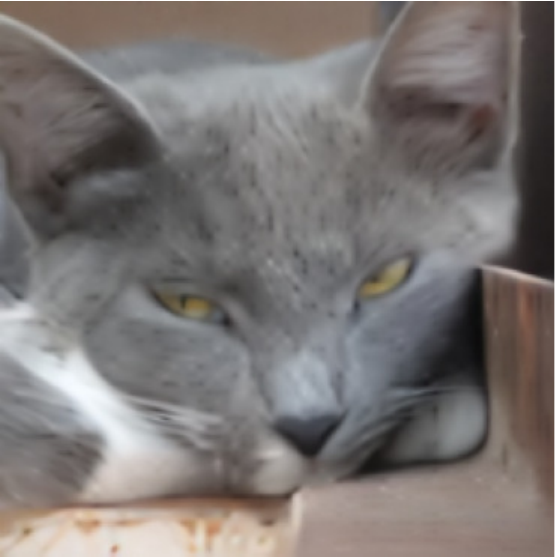}
       \makebox[\linewidth][c]{\footnotesize PSNR: \textbf{28.28}}
    \end{minipage}

    \caption{Comparison of super-resolution results on AFHQ-Cat.}
    \label{fig:afhq_sr}

\end{figure}

\section*{Acknowledgments}
We would like to acknowledge the assistance of volunteers in putting together this example manuscript and supplement.

\vspace{5cm}
\clearpage
\bibliographystyle{amsplain}
\bibliography{references}

\end{document}